\documentclass[twoside,11pt]{article}
\usepackage{amsthm,letltxmacro}
\LetLtxMacro{\myproof}{\proof}
\let\endmyproof\endproof
\usepackage[preprint]{jmlr2e}
\firstpageno{1}
\usepackage{lastpage}
\makeatletter
\if@preprint\else
\jmlrheading{23}{2022}{1-\pageref{LastPage}}{11/21; Revised
7/22}{9/22}{21-1404}{Zuowei Shen, Haizhao Yang, and Shijun Zhang}
\ShortHeadings{Achieving Arbitrary Accuracy with Fixed Number of Neurons}{Shen, Yang, and Zhang}
\fi
\makeatother
\usepackage{mathrsfs,amsmath,amsfonts,amssymb}
\usepackage{mathtools,xcolor,enumerate}
\usepackage{dsfont}
\usepackage{bm}
\usepackage{tabularx,multirow,array,booktabs}
\usepackage{colortbl}
\usepackage{graphicx}
\graphicspath{{figures/}}%
\usepackage{float}
\usepackage{subcaption} 
\newcommand{\R}{\ifmmode\mathbb{R}\else$\mathbb{R}$\fi}
\newcommand{\N}{\ifmmode\mathbb{N}\else$\mathbb{N}$\fi}
\newcommand{\Z}{\ifmmode\mathbb{Z}\else$\mathbb{Z}$\fi}
\newcommand{\Q}{\ifmmode\mathbb{Q}\else$\mathbb{Q}$\fi}
\newcommand{\A}{\ifmmode\mathbb{A}\else$\mathbb{A}$\fi}
\let\tn\textnormal

\newcommand{\bmx}{{\bm{x}}}

\newcommand{\bme}{{\bm{e}}}
\newcommand{\bmh}{{\bm{h}}}
\newcommand{\bmb}{{\bm{b}}}
\newcommand{\bmxi}{{\bm{\xi}}}

\newcommand{\bmy}{{\bm{y}}}
\newcommand{\bmtheta}{{\bm{\theta}}}

\newcommand{\bmA}{\bm{A}}

\newcommand{\bmzero}{{\bm{0}}}
\newcommand{\calO}{{\mathcal{O}}}

\newcommand{\calN}{{\mathcal{N}}}
\newcommand{\calD}{{\mathcal{D}}}
\newcommand{\calS}{{\mathcal{S}}}

\newcommand{\calI}{{\mathcal{I}}}
\newcommand{\calL}{{\mathcal{L}}}
\newcommand{\bmcalL}{{\bm{\mathcal{L}}}}

\newcommand{\calX}{{\mathcal{X}}}
\newcommand{\tildephi}{{\widetilde{\phi}}}
\newcommand{\tildeg}{{\widetilde{g}}}
\newcommand{\tildeh}{{\widetilde{h}}}

\newcommand{\tildey}{{\widetilde{y}}}
\newcommand{\tildeL}{{\widetilde{L}}}
\newcommand{\tildeN}{{\widetilde{N}}}
\newcommand{\tildeE}{{\widetilde{E}}}
\newcommand{\tildef}{{\widetilde{f}}}
\newcommand{\tildeM}{{\widetilde{M}}}

\newcommand{\caltildeI}{{\mathcal{\widetilde{I}}}}
\newcommand{\caltildeL}{{\mathcal{\widetilde{L}}}}
\newcommand{\tildexi}{{\widetilde{\xi}}}
\newcommand{\tildevarrho}{{\widetilde{\varrho}}}

\newcommand{\tildescrH}{{\widetilde{\mathscr{H}}}}

\newcommand{\tildesigma}{{\widetilde{\sigma}}}

\newcommand{\scrE}{{\mathscr{E}}}

\newcommand{\scrF}{{\mathscr{F}}}

\newcommand{\scrH}{{\mathscr{H}}}
\newcommand{\scrC}{{\mathscr{C}}}

\newcommand{\hatxi}{{\widehat{\xi}}}
\newcommand{\hatn}{{\widehat{n}}}
\newcommand{\hata}{{\widehat{a}}}
\newcommand{\hath}{{\widehat{h}}}
\newcommand{\hatbmf}{{\widehat{\bm{f}}}}

\newcommand{\mystep}[2]{\par \vspace{0.25cm}\noindent\textbf{\hspace{8pt}Step }$#1\colon$ #2 \vspace{0.18cm} \par }

\usepackage{tikz}
\newcommand{\myto}[2][1]{\mathop{
        \vcenter{\hbox{\scalebox{1}[#1]{\tikz{\draw[->,line width=0.72pt] (0,0.5) to (0.69*#2,0.5);}}}}
}}

\def\one{{\ensuremath{\mathds{1}}}}
\newcommand{\setMathResizeRate}[1]{\def\mathResizeRate{#1}}
\setMathResizeRate{0.8}
\newcommand{\mathResize}[2][\mathResizeRate]{
    \scalebox{#1}[#1]{\(\displaystyle #2\)}
}

\DeclareMathOperator*{\argmin}{arg\,min}

\usepackage[left,mathlines]{lineno}
\usepackage{refcount}

\definecolor{mylinenumbercolor}{HTML}{BEBEBE}

\makeatletter
\newcommand*\patchAmsMathEnvironmentForLineno[1]{%
    \expandafter\let\csname old#1\expandafter\endcsname\csname #1\endcsname
    \expandafter\let\csname oldend#1\expandafter\endcsname\csname end#1\endcsname
    \renewenvironment{#1}%
    {\linenomath\csname old#1\endcsname}%
    {\csname oldend#1\endcsname\endlinenomath}}%
\newcommand*\patchBothAmsMathEnvironmentsForLineno[1]{%
    \patchAmsMathEnvironmentForLineno{#1}%
    \patchAmsMathEnvironmentForLineno{#1*}}%
\patchBothAmsMathEnvironmentsForLineno{equation}%
\patchBothAmsMathEnvironmentsForLineno{align}%
\patchBothAmsMathEnvironmentsForLineno{flalign}%
\patchBothAmsMathEnvironmentsForLineno{alignat}%
\patchBothAmsMathEnvironmentsForLineno{gather}%
\patchBothAmsMathEnvironmentsForLineno{multline}%
\makeatother
\newcommand{\mailto}[2][]{\href{mailto:#2?cc=#1}{\color{black}#2}}

\LetLtxMacro{\proof}{\myproof}
\let\endproof\endmyproof

\long\def\acks#1{\vskip 0.3in\noindent{\large\bf Acknowledgments}\vskip 0.2in
    \noindent #1}
\hypersetup{hidelinks}
\makeatletter
\if@preprint
\hypersetup{colorlinks=true,
citecolor=cyan,linkcolor=blue,urlcolor=magenta}
\fi
\makeatother
\hypersetup{pdfauthor=Shijun Zhang,
pdftitle={Deep Network Approximation: Achieving Arbitrary Accuracy with Fixed Number of Neurons},
pdfkeywords={universal approximation property,
        fixed-size neural network,
        classification function, 
        periodic function,
        nonlinear approximation }
}
\usepackage{doi}
\newtheorem*{remark*}{Remark}

\let\epsilon\varepsilon
\let\subset\subseteq
\let\mycdots\cdots
\def\cdots{{\mycdots}}
\let\cite\citep

\definecolor{mypurple}{HTML}{A000A0}

\definecolor{mygray}{rgb}{0.905,0.905,0.905}

\begin{document}
{
\makeatletter
\long\def\@makefntext#1{\@setpar{\@@par\@tempdima \hsize 
             \advance\@tempdima-15pt\parshape \@ne 15pt \@tempdima}\par
             \parindent 2em\noindent \hbox to \z@{\hss{\textsuperscript{\scriptsize\@thefnmark}} \hfill}#1}
\makeatother
\hypersetup{allcolors=black}
\title{Deep Network Approximation: Achieving Arbitrary Accuracy with Fixed Number of Neurons}

\author{\name Zuowei Shen
		\email \mailto{matzuows@nus.edu.sg} \\
		\addr   Department of Mathematics\\
		  National University of Singapore
	\AND \name Haizhao Yang 
	\email \mailto{hzyang@umd.edu}  \\
	\addr   Department of Mathematics\\  
	University of Maryland, College Park
	\AND \name 
     \href{https://shijunzhang.top/}{Shijun Zhang}\thanks{Corresponding author.}
	\email \mailto[shijun.math@outlook.com]{zhangshijun@u.nus.edu}\\ 
	\addr   Department of Mathematics\\  
    National University of Singapore
		}
		
\editor{Ruslan Salakhutdinov}
\maketitle
}
\begin{abstract}

This paper develops simple feed-forward neural networks that achieve the universal approximation property for all continuous functions with a fixed finite number of neurons. These neural networks are simple because they are designed with a simple, computable, and continuous activation function $\sigma$ leveraging a triangular-wave function and the softsign function. We first prove that  $\sigma$-activated networks with width $36d(2d+1)$ and depth $11$ can approximate any continuous function on a $d$-dimensional hypercube within an arbitrarily small error. Hence, for supervised learning and its related regression problems, the hypothesis space generated by these networks with a size not smaller than $36d(2d+1)\times 11$ is dense in the continuous function space $C([a,b]^d)$ and therefore dense in the Lebesgue spaces $L^p([a,b]^d)$ for $p\in [1,\infty)$. Furthermore, we show that classification functions arising from image and signal classification are in the hypothesis space generated by $\sigma$-activated networks with width $36d(2d+1)$ and depth $12$ when there exist pairwise disjoint bounded closed subsets of $\mathbb{R}^d$ such that the samples of the same class are located in the same subset. Finally, we use numerical experimentation to show that replacing the rectified linear unit (ReLU) activation function by ours would improve the experiment results.

\end{abstract}



\section{Introduction}
\label{sec:intro}

Deep neural networks have been widely used in data science and artificial intelligence. Their tremendous successes in various applications have motivated extensive research to establish the theoretical foundation of deep learning.  Understanding the approximation capacity of deep neural networks is one of the keys to revealing the power of deep learning. The most basic layers of deep neural networks are nonlinear functions as the composition of an affine linear transform and a nonlinear activation function. The composition of these simple nonlinear functions can generate a complicated deep neural network with powerful approximation capacity, which is the key difference from classic approximation tools. In this paper, we show that the hypothesis  space of deep neural networks generated from the composition of $11$ such simple nonlinear functions is dense in the continuous function space $C([a,b]^d)$ when the affine linear transforms  
are parameterized with $\calO(d^2)$ non-zero parameters in total
and the nonlinear activation function is constructed from a simple triangular-wave function and the softsign function. 

\subsection{Main Results}

One of the key elements of a neural network is its activation functions. Searching for simple activation functions enabling powerful approximation capacity of neural networks is an important mathematical problem that probably originated in the Kolmogorov superposition theorem (KST) \cite{Kol1957} for Hilbert's 13-th problem, where a two-hidden-layer neural network 
with $\calO(d)$ neurons and
complicated activation functions depending on  the target functions are constructed to represent an arbitrary function in $C([0,1]^d)$. Since then, whether simple and computable activation functions independent of  the target function  exist to make the space of neural networks 
with $\calO(d)$ neurons
dense in $C([0,1]^d)$ or even equal to $C([0,1]^d)$ has been an open problem. 
A function $\varrho:\R\to\R$ is said to be a universal activation function (UAF) if the function space generated by $\varrho$-activated networks with $C_{\varrho,d}$ neurons is dense in $C([0,1]^d)$, where $C_{\varrho,d}$ is a constant determined by $\varrho$ and $d$. That is, if $\varrho$ is a UAF, then $\varrho$-activated networks with $C_{\varrho,d}$ neurons can approximate any continuous function within an arbitrary error on $[0,1]^d$ by only adjusting the parameters.

In this paper, we first construct a simple and computable example of UAFs.  
As a typical and simple UAF,
this activation function is called elementary universal activation function (EUAF), and the corresponding networks are called EUAF networks. Then, we prove that the function space generated by EUAF networks with $\calO(d^2)$ neurons is dense in $C([a,b]^d)$. Furthermore, it is shown that EUAF networks with $\calO(d^2)$ neurons can exactly represent $d$-dimensional classification functions. 

While a good activation function should be simple and numerically implementable, the neural network 
activated by it should be able to approximate continuous functions well with a manageable size. Considering these requirements and motivated by previous works \cite{yarotsky:2019:06,shijun4,shijun5}, the activation function to be chosen should have appropriate nonlinearity,  periodicity, and the capacity to reproduce step functions.  It is challenging to find a single activation function with all these properties. Here, we propose an activation function with all required properties by using two simple functions $\sigma_1$ and $\sigma_2$ defined below.

Let $\sigma_1$ be the continuous triangular-wave function with period $2$, i.e.,   
\begin{equation*}
	\sigma_1(x)\coloneqq |x| \quad \tn{for any $x\in[-1,1]$} 
\end{equation*}
and $\sigma_1(x+2)=\sigma_1(x)$ for any $x\in\R$.
Alternatively, $\sigma_1$ can also be written as: 
\begin{equation*}
	\sigma_1(x)=\big|x-2\lfloor \tfrac{x+1}{2}\rfloor\big| \quad \tn{for any $x\in\R$, \quad where $\lfloor\cdot\rfloor$ is the floor function.} 
\end{equation*}
Clearly, $\sigma_1$ is periodic and $x-\sigma_1(x)$ is a continuous variant of the floor function as desired. 

To introduce high  nonlinearity, let $\sigma_2$ be the softsign activation function commonly used in machine learning \cite{10.5555/1620853.1620921,le-zuidema-2015-compositional}:
\begin{equation*}
	\sigma_2(x)\coloneqq\frac{x}{|x|+1} \quad \tn{for any $x\in\R$}.
\end{equation*}

Then the activation function $\sigma$ is  defined as:
\begin{equation}\label{eq:def:sigma}
	\sigma(x)\coloneqq \begin{cases}
		\sigma_1(x) & \tn{for} \  x\in [0,\infty),\\
		\sigma_2(x) & \tn{for} \  x\in (-\infty,0).\\
	\end{cases}
\end{equation}
See an illustration of $\sigma$ in Figure~\ref{fig:sigma}. This activation function $\sigma$ is used to construct powerful neural networks in this paper.

\begin{figure}[htbp!]        
	\centering          
	\includegraphics[width=0.6418\textwidth]{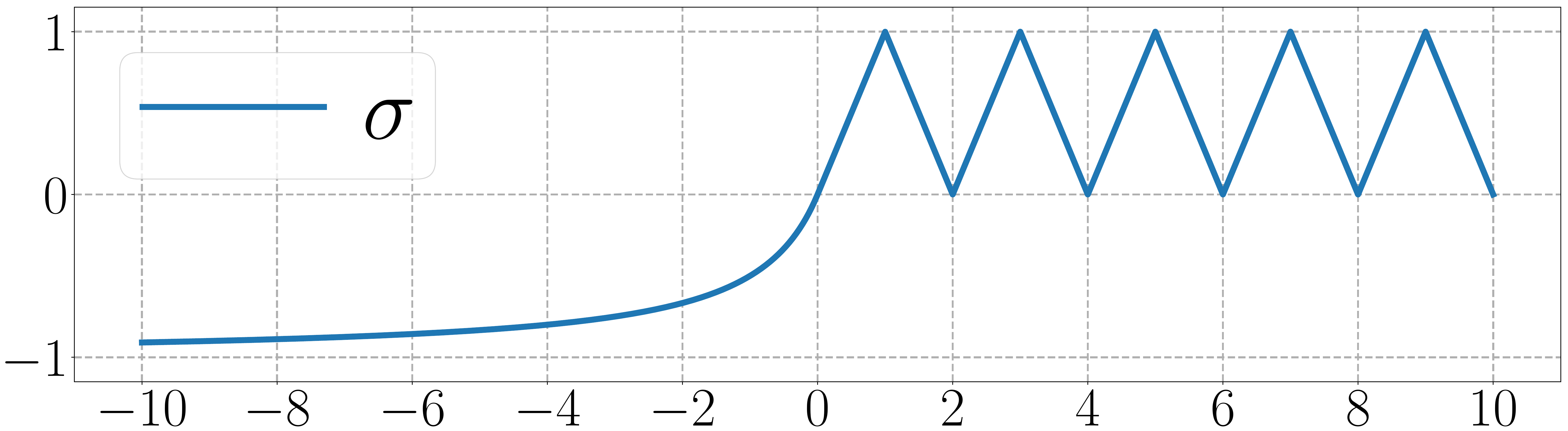}
	\caption{An illustration of $\sigma$ on $[-10,10]$.}
	\label{fig:sigma}
\end{figure}

 As we shall see later, the periodicity of the triangular-wave function $\sigma_1$ and the (high) nonlinearity of the softsign function $\sigma_2$ play crucial roles in the proofs of our main results. One may find more details Section~\ref{sec:key:ideas}, which provides the ideas of proving our main results.
 Observe that $\sigma_1$ is an even function and $\sigma_2$ is an odd function, i.e., $\sigma(x)=\sigma_1(x)=\sigma_1(-x)$ for any $x\ge 0$ and $-\sigma(-x)=-\sigma_2(-x)=\sigma_2(x)$ for any $x\ge 0$. This implies that  
 $\sigma(x)$ and  $-\sigma(-x)$ with $x\geq 0$ have both required periodicity and nonlinearity features and play the same roles as $\sigma_1(x)$ and $\sigma_2(x)$, respectively.
 These requirements lead to our choice of $\sigma$ as the activation function. 
 If allowed to be more complicated, one can design many other UAFs satisfying stronger requirements for various applications. For example, the idea of designing a $C^s$ UAF is given in Section~\ref{sec:UAF:smooth} and
a sigmoidal UAF (see Figure~\ref{fig:tildesigma}) is constructed in Section~\ref{sec:UAF:sig}.

With the  activation function $\sigma$ in hand, let us introduce the network (architecture) using $\sigma$ as the activation function, called $\sigma$-activated network (architecture).
To be precise, a $\sigma$-activated network with a (vector) input $\bmx\in\R^d$, an output $\Phi(\bmx,\bmtheta)\in\R$, and $L\in\N^+$ hidden layers can be briefly described as follows:
\begin{equation}\label{eq:Phi:x:theta}
    \begin{aligned}
    \bm{x}=\widetilde{\bm{h}}_0 
    \myto{2.0}^{\bm{A}_0,\ \bm{b}_0}_{\bmcalL_0} \bm{h}_1
    \myto{1.15}^{\sigma} \widetilde{\bm{h}}_1 
    \quad  \cdots\quad  
    \myto{2.6}^{\bm{A}_{L-1},\ \bm{b}_{L-1}}_{\bmcalL_{L-1}} \bm{h}_L
    \myto{1.15}^{\sigma} \widetilde{\bm{h}}_L
    \myto{2.0}^{\bm{A}_{L},\ \bm{b}_{L}}_{\bmcalL_L} \bm{h}_{L+1}=\Phi(\bm{x},\bmtheta),
    \end{aligned}
\end{equation}
    where $N_0=d\in\N^+$, $N_1,N_2,\cdots,N_L\in\N^+$,  $N_{L+1}=1$, $\bm{A}_i\in \R^{N_{i+1}\times N_{i}}$ and $\bm{b}_i\in \R^{N_{i+1}}$ are the weight matrix and the bias vector in the $i$-th affine linear transform $\bmcalL_i$, respectively, i.e., 
    \[\bm{h}_{i+1} =\bm{A}_i\cdot \widetilde{\bm{h}}_{i} + \bm{b}_i\eqqcolon \bmcalL_i(\widetilde{\bm{h}}_{i})\quad \tn{for $i=0,1,\cdots,L$}\]  
    and
    \[
       \widetilde{{h}}_{i,j}= \sigma({h}_{i,j})\quad \tn{for $j=1,2,\cdots,N_i$ and $i=1,2,\cdots,L$.}
    \]
    Here,
     $\widetilde{{h}}_{i,j}$ and ${h}_{i,j}$ are the $j$-th entries of $\widetilde{\bm{h}}_i$ and $\bm{h}_i$, respectively, for
    $j=1,2,\cdots,N_i$ and $i=1,2,\cdots,L$.
    $\bmtheta$ is a fattened vector consisting of all parameters in  $\bmA_0,\bmb_0,\bmA_1,\bmb_1,\cdots,\bmA_L,\bmb_L$.
    
    With a slight abuse of notation, $\sigma$ can be applied to a vector elementwisely, 
    i.e., given any $k\in\N^+$,
    \begin{equation*}
        \sigma(\bmy)=\big[\sigma(y_1),\,\sigma(y_2),\,\cdots,\,\sigma(y_k)\big]^T\quad\tn{ for any $\bmy=[y_1,y_2,\cdots,y_k]^T\in\R^k$.}
    \end{equation*}
    Then $\Phi$ can be represented in a form of function compositions as follows:
    \begin{equation*}
        \Phi(\bmx,\bmtheta) =\bmcalL_L\circ\sigma\circ
        \bmcalL_{L-1}\circ 
        \,\ \cdots \,\  
        \circ \sigma\circ\bmcalL_1\circ\sigma\circ\bmcalL_0(\bmx)\quad \tn{for any $\bmx\in\R^d$}.
    \end{equation*}
    
Given $N,L\in\N^+$, let $\Phi_{N,L}(\bmx,\bmtheta)$ denote the $\sigma$-activated network architecture $\Phi(\bmx,\bmtheta)$ in Equation~\eqref{eq:Phi:x:theta} with $N_1=N_2=\cdots=N_L=N$.
Let 
\begin{equation*}
    W=W_{d,N,L}=d\times N +N\  +\  (N\times N+N)\times (L-1)\  + \   N\times 1+1=\calO(dN+ N^2L)
\end{equation*}
be the total number of parameters in $\Phi_{N,L}(\bmx,\bmtheta)$, i.e., $\bmtheta\in \R^W$.

Define the hypothesis space $\mathscr{H}_{d}(N,L)$ as the function space generated by $d$-input EUAF networks with width $N$ and depth $L$, i.e.,
\begin{equation}\label{eq:hypothesis:space:def}
    \mathscr{H}_{d}(N,L)\coloneqq \Big\{\phi: \phi(\bmx)=\Phi_{N,L}(\bmx,\bmtheta)\tn{ for any $\bmx\in \R^d$},\quad \bmtheta\in\R^W\Big\}.
\end{equation}

Let $C([a,b]^d)$ be the space of all  continuous functions $f:[a,b]^d\to\R$  with the maximum norm. Our first main result, Theorem~\ref{thm:main} below, shows that EUAF networks with a fixed size $\calO(d^2)$ enjoy the universal approximation property by only adjusting their parameters. 
\begin{theorem}\label{thm:main}
	Let $f\in C([a,b]^d)$ be a continuous function and  $\scrH_d(N,L)$ be the hypothesis space defined in Equation~\eqref{eq:hypothesis:space:def} with $N=36d(2d+1)$ and $L=11$.  Then, for an arbitrary  $\varepsilon>0$,
	there exists $\phi\in \scrH_d(N,L)$  such that
	\begin{equation*}
	|\phi(\bmx)-f(\bmx)|<\varepsilon\quad \tn{for any $\bmx\in[a,b]^d$.}
	\end{equation*}
\end{theorem}

To prove Theorem~\ref{thm:main}, we first summarize key proof ideas in Section~\ref{sec:key:ideas} and then present the detailed proof later in Section~\ref{sec:proof:main}.
\begin{remark*}
The network realizing $\phi$ in Theorem~\ref{thm:main} has
\[d\times N+ N \ +\ (N\times N+N)\times(L-1)\ +\ N\times 1+1  \sim  d^4 \]
parameters, where $N=36d(2d+1)$ and $L=11$.  However, as shown in our constructive proof of Theorem~\ref{thm:main},
it is enough to adjust  $5437(d+1)(2d+1)=\calO(d^2) \ll d^4$ parameters and set all the others to $0$.
\end{remark*}

Since for an arbitrary $M>0$, $2M\sigma(\tfrac{x+M}{2M})-M=x$ for all  $x\in [-M,M]$, we can manually add hidden layers to EUAF networks without changing the output.
This leads to the following immediate corollary of Theorem~\ref{thm:main}.

\begin{corollary}\label{cor:main}
    Assume  $N\ge 36d(2d+1)$ and  $L\ge 11$. Then the hypothesis space
    $\mathscr{H}_d(N,L)$ defined in Equation~\eqref{eq:hypothesis:space:def} is dense in $C([a,b]^d)$.
\end{corollary}

The stable and accurate approximation of discontinuities has many real-world applications and has been widely studied \cite{2639-8001_2019_4_491,BECK2020109824,app10072279,10.1088/2632-2153/ac3149,2021arXiv210605587H}.
Most of common discontinuous functions are in the Lebesgue spaces $L^p([a,b]^d)$ for $p\in [1,\infty)$.
Let us consider the denseness of our hypothesis space in these function spaces.
 Since $C([a,b]^d)$ is dense in $L^p([a,b]^d)$ for $p\in [1,\infty)$,
the hypothesis space in Corollary~\ref{cor:main} is also dense in $L^p([a,b]^d)$ as shown in the following corollary.
 \begin{corollary}\label{cor:main:Lp}
     Assume  $N\ge 36d(2d+1)$, $L\ge 11$, and $p\in [1,\infty)$. Then the hypothesis space
    $\mathscr{H}_d(N,L)$ defined in Equation~\eqref{eq:hypothesis:space:def} is dense in $L^p([a,b]^d)$.
 \end{corollary}
This corollary implies that, for  $f\in L^p([a,b]^d)$ and an arbitrary $\varepsilon>0$, there exists $\phi\in \mathscr{H}_d(N,L)$ such that $\|\phi-f\|_{L^p([a,b]^d)}< \varepsilon$.

One can ask whether the arbitrary error $\varepsilon>0$ in Theorem~\ref{thm:main} can be further reduced to $0$.  This is not true in general, but it is true for a class of interesting functions   
widely used in image classification. 
Given any pairwise disjoint bounded closed subsets $E_1,E_2,\cdots,E_J\subseteq \R^d$, define ``the classification function space'' of these subsets as 
\begin{equation*}
    \mathscr{C}_d(E_1,E_2,\cdots,E_J)\coloneqq \bigg\{f: f=\sum_{j=1}^J r_j\cdot \one_{E_j}\tn{\ for any \ }
    r_1,r_2,\cdots,r_J\in \Q 
    \bigg\},
\end{equation*}
where $\one_{E_n}$ is the indicator function of $E_j$ for each $j$. 
Our second main result, Theorem~\ref{thm:main:classification} below, shows that each element of $\mathscr{C}_d(E_1,E_2,\cdots,E_J)$
can be exactly represented by a $\sigma$-activated network 
with $\calO(d^2)$ neurons
in $\bigcup_{j=1}^J E_j$.

\begin{theorem}\label{thm:main:classification}
	Let $E_1,E_2,\cdots,E_J\subseteq \R^d$ be  pairwise disjoint bounded closed subsets and $\mathscr{H}_d(N,L)$ be the hypothesis space
 defined in Equation~\eqref{eq:hypothesis:space:def} with $N=36d(2d+1)$ and $L=12$. Then, for an arbitrary $f\in \mathscr{C}_d(E_1,E_2,\cdots,E_J)$,
	there exists  $\phi\in \scrH_d(N,L)$   such that
	\begin{equation*}
		\phi(\bmx)=f(\bmx)\quad \tn{for any $\bmx\in \bigcup_{j=1}^J E_j$.}
	\end{equation*}
\end{theorem}

\begin{remark*}
The network realizing $\phi$ in Theorem~\ref{thm:main:classification} has
\[d\times N+ N \ +\ (N\times N+N)\times(L-1)\ +\ N\times 1+1  \sim  d^4 \]
parameters, where $N=36d(2d+1)$ and $L=12$.  However, as shown in our constructive proof of Theorem~\ref{thm:main:classification} in Section~\ref{sec:proof:thm:main:classification}, it is enough to adjust  $5509(d+1)(2d+1)=\calO(d^2) \ll d^4$ parameters and set all the others to $0$.
\end{remark*}

For a general function space $\scrF$, define  $\scrF|_E\coloneqq \big\{f|_E: f\in \scrF\big\}$, where $f|_E$ is the function achieved via limiting $f$ on $E$. Then, we have a corollary of Theorem~\ref{thm:main:classification} as follows.
\begin{corollary}\label{cor:main:classification}
	 Let $E_1,E_2,\cdots,E_J\subseteq \R^d$  be pairwise disjoint bounded closed subsets and $\mathscr{H}_d(N,L)$ be the hypothesis space
 defined in Equation~\eqref{eq:hypothesis:space:def}.  If $N\ge 36d(2d+1)$ and  $L\ge 12$, then
	\begin{equation*}
	    \scrC_d(E_1,E_2,\cdots,E_J)\big|_E \subseteq \scrH_d(N,L)\big|_E
	    \quad \tn{with $E=\bigcup_{j=1}^J E_j$.}
	\end{equation*}
\end{corollary}

 One of the most successful applications of deep learning is image and signal classification. In supervised classification problems, given a few samples and their labels (usually integers), the goal of the task is to learn how to assign a label to a new sample. For example, in binary classification via deep learning, a neural network is trained based on given samples (and labels) to approximate a classification function mapping one class of samples to $0$ and the other class of samples to $1$. 
Theorem~\ref{thm:main:classification} (or Corollary~\ref{cor:main:classification}) implies that the classification function can be exactly realized by an EUAF network with a size depending only on the dimension of the problem domain via adjusting its parameters.
This means that the best approximation error of EUAF networks to classification functions in the classification problem is $0$.

We remark that, in the worst scenario, there might exist complicated high-dimensional functions such that,
the parameters of the EUAF network in Theorem~\ref{thm:main} (or \ref{thm:main:classification}) require high computer precision for storage, and the precision might be exponentially high in the problem dimension. We refer to this as the curse of memory, which may make Theorem~\ref{thm:main} and  \ref{thm:main:classification} less interesting in real-world applications, though the number of parameters can be very small. The key question to be addressed is how rare the curse of memory would happen in real-world applications. If the target functions in real-world applications typically have no curse of memory with a high probability, then EUAF networks would be very useful in real-world applications. In future work, we will explore the statistical characterization of high-dimensional functions for the curse of memory of EUAF networks. Another approach to reducing the memory requirement is to increase the network size. Our main result has provided a network size $\calO(d^2)$ to achieve an arbitrary error. If a larger network size is used, the curse of memory can be lessened as we shall discuss in Section~\ref{sec:computablity}.


\subsection{Related Work}

In recent years, there has been an increasing amount of literature on the approximation power of neural networks as a special case of nonlinear approximation \cite{devore_1998,cohen2020optimal,Ingrid}. 
In the early works of approximation theory for neural networks, the universal approximation theorem  \cite{Cybenko1989ApproximationBS,HORNIK1991251,HORNIK1989359} without approximation errors showed that there exists a sufficiently large neural network approximating a target function in a certain function space within any given error $\varepsilon>0$. 
There are also other versions of the universal approximation theorem. 
For example, it was shown in \cite{NEURIPS2018_03bfc1d4} that the residual neural networks activated the rectified linear unit (ReLU) with one neuron per hidden layer and a sufficiently large depth are a universal approximator. The universal approximation property for general residual neural networks was proved in \cite{2019arXiv191210382L} via a dynamical system approach. In all papers discussed above, the network size goes to infinity when the target approximation error approaches $0$. However, our result in Theorem~\ref{thm:main} implies that EUAF networks with a fixed size ($\calO(d^2)$ neurons in total)  can achieve an arbitrary small  error for approximating $f\in C([a,b]^d)$.

The approximation errors in terms of  the total number of parameters  of ReLU networks are well studied for basic function spaces with (nearly) optimal approximation errors, e.g., (nearly) optimal asymptotic errors for continuous functions \cite{yarotsky18a}, $C^s$  functions \cite{yarotsky:2019:06}, piecewise smooth functions \cite{PETERSEN2018296}, solutions of special PDEs \cite{sPDE1,sPDE2}, functions that can be optimally approximated by affine systems \cite{B_lcskei_2019},  and Sobolev spaces \cite{yang2020approximation,hon2021simultaneous}. Approximation errors in terms of width and depth would be more useful than those in terms of the total number of nonzero parameters in practice, because width and depth are two essential hyper-parameters in every numerical algorithm instead of the number of nonzero parameters. This motivated the works on the (nearly) optimal non-asymptotic errors in terms of width and depth with explicit pre-factors for approximating continuous functions in \cite{shijun2,shijun6,shijun:thesis} and for $C^s$ functions in \cite{shijun3,shijun:thesis}. 
As the errors are (nearly) optimal, there are two possible directions to improve the approximation error in order to reduce the effect of the curse of dimensionality. The first one is to consider smaller target function spaces, e.g., analytic functions \cite{DBLP:journals/corr/abs-1807-00297,10.1007/s00365-020-09511-4}, Barron
spaces \cite{barron1993,Weinan2019,e2020representation,siegel2021optimal},  and band-limited functions \cite{doi:10.1002/mma.5575,bandlimit}.

Another direction is to design advanced activation functions, where one can use multiple activation functions,  to enhance the power of neural networks, especially to conquer the curse of dimensionality in network approximation.
 There have been several papers
designing activation functions to achieve good approximation errors. 
The results in \cite{yarotsky:2019:06} imply that
 $(\sin,\tn{ReLU})$-activated neural networks (i.e., the activation function of a neuron can be  chosen from either $\sin$ or $\tn{ReLU}$) with $W$ parameters can approximate Lipschitz continuous functions with an asymptotic approximation error $\calO(e^{-c_d\sqrt{W}})$, where $c_d$ is a constant depending on $d$.
In \cite{shijun4}, it was shown that $(\tn{Floor},\tn{ReLU})$-activated neural networks with width $\calO(N)$ and depth $\calO(L)$ admit an quantitative approximation error $\calO(\sqrt{d}N^{-\sqrt{L}})$ for Lipschitz continuous functions, conquering the curse of dimensionality in approximation with a root-exponentially small error in depth $L$.\footnote{ 
Although there is no curse of dimensionality in network approximation, the construction requires exponentially many data samples of the target function and computer memory. Hence, there would be a curse of dimensionality in inferring a target function from its finite samples when standard learning techniques are applied to a computer.
} 
In \cite{shijun5}, it was shown that, even if the depth is as small as $3$, neural networks with width $N$  and $\calO(d+N)$ nonzero parameters can approximate Lipschitz continuous functions with an exponentially small error $\calO(\sqrt{d}\,2^{-N})$, if the floor function $\lfloor x \rfloor$, the exponential function $2^x$, and the step function $\one_{\{x\ge 0\}}$ are used as activation functions. Recently in \cite{DBLP:journals/corr/abs-2103-00542}, the results in \cite{yarotsky:2019:06,shijun5} were combined to avoid the curse of dimensionality using ReLU, sin, and $2^x$ activation functions. Corollary~\ref{cor:main} implies that the hypothesis  space  of EUAF networks activated by a single activation function with  $\calO(d^2)$ neurons is dense in $C([a,b]^d)$. Particularly,  all continuous functions can be arbitrarily approximated by fixed-size EUAF networks with width $N$ and depth $L$ on a $d$-dimensional hypercube whenever $N\ge 36d(2d+1)$ and  $L\ge 11$.

 There is another research line for the approximation error of neural networks: applying KST \cite{Kol1957} or its variants to explore new activation functions for a fixed-size network to achieve an arbitrary error. 
 The original KST shows that any multivariate function $f\in C([0,1]^d)$ can be represented as $f(\bmx)=\sum_{i=0}^{2d}g_i\big(\sum_{j=1}^d h_{i,j}(x_j)\big)$ for any $\bmx=[x_1,x_2,\cdots,x_d]^T\in[0,1]^d$, where $g_i$ and $h_{i,j}$ are univariate continuous functions. In fact, the composition architecture of KST  can be regarded as a special neural network with (complicated) activation functions depending on the target function, which results in the failure of KST in practice. To alleviate this issue, 
a single activation function independent of the target function is designed in \cite{MAIOROV199981} to construct networks with a fixed size ($\calO(d)$ neurons) to achieve an arbitrary error for approximating functions in $C([-1,1]^d)$. However, the activation function in \cite{MAIOROV199981} has no closed form and is hardly computable. See Section~\ref{sec:key:ideas} for a detailed discussion of the construction in \cite{MAIOROV199981}. The computability issue of activation functions was addressed recently in \cite{yarotsky:2021:02}. It was shown in \cite{yarotsky:2021:02} that, for an arbitrary  $\varepsilon>0$ and any function $f$ in $C([0,1]^d)$, there exists a network of size only depending on $d$ constructed with multiple activation functions either  ($\sin$ \& $\arcsin$) or ($\lfloor \cdot\rfloor$ \& a non-polynomial analytic function) to approximate $f$ within an error $\varepsilon$. To the best of our knowledge, there is no explicit characterization of the size dependence on $d$ in \cite{yarotsky:2021:02}. For example, a very important question is whether the dependence can be mild, e.g., only a polynomial of $d$, or has to be severe, e.g., exponentially in $d$. The results of the current paper provide positive answers to all the issues discussed above: We show that  EUAF networks with a simple and computable activation function, width $36d(2d+1)$, and depth $11$ can approximate functions in $C([a,b]^d)$ within an arbitrary pre-specified error $\varepsilon>0$.
 
In summary, this paper aims to design a simple and computable activation function $\sigma$  to construct fixed-size neural networks with the universal approximation property. The network width and depth are explicitly characterized,   depending only on the dimension $d$.   The fixed-size neural network is designed to approximate any continuous functions on a hypercube within an arbitrary error by only adjusting  $\calO(d^2)$ network parameters.   Moreover, we prove that an arbitrary classification function can be exactly represented by such a fixed-size network architecture via only adjusting $\calO(d^2)$ network parameters.  The main contribution of this paper is to develop a  rigorous mathematical analysis for the universal approximation property of fixed-size neural networks. The mathematical analysis developed here would provide a deeper understanding for other neural networks and the approximation results discussed here can be applied to the full error analysis of deep learning in the next subsection.

\subsection{Error Analysis}
\label{sec:error:analysis}

We will briefly discuss the full error analysis of deep neural networks.
Let $\Phi(\bm{x},\bm{\theta})$ denote a function of $\bmx\in\calX$ generated by a  network architecture  parameterized with $\bm{\theta}\in\R^W$. 
Given a target function $f$ defined on $\calX$, the final goal is to find the expected risk minimizer
\begin{equation*}
	\bm{\theta}_{\mathcal{D}}
	\in 
	\argmin_{\bm{\theta}\in\R^W} R_{\mathcal{D}}(\bm{\theta}),\quad \tn{where}\   R_{\mathcal{D}}(\bm{\theta})\coloneqq \mathbb{E}_{\bm{x}\sim U(\mathcal{X})} \big[\ell\big( \Phi(\bm{x},\bm{\theta}),f(\bm{x})\big)\big]
\end{equation*}
with an unknown data {distribution} $U(\mathcal{X})$ over $\calX$ and a loss function $\ell(\cdot,\cdot)$ typically taken as $\ell(y_1,y_2)=\tfrac{1}{2}|y_1-y_2|^2$. 
Note that $\bmtheta_\calD$ may not be always achievable.
For any pre-specified $\eta>0$, one can always identify $\bmtheta_{\calD,\eta}\in\R^W$ instead of $\bmtheta_\calD$ such that
\begin{equation}\label{eq:RD:inf:error}
    R_\calD\big(\bmtheta_{\calD,\eta}\big)
    \le \inf_{\bmtheta\in\R^W} R_\calD\big(\bmtheta\big)+\eta/2.
\end{equation}
Since the expected risk $R_\calD(\bmtheta)$ is not available in practice, we use the empirical risk $R_\calS(\bmtheta)$ to approximate $R_\calD(\bmtheta)$ for given {samples} $\big\{\big( \bm{x}_i,f(\bm{x}_i)\big)\big\}_{i=1}^n$ and our goal is to identify  the empirical risk minimizer
\begin{equation*}
	\bm{\theta}_{\mathcal{S}}
	\in 
	\argmin_{\bm{\theta}\in\R^W}R_{\mathcal{S}}(\bm{\theta}),\quad \tn{where}\ R_{\mathcal{S}}(\bm{\theta}):=
	\frac{1}{n}\sum_{i=1}^n \ell\big( \Phi(\bm{x}_i,\bm{\theta}),f(\bm{x}_i)\big).
\end{equation*}
Similarly, 
$\bmtheta_\calS$ is not always achievable. For any pre-specified $\eta>0$, one can always identify $\bmtheta_{\calS,\eta}\in \R^W$ instead of $\bmtheta_\calS$ such that
\begin{equation}\label{eq:RS:inf:error}
    R_\calS\big(\bmtheta_{\calS,\eta}\big)
    \le \inf_{\bmtheta\in\R^W} R_\calS\big(\bmtheta\big)+\eta/2.
\end{equation}

In practical implementation, only a numerical minimizer $\bm{\theta}_{\mathcal{N}}$ of $R_\calS(\bmtheta)$ can be achieved via  a numerical optimization method. The discrepancy between  the learned function $\Phi(\bm{x},\bm{\theta}_{\mathcal{N}})$ and the target function $f$ is measured by $R_{\mathcal{D}}(\bm{\theta}_{\mathcal{N}}) $, which is bounded by
\begin{equation*}\setMathResizeRate{0.70}
\begin{split}
		\mathResize{
		R_{\mathcal{D}}(\bm{\theta}_{\mathcal{N}})\,  
		} 
		&
		\mathResize{
		 =  \underbrace{[R_\calD(\bmtheta_\calN)-R_\calS(\bmtheta_\calN)]}_{\tn{GE}}
		+\underbrace{[R_\calS(\bmtheta_\calN)-R_\calS(\bmtheta_{\calS,\eta})]}_{\tn{OE}}
		+\underbrace{[R_\calS(\bmtheta_{\calS,\eta})-R_\calS(\bmtheta_{\calD,\eta})]}_{\le \eta/2\  \tn{by \eqref{eq:RS:inf:error}}}
		+\underbrace{[R_\calS(\bmtheta_{\calD,\eta})-R_\calD(\bmtheta_{\calD,\eta})]}_{\tn{GE}}
		+\underbrace{R_\calD(\bmtheta_{\calD,\eta})}_{\le \inf\limits_{\bmtheta\in\R^W}R_\calD(\bmtheta)+\eta/2\ \tn{by \eqref{eq:RD:inf:error}}}
		}\\  
		&
		\mathResize{
		\le \underbrace{\eta}_{\tn{perturbation}} 
		\ +\
		\underbrace{\inf_{\bmtheta\in\R^W} R_{\mathcal{D}}(\bm{\theta})}_{\tn{approximation error}} 
		\   + \ 
		\underbrace{[R_{\mathcal{S}}(\bm{\theta}_{\mathcal{N}})-
		\inf_{\bmtheta\in\R^W} R_{\mathcal{S}}(\bm{\theta})]}_{\tn{optimization error (OE)}}
		\  + \ \underbrace{[R_{\mathcal{D}}(\bm{\theta}_{\mathcal{N}})-R_{\mathcal{S}}(\bm{\theta}_{\mathcal{N}})]
			+[R_{\mathcal{S}}(\bm{\theta}_{\mathcal{D},\eta})-R_{\mathcal{D}}(\bm{\theta}_{\mathcal{D},\eta})]}_{\tn{generalization error (GE)}}.  
			}
\end{split}
\end{equation*}
If $\Phi(\bmx,\bmtheta)$ is realized by  EUAF networks, then
Theorem~\ref{thm:main} implies
\begin{equation*}
    \inf_{\bmtheta\in\R^W}\|\Phi(\cdot,\bmtheta)-f(\cdot)\|_{L^\infty(\calX)}=0\quad \tn{ for all $f\in C(\calX)$ with $\calX=[a,b]^d$.}
\end{equation*} 
It follows that 
\begin{equation*}
    \inf_{\bmtheta\in\R^W}R_\calD(\bmtheta)=\inf_{\bmtheta\in\R^W}\mathbb{E}_{\bm{x}\sim U(\mathcal{X})} \big[\ell\big( \Phi(\bm{x},\bm{\theta}),f(\bm{x})\big)\big]=0.
\end{equation*}
Since the pre-specified hyper-parameter $\eta$ can be arbitrarily small, the full error analysis can be reduced to
the analysis of the optimization and generalization errors, which depends on data samples, optimization algorithms, etc. 
One could refer to \cite{neyshabur2018the,Weinan2019APE,Weinan2019,2020arXiv200913500E,NIPS2016_6112,DBLP:journals/corr/NguyenH17,opt,xu2020,JMLR:v20:17-526} for the analysis of the generalization and optimization errors. 

\subsection{Computability}\label{sec:computablity}

The EUAF network is simple and computable in the sense that the output and subgradient of EUAF networks can be efficiently evaluated. The computability of EUAF implies that we can numerically implement the optimization algorithm to find a numerical minimizer of the empirical risk.  Therefore, EUAF can be directly applied to existing deep learning software in the same way as other popular activation functions (such as ReLU or Sigmoid). 
For further discussion on the computability of EUAF, one may refer to Section~\ref{sec:experimentation}, which provides experiments to explore the numerical properties of EUAF.
As opposed to the computability of EUAF, the  activation function proposed in  \cite{MAIOROV199981} is not computable in the sense that there is no numerical algorithm to evaluate the output and subgradient of the corresponding network. 

As we shall see later in the proof of Theorem~\ref{thm:main}, our EUAF network may require sufficiently large parameters to achieve an arbitrarily small error.
The magnitude of network parameters in Theorem~\ref{thm:main} can be dramatically reduced by increasing the network size. 
In particular, if we replace each elemental block like Figure~\ref{fig:magnitude:reduction}(a) by a block like Figure~\ref{fig:magnitude:reduction}(b), then the magnitude of parameters can be roughly reduced to its square root.
By repeatedly applying this idea, it is easy to prove that the magnitude of parameters can be exponentially reduced as the network size increases linearly. If we fix the size of these larger networks and only tune their parameters, they can still approximate high-dimensional continuous functions within an arbitrarily small error. How to fix a network size to balance the number of parameters and their memory depends on both the computer hardware and software. The goal of this paper is to demonstrate the existence of a simple network with a fixed size achieving an arbitrary error in spite of the magnitude of parameters and we have shown that the network size can be as small as $\calO(d^2)$. 
It is interesting to investigate the balance between the network size and the memory requirement in the future. 

\begin{figure}[htbp!]
    \centering
	\begin{subfigure}[b]{0.25\textwidth}
		\centering           \includegraphics[width=0.9868\textwidth]{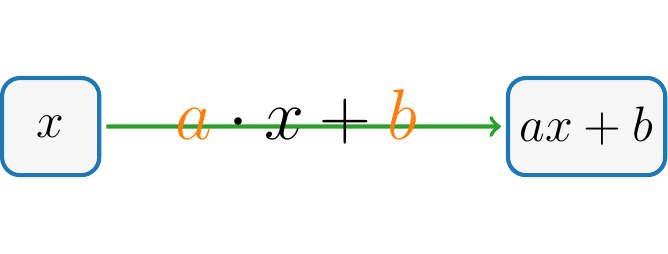}
		\subcaption{}
	\end{subfigure}
	\hspace*{21pt}
	\begin{subfigure}[b]{0.59\textwidth}
		\centering            \includegraphics[width=0.9868\textwidth]{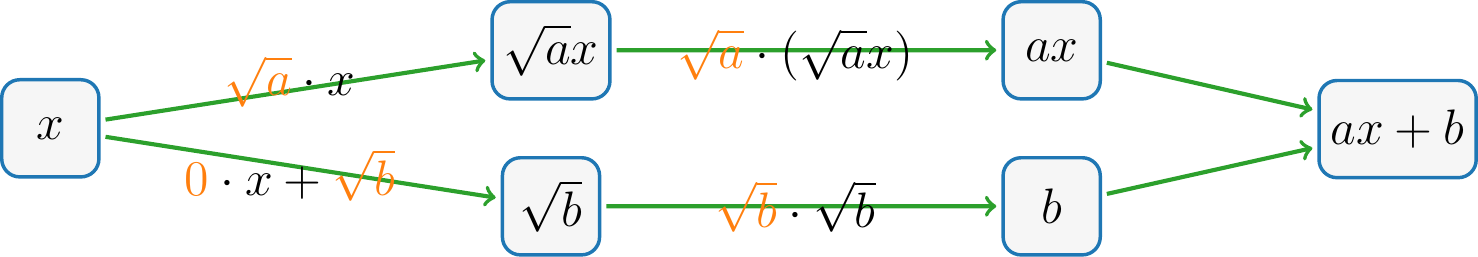}
		\subcaption{}
	\end{subfigure}
	\caption{Illustrations of the magnitude reduction of parameters for a sub-network. The parameters are marked in orange. Without loss of generality, $a\gg 1$ and $b \gg 1$. (a) Return $ax+b$ via two large parameters $a$ and $b$. (b)  Return $ax+b$ via several small parameters bounded by $\max\{\sqrt{a},\sqrt{b}\}$.}
	\label{fig:magnitude:reduction}
\end{figure}

In real-world applications, the parameters of the EUAF network are learned from the samples of the target function, which involves sophisticated numerical optimization. We refer to the learnability of network parameters as the existence of a numerical optimization algorithm that can identify network parameters to achieve a target approximation error. The computability of the EUAF networks does not imply learnability, which involves approximation, optimization, and generalization error analyses. 
The result in this paper shows that there exist computable EUAF networks achieving an arbitrarily small approximation error.  This means the learnability of the best approximation is reduced to achieving small generalization and optimization errors, which depend on the given data, the empirical risk model,  and the optimization algorithm. Therefore, whether or not EUAF networks would be useful in real-world applications also depends on optimization and generalization, which is out of the scope of this paper. The optimization and generalization error analyses of practical deep neural networks including EUAF networks is a challenging problem. To the best of our knowledge, there is no complete error analysis to address the learnability of neural networks with nonlinear activation functions.

\vspace{0.2cm}
The rest of this paper is organized as follows. In Section~\ref{sec:notation:idea}, we first summarize notations used in this paper and then discuss the ideas of proving Theorem~\ref{thm:main}. 
Section~\ref{sec:experimentation} focuses on numerical experimentation of EUAF, which acts as a proof of concept to explore the numerical properties of EUAF.
Next, several UAFs with better properties are proposed in Section~\ref{sec:UAFs}.
After that, we use several sections to present the complete proofs of Theorems~\ref{thm:main} and \ref{thm:main:classification}.
In Section~\ref{sec:proof:main:and:classification},
by assuming Theorem~\ref{thm:main:d=1} is true,
we give the detailed proofs of Theorems~\ref{thm:main} and \ref{thm:main:classification}.
Theorem \ref{thm:main:d=1} is proved in Section~\ref{sec:proof:main:d=1} based on Proposition~\ref{prop:dense}, the proof of which can be found in Section~\ref{sec:proof:prop:dense}.
Finally, Section~\ref{sec:conclusion} concludes this paper with a short discussion.

\section{Notations and Proof Ideas}
\label{sec:notation:idea}

In this section, we first  summarize notations used in this paper and then discuss the ideas of proving Theorem~\ref{thm:main}.

\subsection{Notations}
\label{sec:notation}

    Let us summarize all basic notations used in this paper as follows.
\begin{itemize}
    \item Let $\R$, $\Q$, and $\Z$ denote the set of real numbers, rational numbers, and integers, respectively.
    
    \item Let $\N$ and $\N^+$ denote the set of natural numbers and positive natural numbers, respectively.  That is,
    $\N^+=\{1,2,3,\cdots\}$ and $\N=\N^+\bigcup\{0\}$.
    
    \item For any $x\in \R$, let $\lfloor x\rfloor:=\max \{n: n\le x,\ n\in \Z\}$ and $\lceil x\rceil:=\min \{n: n\ge x,\ n\in \Z\}$.
    
    \item Let $\one_{S}$ be the indicator (characteristic) function of a set $S$, i.e., $\one_{S}$ is equal to $1$ on $S$ and $0$ outside $S$.

    
    \item The set difference of two sets $A$ and $B$ is denoted by $A\backslash B:=\{x:x\in A,\ x\notin B\}$. 

    \item Matrices are denoted by bold uppercase letters. For instance,  $\bm{A}\in\mathbb{R}^{m\times n}$ is a real matrix of size $m\times n$, and $\bm{A}^T$ denotes the transpose of $\bm{A}$.  
    Vectors are denoted as bold lowercase letters. For example, 
    $\bm{v}=[v_1,v_2,\cdots,v_d]^T=\scalebox{0.66}{$\left[\begin{array}{c}
    	v_1  \\
    	v_2 \\
    	\vdots \\
    	v_d
    \end{array}\right]$}\in\R^d$
    is a column vector. 
    Besides, ``$[$'' and ``$]$''  are used to  partition matrices (vectors) into blocks, e.g., $\bmA=\left[\begin{smallmatrix}\bmA_{11}&\bmA_{12}\\ \bmA_{21}&\bmA_{22}\end{smallmatrix}\right]$.

    \item For any $p\in [1,\infty)$, the $p$-norm (or $\ell^p$-norm) of a vector $\bmx=[x_1,x_2,\cdots,x_d]^T\in\R^d$ is defined by 
    \begin{equation*}
    	\|\bmx\|_p=\|\bmx\|_{\ell^p}\coloneqq \big(|x_1|^p+|x_2|^p+\cdots+|x_d|^p\big)^{1/p}.
    \end{equation*}
    In the case $p=\infty$, \begin{equation*}
    	\|\bmx\|_{\infty}=\|\bmx\|_{\ell^\infty}\coloneqq \max\big\{|x_i|: i=1,2,\cdots,d\big\}.
    \end{equation*}

    \item For any $a_1,a_2,\cdots,a_J\in\R$, we say $a_1,a_2,\cdots,a_J$ are \textbf{rationally independent} if they are linearly independent over the rational numbers $\Q$. That is, if there exist $\lambda_1,\lambda_2,\cdots,\lambda_J\in\Q$ such that $\sum_{j=1}^J \lambda_j\cdot a_j=0$, then $\lambda_1=\lambda_2=\cdots=\lambda_J=0$. For a simple example, $1$, $\sqrt{2}$, and $\sqrt{3}$ are rationally independent.

     \item An \textbf{algebraic} number is any complex number (including real numbers) that
    is a root of a polynomial equation with rational coefficients,  i.e., $\alpha$ is an algebraic number if and only if there exist
    $\lambda_0,\lambda_1,\cdots,\lambda_J\in \Q$ with  $\sum_{j=0}^J \lambda_j\alpha^j=0$.\footnote{For simplicity, we denote $1=x^0$ for any $x\in\R$, including the case $0^0$.} Denote the set of all algebraic numbers by $\A$. We say a complex number is \textbf{transcendental} if it is not in $\A$. 
    The set $\A$ is countable, and, therefore, almost all numbers are transcendental.
    The best known transcendental numbers are $\pi$ (the ratio of a circle's circumference to its diameter) and $e$ (the natural logarithmic base).
     
     \item The expression ``a network (architecture) with width $N$ and depth $L$'' means
             \begin{itemize}
             	\item The number of neurons in each \textbf{hidden} layer of this network (architecture) is no more than $N$.
             	\item The number of \textbf{hidden} layers of this network (architecture) is no more than $L$.
             \end{itemize} 
\end{itemize}

\subsection{Key Ideas of Proving Theorem~\ref{thm:main}}
\label{sec:key:ideas}

The proof of Theorem~\ref{thm:main} has two main steps: 1) prove the one-dimensional case; 2) reduce the $d$-dimensional approximation to the one-dimensional case via KST \cite{Kol1957}. In fact, in the case of $d=1$, the size of the network in Theorem~\ref{thm:main} can be further reduced as shown in Theorem~\ref{thm:main:d=1} below. Theorem~\ref{thm:main:d=1} is actually an enhanced version of Theorem~\ref{thm:main} and hence implies Theorem~\ref{thm:main}  in the case $d=1$.
\begin{theorem}\label{thm:main:d=1}
	Let  $f\in C([a,b])$ be a continuous function. Then,  for an arbitrary  $\varepsilon>0$,
	there exists a function $\phi$ generated by an EUAF network with width $36$ and depth $5$ such that
	\begin{equation*}
		|\phi(x)-f(x)|<\varepsilon\quad \tn{for any $x\in [a,b]\subseteq\R$.}
	\end{equation*}
\end{theorem}

The detailed proof of Theorem~\ref{thm:main:d=1} can be found in Section~\ref{sec:proof:main:d=1}. The main ideas of proving Theorem~\ref{thm:main:d=1}  are developed from some ideas of  our early works \cite{shijun4,shijun5}.
Roughly speaking,  we eventually convert a function approximation problem in an interval (e.g., $[0,1)$) to a point-fitting problem via the composition architecture of neural networks in the following three main steps.\footnote{The goal of a point-fitting problem is to identify a function $\phi:\R^d\to\R$ in a given hypothesis space (e.g., the space of functions realized by neural networks) such that $|\phi(\bmx_i)- y_i|< \varepsilon$ for $i=1,2,\cdots,n$ and a pre-specified error $\varepsilon>0$, where $\{(\bmx_i,y_i)\}_{i=1}^{n}\subseteq \R^{d+1}$ are given samples.} 
\begin{itemize}
	\item  Divide $[0,1)$ into small intervals $\calI_k=[\tfrac{k-1}{K},\tfrac{k}{K})$  with a left endpoint $x_k$ for $k\in\{1,2,\cdots,K\}$, where $K$ is an integer determined by the given error and the target function $f$. 
	
	\item Construct a sub-network to generate a function $\phi_1$ mapping the whole interval $\calI_k$ to $k$ for each $k$. The floor function $\lfloor\cdot\rfloor$ is a good choice to implement this step. Precisely, we can define $\phi_1(x)=\lfloor Kx\rfloor$.  
	The floor function is not continuous and has zero-derivative almost everywhere. As we shall see later, $\sigma_1$ (or $\sigma$) can be a continuous alternative to implement this step, but the construction is more complicated.
	
	\item The final step is to design another sub-network to generate a function $\phi_2$ mapping $k$ approximately to $f(x_k)$ for each $k$. Then $\phi_2\circ\phi_1(x)=\phi_2(k)\approx f(x_k)\approx f(x)$ for any $x\in \calI_k$ and $k\in\{1,2,\cdots,K\}$, which implies $\phi_2\circ\phi_1\approx f$ on $[0,1)$. After the above two steps, we simplify the approximation problem to a point-fitting problem, where $k$ is approximately mapped  to $f(k)$. 
	This step is the bottleneck of the construction in our previous papers \cite{shijun4,shijun5}. Roughly speaking, the final approximation error is essentially determined by how many points we can fit using a neural network. 
\end{itemize}

For the second step,  the capacity to generate step functions with sufficiently many ``steps'' via a sub-network with a limited number of neurons plays an important role. The reproduced step functions can be considered as a  continuous version of the floor function ($\lfloor\cdot\rfloor$) in \cite{shijun4,shijun5},  which is a perfect step function with infinite ``steps''  that improves the approximation power of networks as shown in \cite{shijun4,shijun5}. 
The key ingredient in the third step of the proof of Theorem~\ref{thm:main:d=1} is
essentially a point-fitting problem with arbitrarily many points.  This requires the following proposition motivated by the well-known fact that an irrational winding on the torus is dense.
See Figure~\ref{fig:Er} for illustrations of such a fact.
Here, we propose a new point-fitting technique that can fit arbitrarily many points within an arbitrary error using  fixed-size neural networks.

	\begin{figure}[htbp!]
		\centering
		\begin{subfigure}[b]{0.24\textwidth}
			\centering			\includegraphics[width=0.975\textwidth]{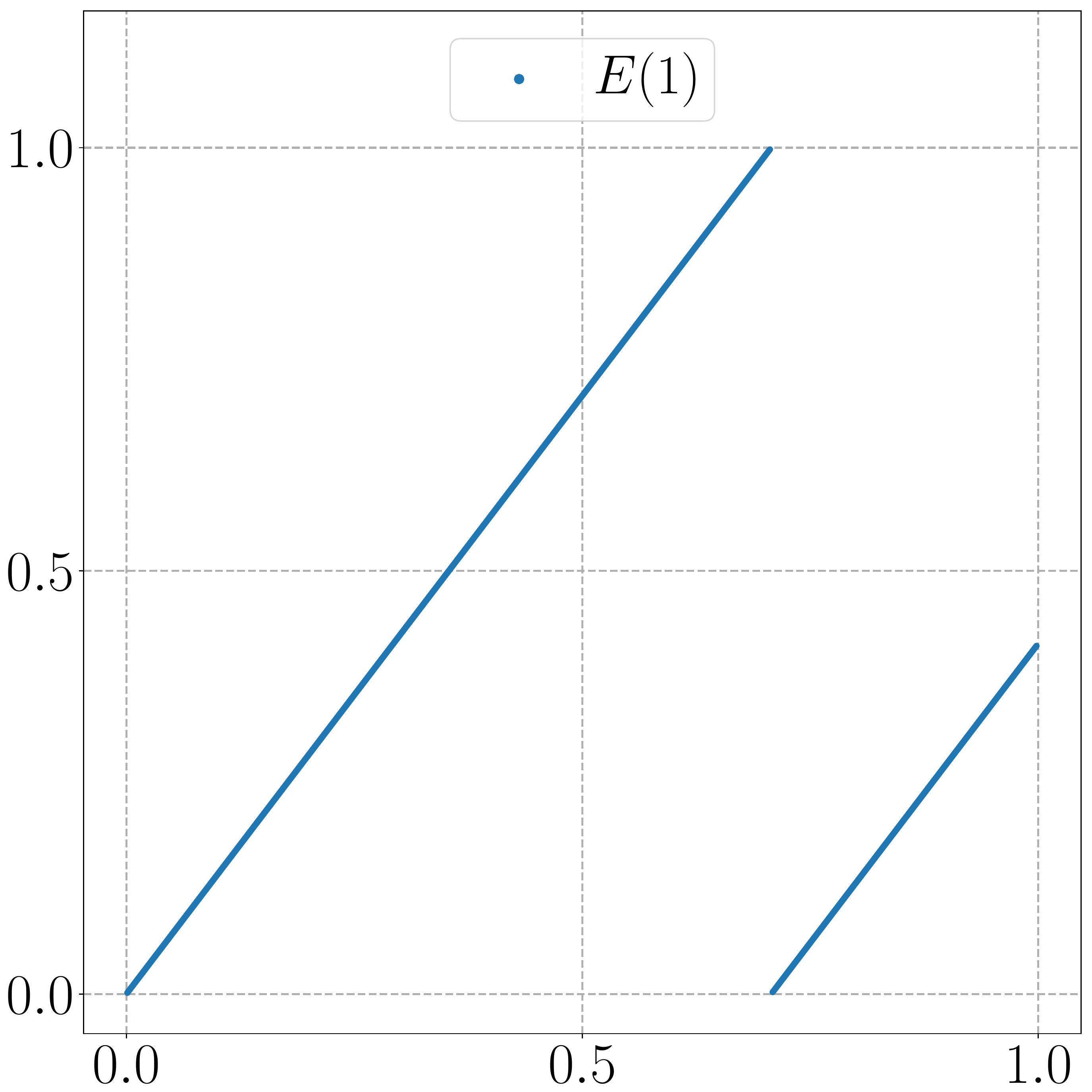}
		\end{subfigure}
		\begin{subfigure}[b]{0.24\textwidth}
			\centering            \includegraphics[width=0.975\textwidth]{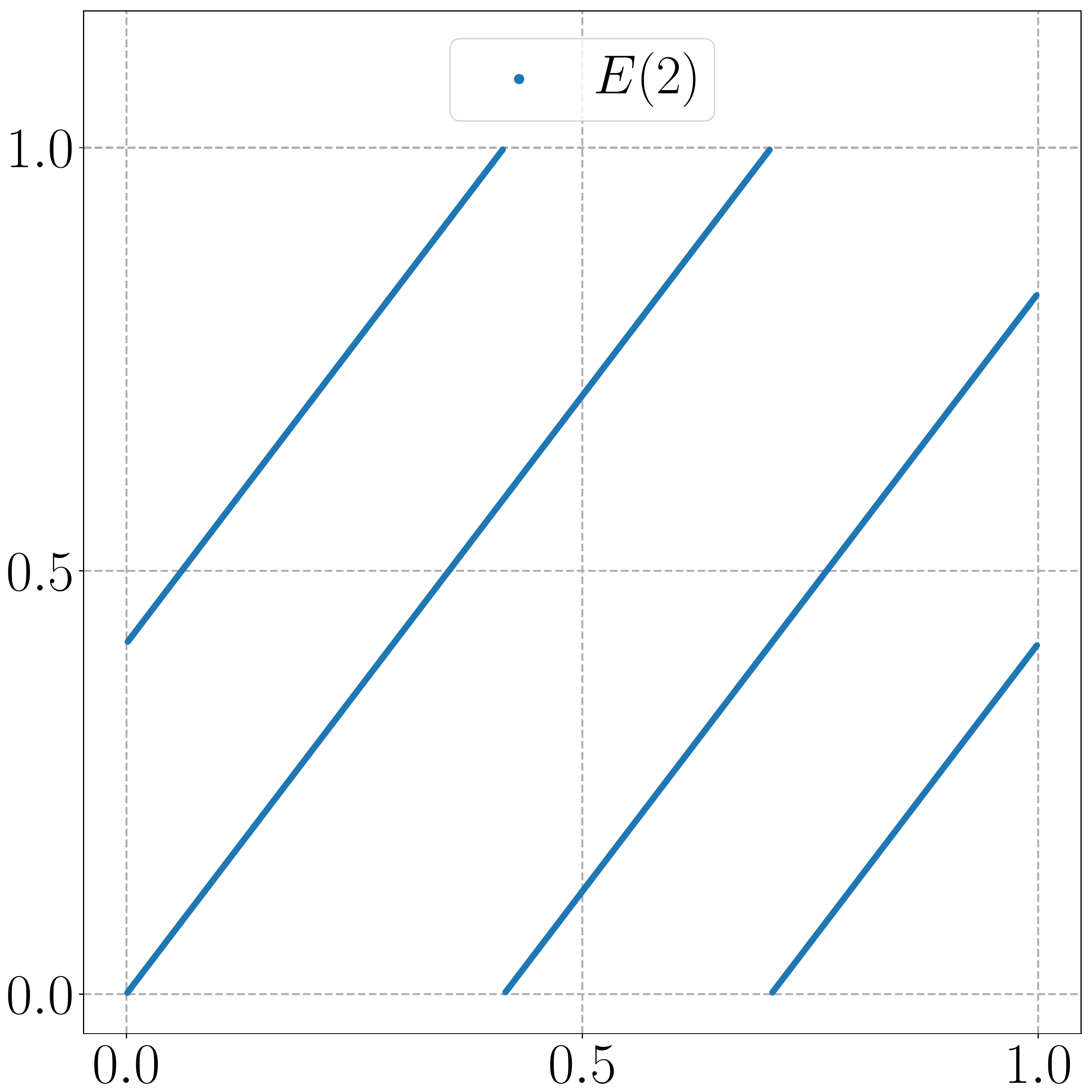}
		\end{subfigure}
		\begin{subfigure}[b]{0.24\textwidth}
			\centering           \includegraphics[width=0.975\textwidth]{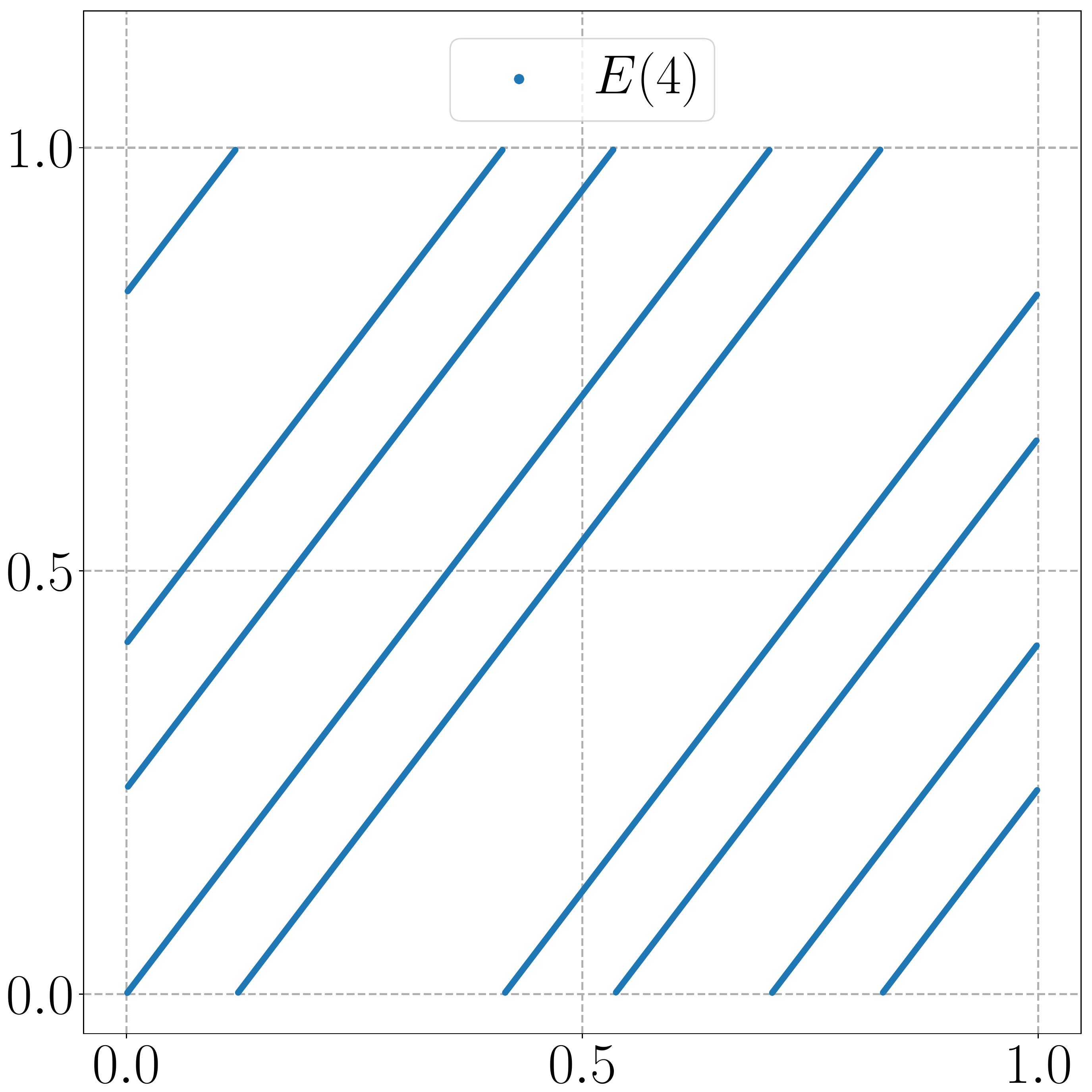}
		\end{subfigure}
		\begin{subfigure}[b]{0.24\textwidth}
			\centering            \includegraphics[width=0.975\textwidth]{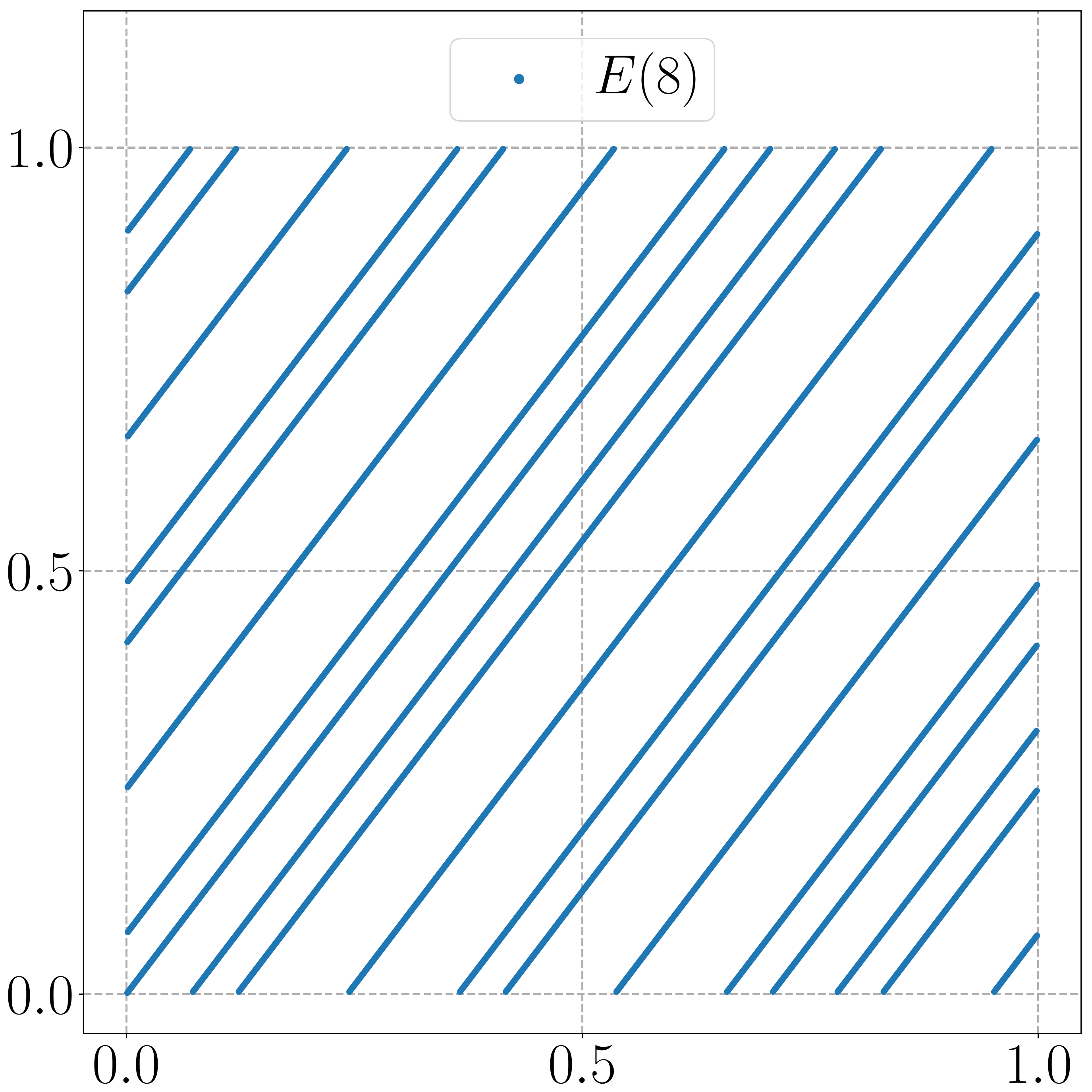}
		\end{subfigure}
		
		\vspace*{6pt}
		\begin{subfigure}[b]{0.24\textwidth}
			\centering			\includegraphics[width=0.975\textwidth]{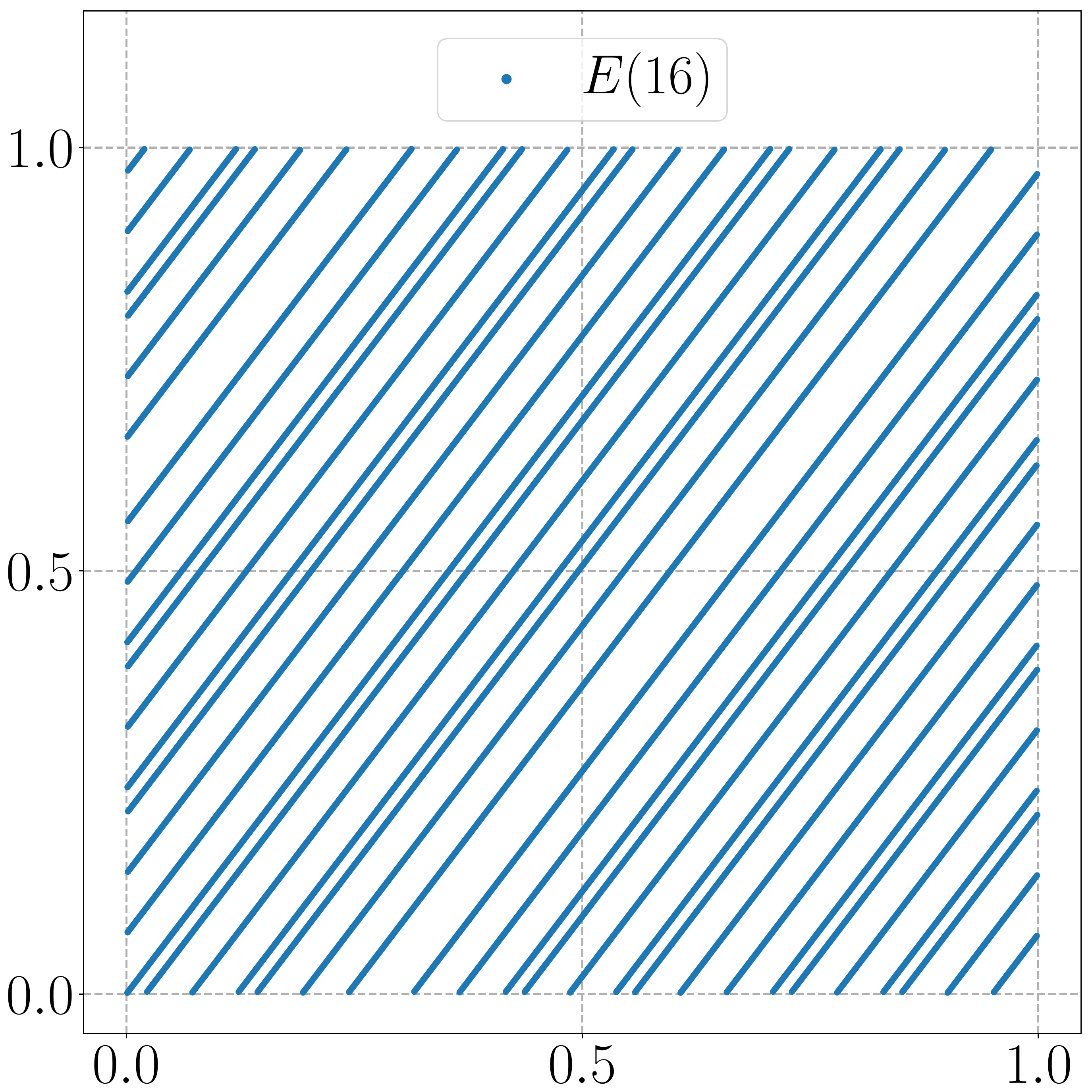}
		\end{subfigure}
		\begin{subfigure}[b]{0.24\textwidth}
			\centering            \includegraphics[width=0.975\textwidth]{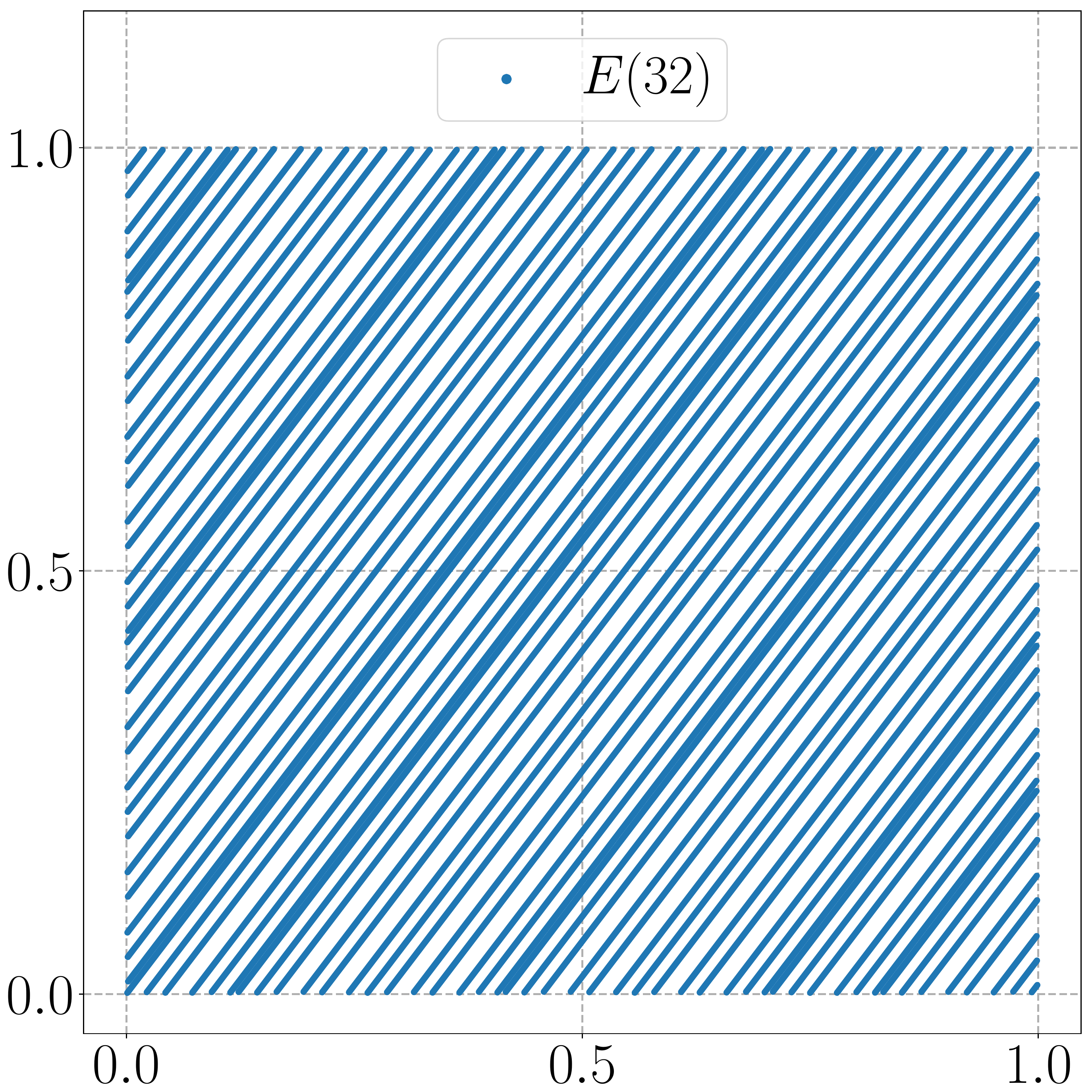}
		\end{subfigure}
		\begin{subfigure}[b]{0.24\textwidth}
			\centering           \includegraphics[width=0.975\textwidth]{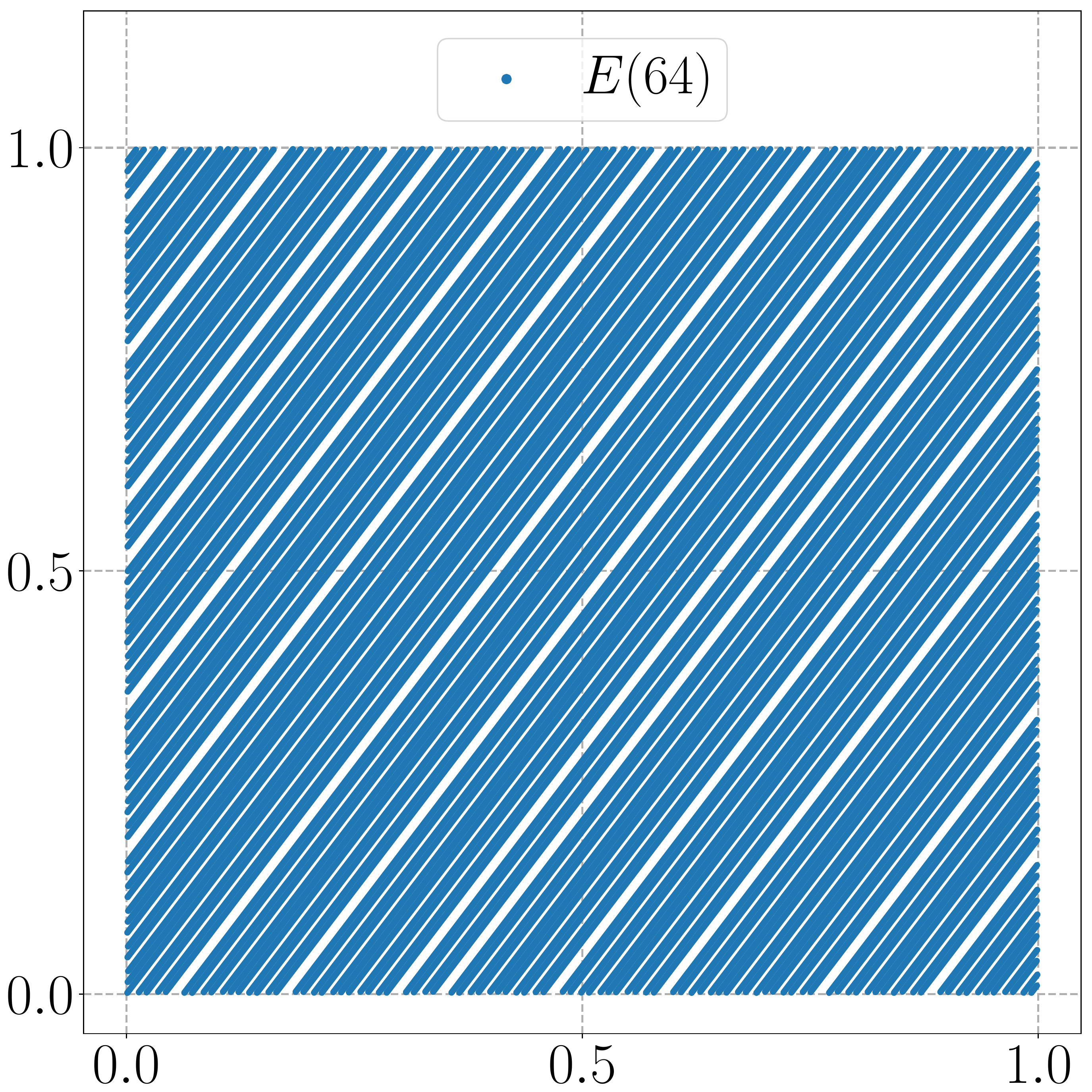}
		\end{subfigure}
		\begin{subfigure}[b]{0.24\textwidth}
			\centering            \includegraphics[width=0.975\textwidth]{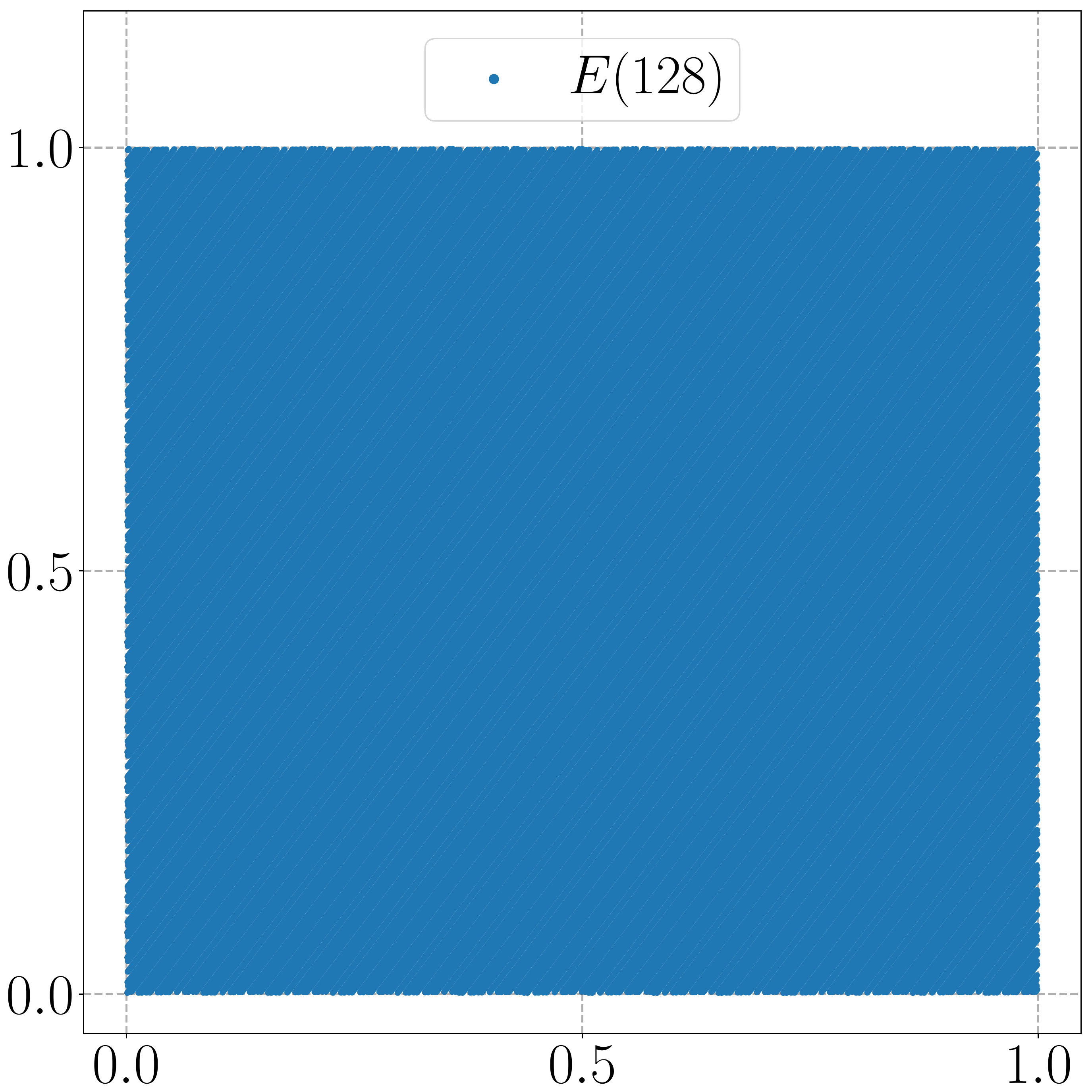}
		\end{subfigure}
		\caption{Illustrations of the denseness of $E(\infty)$ in $[0,1]^2$, where $E(r)$ is a  winding of an ``irrational''  direction $[1,\sqrt{2}]^T$ on $[0,r)$, i.e., $E(r)=\big\{[\tau(t),\tau(\sqrt{2}\,t)]^T:\, t\in[0,r)\big\}$ with $\tau(t)=t-\lfloor t\rfloor$. }
		\label{fig:Er}
	\end{figure}

\begin{proposition}\label{prop:dense}
	For any $K\in \N^+$, the following point set
	\begin{equation*}
	\Big\{\big[\sigma_1(\tfrac{w}{\pi+1}),\   
	\sigma_1(\tfrac{w}{\pi+2}),\   \cdots,\  
	\sigma_1(\tfrac{w}{\pi+K})\big]^T\,  :\, 
	w\in\R   \Big\}\subseteq [0,1]^K
\end{equation*}
is dense in $[0,1]^K$, where $\pi$ is the ratio of a circle's circumference to its diameter.
\end{proposition}

The proof of Proposition~\ref{prop:dense} can be found in Section~\ref{sec:proof:prop:dense}. To prove the  denseness in Proposition~\ref{prop:dense}, we borrow some ideas from transcendental number theory and Diophantine approximations in number theory. 
The number $\pi$ used in Proposition~\ref{prop:dense} is transcendental. It can be replaced by any other transcendental number.

Proposition~\ref{prop:dense} implies that for any given sample points $(k,y_k)\in\R^2$ with $y_k\in[0,1]$ for $k=1,2,\cdots,K$ and any $K\in\N^+$, there exists $w_0\in\R$ such that the function $x\mapsto\sigma_1(\tfrac{w_0}{\pi +x})$ can fit the  points $(k,y_k)\in\R^2$ for $k=1,2,\cdots,K$ within an arbitrary pre-specified error $\varepsilon>0$. To put it another way, for any $\varepsilon>0$, there exists $w_0\in\R$ such that $|\sigma_1(\tfrac{w_0}{\pi +k})-y_k|<\varepsilon$ for all $k$.

As we shall see later in the proof of Proposition~\ref{prop:dense},
the key point  is the periodicity of the outer function
$\sigma_1$. Of course, the inner function $x\mapsto\tfrac{w_0}{\pi +x}$ is also necessary since it helps to adjust sample points for $x=1,2,\cdots,K$. In fact, 
the inner function $x\mapsto\tfrac{w_0}{\pi +x}$ can be regarded as a variant of $\sigma_2$ via scaling and shifting. The periodicity has been explored to improve neural network  approximation in the literature, e.g.  the sine function in \cite{yarotsky:2019:06} is periodic and the floor function ($\lfloor\cdot\rfloor$) in \cite{shijun4,shijun5} is implicitly periodic because $x-\lfloor x\rfloor$ is periodic. 
We remark that a similar result holds if we replace $\sigma_1$  by a non-trivial periodic function and replace the sample locations
$x=1,2,\cdots,K$ by
distinct rational numbers $r_1,r_2,\cdots,r_K\in\Q$. See Section~\ref{sec:proof:prop:dense} for a further discussion.

 Theorem \ref{thm:main:d=1} essentially proves Theorem~\ref{thm:main} for  the univariate  case.  To prove the general case, we need the Kolmogorov superposition theorem (KST) \cite{Kol1957} given below to reduce a multivariate problem to a one-dimensional case.  
\begin{theorem}[KST]
	\label{thm:kst}
	There exist continuous functions $h_{i,j}\in C([0,1])$ for $i=0,1,\cdots,2d$ and $j=1,2,\cdots,d$ such that any continuous function $f\in C([0,1]^d)$ can be represented as 
	\begin{equation*}
		f(\bmx)=\sum_{i=0}^{2d}  g_i \Big(\sum_{j=1}^d h_{i,j}(x_j)\Big)\quad \tn{for any $\bmx=[x_1,x_2,\cdots,x_d]^T\in [0,1]^d$,}
	\end{equation*}
	where $g_i:\R\to\R$ is a continuous function for each $i\in \{0,1,\cdots,2d\}$.
\end{theorem}

KST is often used to reduce a multidimensional problem to a one-dimensional one. In fact, the compositional representation in KST  can be regarded as a special neural network with (complicated) activation functions depending on the target function, which makes KST useless in practical computation. To avoid this dependency, an activation function was designed in \cite{MAIOROV199981} to construct neural network representations with $\calO(d)$ neurons that can approximate functions in $C([-1,1]^d)$ within an arbitrary error.  Let us briefly summarize the main ideas in \cite{MAIOROV199981}: 1) Identify a dense and countable subset $\{u_k\}_{k=1}^\infty$ of $C([-1,1])$, e.g., polynomials with rational coefficients. 
2) Construct an activation function $\varrho$ to encode all $u_k(x)$ for $x\in[-1,1]$. In fact, for each $k$, $u_k|_{[-1,1]}$ is ``stored'' in  $\varrho$ on $[4k,4k+2]$, and the values of $\varrho$ on $[4k+2,4k+4]$ are properly assigned to make $\varrho$ a smooth and monotonically increasing function.
That is, let $\varrho(x+4k+1)=a_k+b_k x+c_k u_k(x)$ for any $x\in [-1,1]$ with carefully chosen constants $a_k$, $b_k$, and $c_k\neq 0$ such that $\varrho(x)$ can be a sigmoidal function. 
3)  For any $g\in C([-1,1])$, there exists a one-hidden-layer $\varrho$-activated network with width $3$ approximating $g$ within an arbitrary error $\delta>0$, i.e., there exists $k$ such that $g(x)\mathop{\approx}\limits^{\delta} u_k(x)= \tfrac{\varrho(x+4k+1)-a_k-b_k x}{c_k}$ for any $x\in [-1,1]$.  
4) Replace the inner and outer functions in KST with these one-hidden-layer networks to achieve a two-hidden-layer $\varrho$-activated network with width $\calO(d)$ to approximate  $f\in C([-1,1]^d)$ within an arbitrary error $\varepsilon>0$.
 As we can see, the key point of the construction in \cite{MAIOROV199981} is to encode a dense and countable subset of the target function space in  an  activation function.

 Note that both \cite{MAIOROV199981} and this paper use KST  to reduce dimension.  However, the activation function of \cite{MAIOROV199981} is complicated without any closed form and there is no efficient numerical algorithm to evaluate it. After encoding a dense subset of continuous function into a single but complicated activation function, one only needs to construct affine linear transformations to select appropriate functions of this dense subset from this complicated activation function to construct approximation. Hence, such a complicated activation function simplifies the proof of the denseness, since the denseness is encoded in the activation function. As a contrast, we design a simple activation function with efficient numerical implementation (see Figure~\ref{fig:sigma} for an illustration) achieving the universal approximation property with fixed-size networks, 
because simple and implementable activation functions are a basic requirement for a neural network to be used in applications. However, the proof of the denseness of a neural network generated by such a simple activation function becomes difficult.  A sophisticated analysis will be developed in the rest of this paper to overcome the difficulties.


\section{Experimentation}
\label{sec:experimentation}


In this section, we will conduct two simple experiments as a proof of concept to explore the numerical performances of the EUAF activation function.
Let us first discuss the numerical implementation of EUAF in \href{https://pytorch.org/}{PyTorch}.
To enable the automatic differentiation feature for EUAF, we need to implement EUAF based on PyTorch built-in functions.
With the following four built-in functions
$\tn{abs}(x)=|x|$, $\tn{floor}(x)=\lfloor  x\rfloor,$
\begin{equation*}
\tn{softsign}(x)=\frac{x}{|x|+1},\quad \tn{and}\quad
    \tn{sign}(x)=\begin{cases}
    1 &\tn{if} \ x>0,\\
    0 &\tn{if} \ x=0,\\
    -1 &\tn{if} \ x<0,\\
    \end{cases}
\end{equation*}
we can represent EUAF as
\begin{equation*}
    \begin{split}
            \tn{EUAF}(x)
            &=\begin{cases}
    \tn{softsign}(x)& \tn{if}\ x<0,\\
    \big|x-2\lfloor \tfrac{x+1}{2}\rfloor\big|& \tn{if}\ x\ge 0
    \end{cases}\\
    &=   \tn{softsign}(x)\cdot \frac{1-\tn{sign}(x)}{2} + 
    \Big|x-2\big\lfloor \frac{x+1}{2}\big\rfloor\Big|
    \cdot\frac{1+\tn{sign}(x)}{2}\\
     &=\tn{softsign}(x)\cdot \frac{1-\tn{sign}(x)}{2} +     \tn{abs}\Big(x-2\cdot\tn{floor}\big( \frac{x+1}{2}\big)\Big)
    \cdot\frac{1+\tn{sign}(x)}{2}.\\
    \end{split}
\end{equation*}
Thus, it is numerically cheap to compute EUAF and its subgradient. We believe the EUAF activation function can achieve good results in some real-world applications if proper optimization algorithms are developed for EUAF.
In this paper, we only 
conduct two simple experiments: a function approximation experiment in Section~\ref{sec:experimentation:func:approx} and a classification experiment in Section~\ref{sec:experimentation:classification}.

Next, let us briefly discuss when our EUAF activation function would outperform the practically used ones (e.g., ReLU, Sigmoid, and Softsign), which is based on full error analysis in Section~\ref{sec:error:analysis}.
In our discussion, we take the ReLU activation function as an example and suppose 
the optimization error is well-controlled.
Clearly, replacing ReLU by EUAF can reduce the approximation error, but would result in a large generalization error. Thus, we would expect that EUAF achieves better results than ReLU if the approximation error is larger than the generalization error. That means EUAF would outperform ReLU in the following two cases.
\begin{itemize}
    \item The approximation error is pretty large (e.g.,  the target function is sufficiently complicated).
    \item The generalization error is well-controlled (e.g., there are sufficiently many samples).
\end{itemize}
If a given problem does not belong to these two cases, one may consider replacing only a small number of ReLUs by EUAFs.
In the function approximation experiment in Section~\ref{sec:experimentation:func:approx}, we first choose a complicated target function and then generate sufficiently many samples to reduce the  generalization error. In the classification experiment in Section~\ref{sec:experimentation:classification}, we control the generalization error via three common methods: 
keeping network parameters small via L2 regularization, dropout \cite{DBLP:journals/corr/abs-1207-0580,JMLR:v15:srivastava14a}, and batch normalization \cite{10.5555/3045118.3045167}.

\subsection{Function Approximation}
\label{sec:experimentation:func:approx}

We will design fully connected neural network (FCNN) architectures activated by ReLU or EUAF to solve a function approximation problem. To better compare the approximation power of ReLU and EUAF activation functions, we choose a complicated (oscillatory) function $f$ as the target function, where $f$ is defined as
\begin{equation*}
   f(x_1,x_2)\coloneqq 0.6\sin(8x_1)+0.4\sin(16x_2)\quad \tn{ for any $(x_1,x_2)\in [0,1]^2$.}
\end{equation*}


To compare the numerical performances of ReLU and EUAF activation functions, we design two FCNN architectures with different activation functions.
Both of them have $4$ hidden layers and each hidden layer has $80$ neurons.  For simplicity, we denote them as FCNN1 and FCNN2. See illustrations of them in Figure~\ref{fig:FCNN:arc}. FCNN1 is a standard fully connected ReLU network and FCNN2 can be regarded as a variant of FCNN1 by replacing ReLU by EUAF.

\begin{figure}[htbp!]
	\centering
	\begin{subfigure}[b]{0.4\textwidth}
		\centering            \includegraphics[width=0.98728\textwidth]{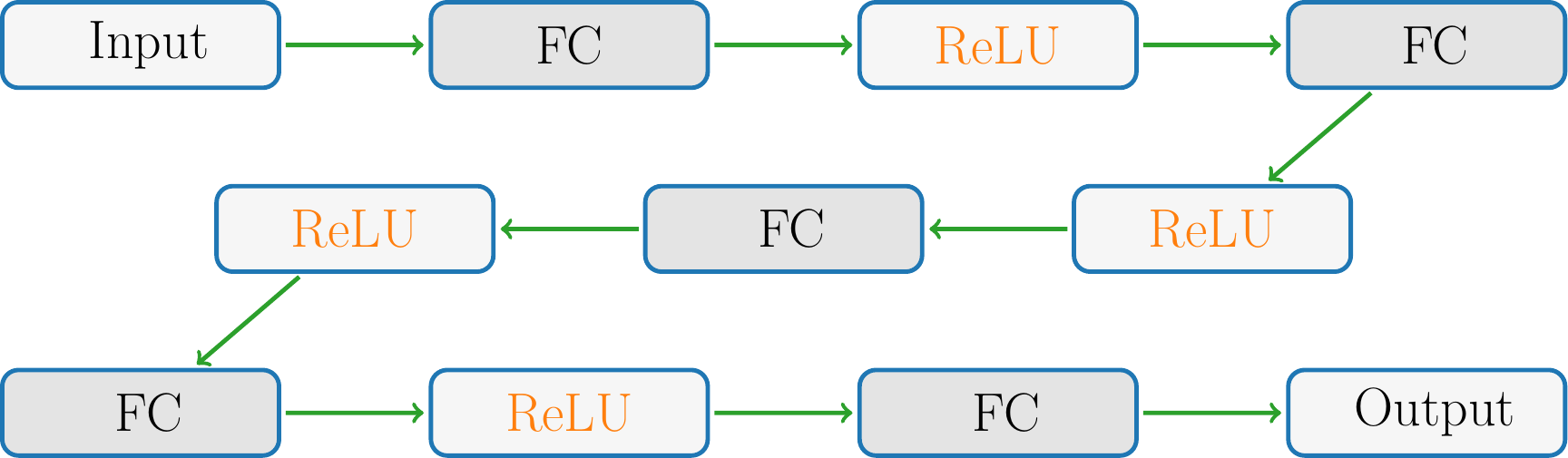}
		\subcaption{FCNN1.}
	\end{subfigure}
		\hspace{21pt}
	\begin{subfigure}[b]{0.4\textwidth}
		\centering           \includegraphics[width=0.98728\textwidth]{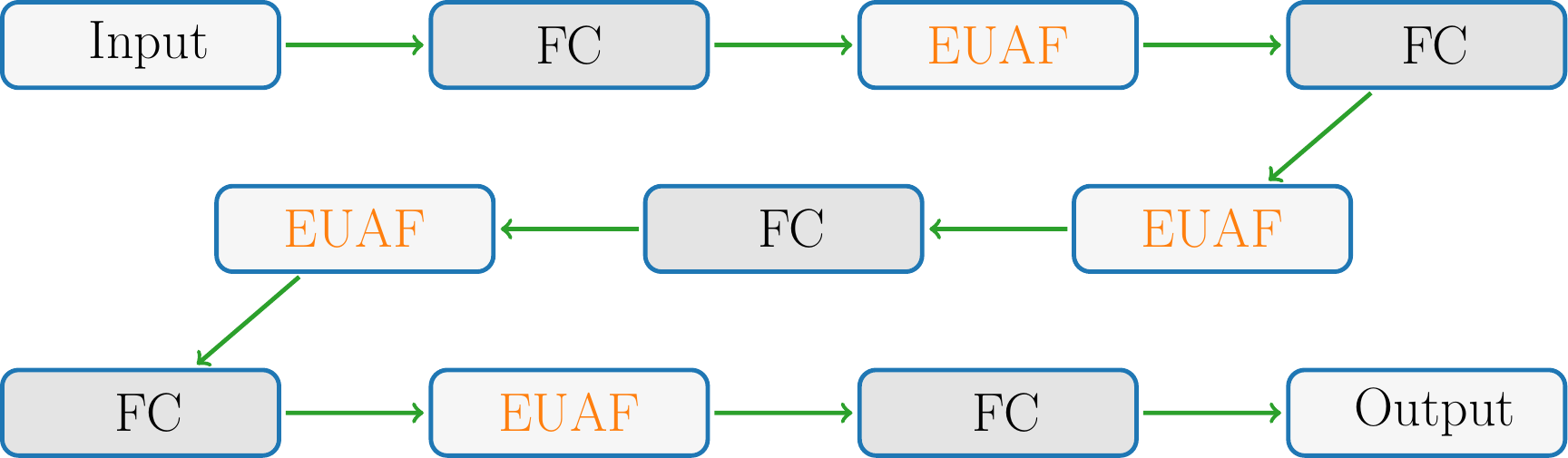}
		\subcaption{FCNN2.}
	\end{subfigure}
	\caption{Illustrations of FCNN1 and FCNN2. FC represents a fully connected layer.}
	\label{fig:FCNN:arc}
\end{figure}

Before presenting the numerical results, 
let us present the hyper-parameters for training FCNN1 and FCNN2. We randomly choose $10^6$ training samples and $10^5$ test samples in $[0,1]^2$.
The number of epochs and the batch size are set to $500$ and $256$, respectively.
We adopt 
RAdam \cite{Liu2020On}
as the
optimization method and the learning rate is $0.002\times0.9^{i-1}$ in epochs $5(i-1)+1$ to $5i$ for $i=1,2,\cdots,100$. 
Several loss functions are used to estimate the training and test losses, including the mean squared error (MSE), the mean absolute error (MAE), and the maximum (MAX) loss functions.
To illustrate MSE, MAE and MAX losses, we denote $\phi$ as the  network-generated function and $\bmx_1,\cdots,\bmx_m$ as the test samples ($m=10^5$ in our setting).
Then, the MSE loss is given by 
$\tfrac{1}{m}\sum_{i=1}^m \big(\phi(x_i)-f(x_i)\big)^2,$
the MAE loss is given by 
$	\tfrac{1}{m}\sum_{i=1}^m \big|\phi(x_i)-f(x_i)\big|,$
and the MAX loss is given by 
$\max\big\{ |\phi(x_i)-f(x_i)|: i=1,2,\cdots,m\big\}.$
The MSE loss is used in our training process.
In the settings above,  we repeat the experiment $12$ times and discard 2 top-performing and 2 bottom-performing trials by using the average of test losses (MSE) in the last 100 epochs as the performance criterion. For each epoch, we adopt the average of training (test) losses in the rest $8$ trials as the target training (test) loss.

Next, let us present the experiment results to compare the numerical performances of ReLU and EUAF activation functions. Training and test losses (MSE) over epochs  for FCNN1 and FCNN2 are summarized in Figure~\ref{fig:training:test:loss}. 



\begin{figure}[htbp!]
	\centering
 \includegraphics[width=0.631\textwidth]{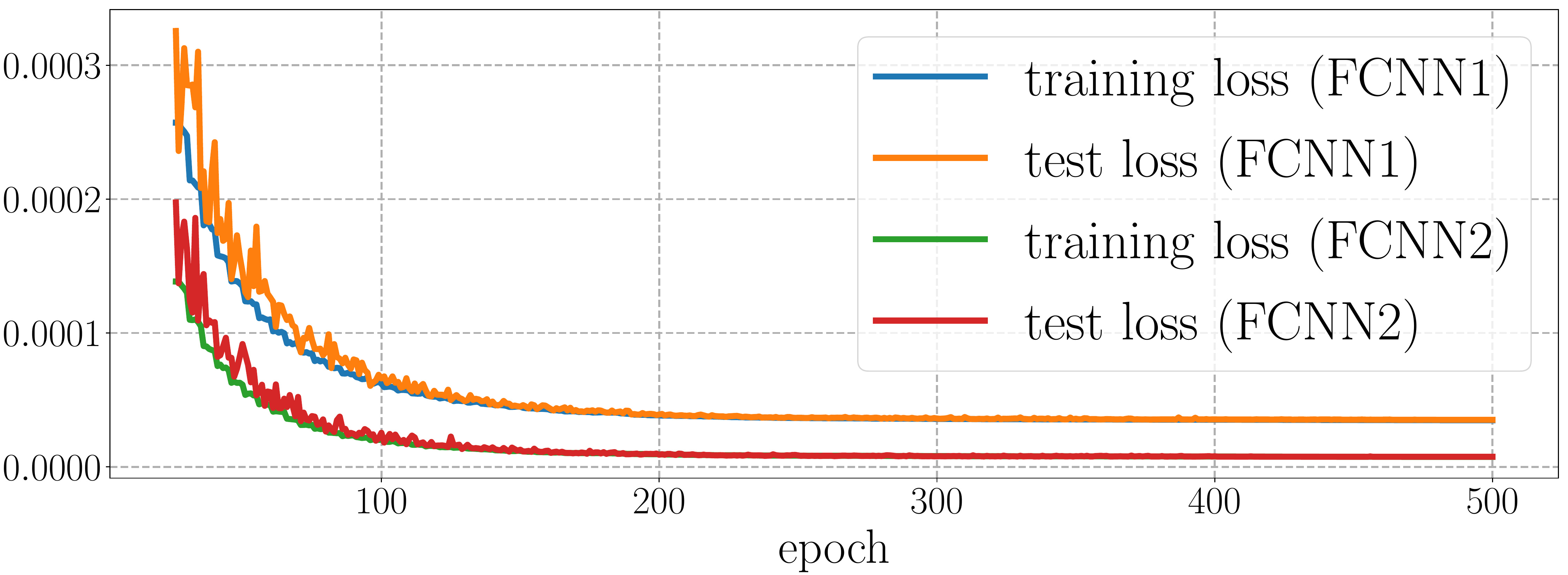}
	\caption{Training and test losses (MSE) in epochs 25-500 for FCNN1 and FCNN2. }
	\label{fig:training:test:loss}
\end{figure}

In Table~\ref{tab:loss:comparison}, we present a comparison
of FCNN1 and FCNN2 for the average of the test losses in the last 100 epochs  measured in several loss functions. As we can see from Figure~\ref{fig:training:test:loss} 
and Table~\ref{tab:loss:comparison}, FCNN2 performs better than FCNN1. That means replacing ReLU by EUAF would improve experiment results.

\begin{table}[htbp!]
	\centering  
	\resizebox{0.651\textwidth}{!}{ 
		\begin{tabular}{ccccccccc} 
			\toprule
			 &  \multirow{2}{*}{activation function}  &     \multicolumn{3}{c}{test loss} \\
			 \cmidrule{3-5}
			 & & MSE & MAE & MAX\\
			\midrule
			 \rowcolor{mygray}
			 FCNN1 & ReLU  &  $3.53\times 10^{-5}$ & $4.57\times 10^{-3}$ & $3.69\times 10^{-2}$\\	
			\midrule
			FCNN2 & EUAF  &  $7.56\times 10^{-6}$ & $2.13\times 10^{-3}$ &  $1.48\times 10^{-2}$ \\			
			\bottomrule
		\end{tabular} 
	}
\caption{Test loss comparison.}
	\label{tab:loss:comparison}
\end{table} 

\subsection{Classification}
\label{sec:experimentation:classification}

The goal of a classification problem with $J\in\N^+$ classes is to identify a classification function $f$ defined by 
\begin{equation*}
    f(\bmx)=j \quad \tn{for any $\bmx\in E_j$ and $j=0,1,\cdots,J-1$,}
\end{equation*}
where $E_0,E_1,\cdots,E_{J-1}$ are pairwise disjoint bounded closed subsets of $\R^d$ and all samples with a label $j$ are contained in $E_j$  for each $j$. Such a classification function $f$ can be continuously extended to $\R^d$, which means a classification problem can also be regarded as a continuous function approximation problem.
We take the case $J=2$ as an example to illustrate the extension. The multiclass case is similar. 
By defining
\begin{equation*}
	\tn{dist}(\bmx,E_i)\coloneqq \inf_{\bmy\in E_i}\|\bmx-\bmy\|_2 \quad \tn{for any $\bmx\in \R^d$ and $i=0,1$,}
\end{equation*} 
we have $\tn{dist}(\bmx,E_0)+\tn{dist}(\bmx,E_1)>0$ for any $\bmx\in \R^d$. Thus, we can define 
\begin{equation*}
	\widetilde{f}(\bmx)\coloneqq \frac{\tn{dist}(\bmx,E_0)}{\tn{dist}(\bmx,E_0)+\tn{dist}(\bmx,E_1)} \quad \tn{for any $\bmx\in \R^d$}.
\end{equation*}
It is easy to verify that $\widetilde{f}$ is continuous on $\R^d$ and 
\begin{equation*}
    \widetilde{f}(\bmx)=\begin{cases}
    0 &\tn{if}\  \bmx\in E_0,\\
    1 &\tn{if}\  \bmx\in E_1
    \end{cases}=f(\bmx)\quad \tn{ for any $\bmx\in E_0\bigcup E_1$.}
\end{equation*}
That means $\widetilde{f}$ is a continuous extension of $f$.
That means we can apply our theory to classification problems.

We will design convolutional neural network (CNN) architectures activated by ReLU or EUAF to solve a classification problem corresponding to a standard benchmark data set Fashion-MNIST \cite{2017arXiv170807747X}.
This data set 
consists of a training set of 60000 samples and a test set of 10000 samples. Each sample is a $28\times28$ grayscale image, associated with a label from 10 classes. 
To compare the numerical performances of ReLU and EUAF activation functions, we design two small CNN architectures with different activation functions.
Both of them have two convolutional layers and two fully connected layers. For simplicity, we denote them as CNN1 and CNN2. 
See illustrations of them in Figure~\ref{fig:CNN:arc}.
We present more details of CNN1 and CNN2 in Table~\ref{tab:CNN:arc}. 

\begin{figure}[htbp!]
	\centering
	\begin{subfigure}[b]{0.48\textwidth}
		\centering            \includegraphics[width=0.97268\textwidth]{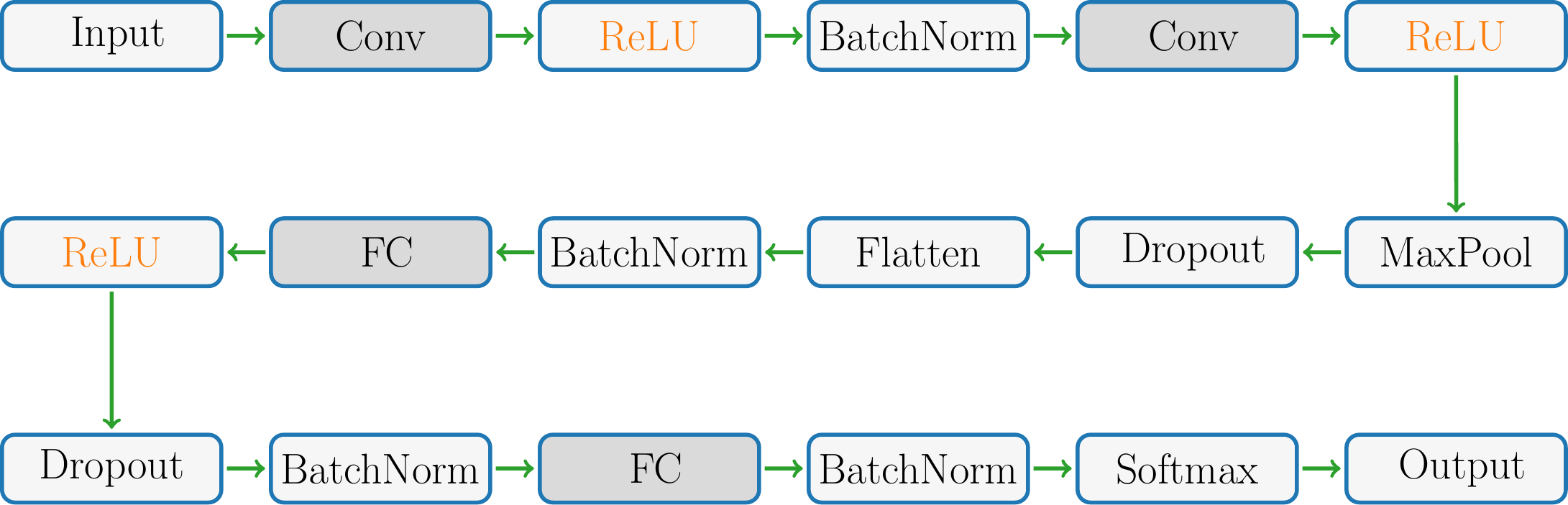}
		\vspace*{3pt}
		\subcaption{CNN1.}
	\end{subfigure}
		\hspace{3pt}
	\begin{subfigure}[b]{0.48\textwidth}
		\centering           \includegraphics[width=0.97268\textwidth]{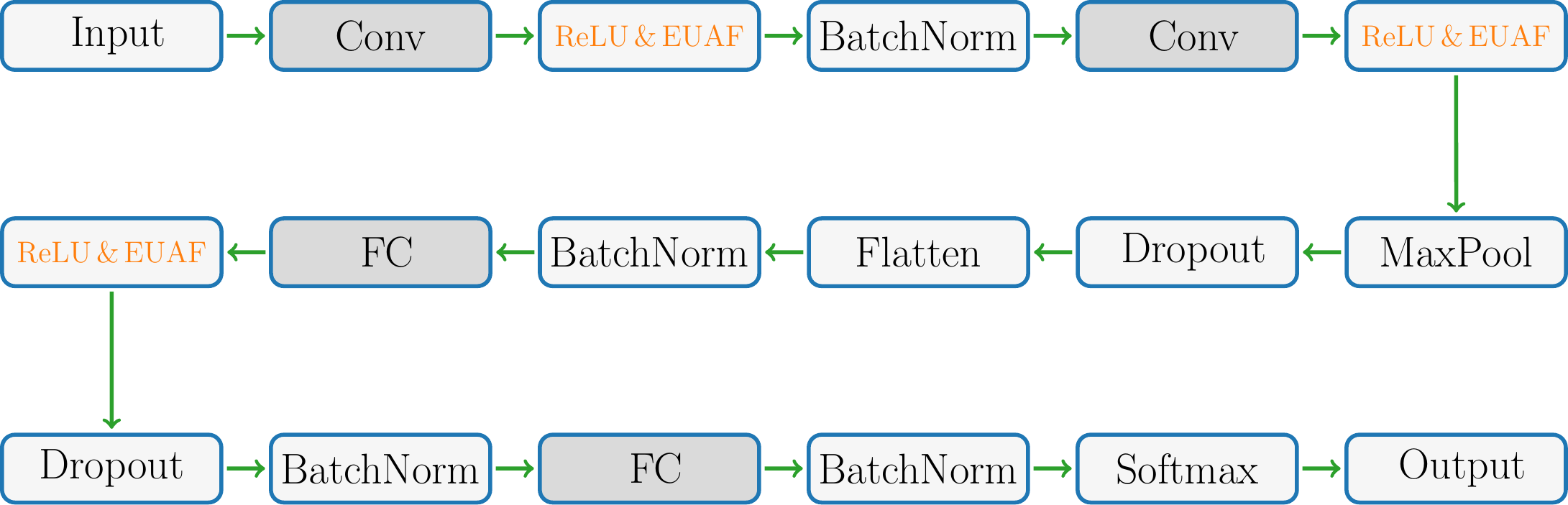}
		\vspace*{3pt}
		\subcaption{CNN2.}
	\end{subfigure}
	\caption{Illustrations of CNN1 and CNN2. Conv and FC represent convolutional  and fully connected layers, respectively.}
	\label{fig:CNN:arc}
\end{figure}

\begin{table}[H]
	\centering  
	\resizebox{0.988\textwidth}{!}{ 
		\begin{tabular}{ccccccccccc} 
			\toprule
			\multirow{2}{*}{layers}   &     
		\multicolumn{2}{c}{activation function} 
			& \multirow{2}{*}{output size of each layer} & \multirow{2}{*}{dropout} & \multirow{2}{*}{batch normalization }  \\
			\cmidrule{2-3}
			& CNN1 & CNN2 \\
			
			\midrule
			input $\in \R^{28\times 28}$  &
			\phantom{$\tn{ReLU},\hspace{6pt}  22\times(26\times26)$}& & $28\times 28$   \\
			
			\midrule
			
			Conv-1: $1\times (3\times 3), \, 24$ & ReLU  & $\begin{array}{cc}
			   \tn{EUAF},\hspace{6pt}  1\times (26\times26) \\
			    \tn{ReLU},\hspace{6pt}  23\times(26\times26)
			\end{array}$ & $24\times(26\times 26)$ &  & yes   \\
			
			\midrule
			
			Conv-2: $24\times (3\times 3), \, 24$ & ReLU & $\begin{array}{cc}
			   \tn{EUAF},\hspace{6pt}  1\times(24\times24) \\
			    \tn{ReLU},\hspace{6pt}  23\times(24\times24)
			\end{array}$ & $3456$ (MaxPool \& Flatten) &  $0.25$ & yes  \\
			
			\midrule
			
			FC-1: $3456, \, 48$ & ReLU & $\begin{array}{cc}
			   \tn{EUAF},\hspace{6pt}  1 \\
			    \tn{ReLU},\hspace{6pt}  47
			\end{array}$& $48$ & $0.5$  & yes   \\
				\midrule
			FC-2: $48, \, 10$ &  &    & $10$ (Softmax) &  & yes  \\
			
			\midrule
			
			output $\in \R^{10}$ &   \\
			
			\bottomrule
		\end{tabular} 
	}
	\caption{Details of CNN1 and CNN2.} 
	\label{tab:CNN:arc}
\end{table} 

CNN1 is activated by ReLU, 
while CNN2 is activated by ReLU and EUAF. In CNN2, only one channel (neuron) of a convolutional (fully connected) hidden layer is activated by EUAF.
CNN2 can be regarded as a variant of CNN1 by replacing a small number of ReLUs by EUAFs. 
This follows a natural question: Why do we not make all (or most) neurons (channels) of CNN2 activated by EUAF?
We use only a few EUAFs in CNN2 for two reasons listed below.
\begin{itemize}
    \item Since the number of available training samples is limited, using too many EUAF activation functions would lead to a large generalization error.
    \item The key difference of EUAF to the practical used activation functions (e.g.,  ReLU, Sigmoid, and Softsign) is the periodic part on $[0,\infty)$. As we shall see later in the proof of our main theorem, only a small number of neurons in the constructed network require the periodic property. Thus, we would expect that neural networks activated by the practical used activation functions and a few EUAFs are super expressive.
\end{itemize}

Next, let us discuss why we choose relatively small network architectures. Since the Fashion-MNIST  classification problem is simple, the expressive power of a relatively large ReLU CNN architecture is enough. That means there is no need to introduce EUAF if the network architecture is relatively large. 
We believe EUAF would be useful for complicated classification problems.

We remark that we use CNNs to approximate an equivalent variant $\hatbmf$ of the original classification function $f$ mentioned previously, where $\hatbmf$ is given by \begin{equation*}
    \hatbmf(\bmx)=\bme_j \quad \tn{for any $\bmx\in E_j$ and $j=0,1,\cdots,J-1$,}
\end{equation*}
where $\{\bme_1,\bme_2,\cdots,\bme_J\}$ is the standard basis of $\R^J$, i.e., $\bme_j\in\R^J$ denotes the vector with a $1$ in the $j$-th coordinate and $0$'s elsewhere.

Before presenting the numerical results, 
let us present the hyper-parameters for training two CNN architectures above. 
We use the cross-entropy loss function to evaluate the loss. 
The number of epochs and the batch size are set to $500$ and $128$, respectively.
We adopt 
RAdam \cite{Liu2020On}
as the
optimization method. The weight decay of the optimizer is $0.0001$ and the learning rate is $0.002\times0.9^{i-1}$ in epochs $5(i-1)+1$ to $5i$ for $i=1,2,\cdots,100$. 
All training (test) samples in the Fashion-MNIST data set are standardized in our experiment, i.e., we rescale all training (test) samples to have a mean of $0$ and a standard deviation of $1$.
In the settings above,  we repeat the experiment $48$ times and discard $8$ top-performing and $8$ bottom-performing trials by using the average of test accuracy in the last 100 epochs as the performance criterion. For each epoch, we adopt the average of test accuracies  in the rest $32$ trials as the target test accuracy.

Let us present the experiment results to compare the numerical performances of  CNN1 and CNN2.
The test accuracy comparison of CNN1 and CNN2 is summarized in Table~\ref{tab:accuracy:comparison}. 

\begin{table}[H]
	\centering  
	\resizebox{0.95\textwidth}{!}{ 
		\begin{tabular}{ccccccccc} 
			\toprule
			 & activation function &  largest accuracy & average of largest 100 accuracies & average accuracy in last 100 epochs \\
			\midrule
		\rowcolor{mygray}	CNN1 & ReLU & 0.933066 & 0.932852 & 0.932698 \\			
			\midrule
			 CNN2 & ReLU and EUAF & 0.933922 & 0.933685 &  0.933508  \\
			\bottomrule
		\end{tabular} 
	}
	\caption{Test accuracy comparison.} 
	\label{tab:accuracy:comparison}
\end{table} 

For each of CNN1 and CNN2, we present
the largest test accuracy,  the average of largest $100$ test accuracies over epochs, and the average of test accuracies in the last 100 epochs. 
For an intuitive comparison, we also provide illustrations of the test accuracy over epochs for CNN1 and CNN2 in Figure~\ref{fig:accuracy:comparison}.
As we can see from Table~\ref{tab:accuracy:comparison} and Figure~\ref{fig:accuracy:comparison}, CNN2 performs better than CNN1. That means replacing a small number of ReLUs by EUAFs would improve the experiment results.

\begin{figure}[htbp!]
	\centering
	\begin{subfigure}[b]{0.48\textwidth}
		\centering            \includegraphics[width=0.9768\textwidth]{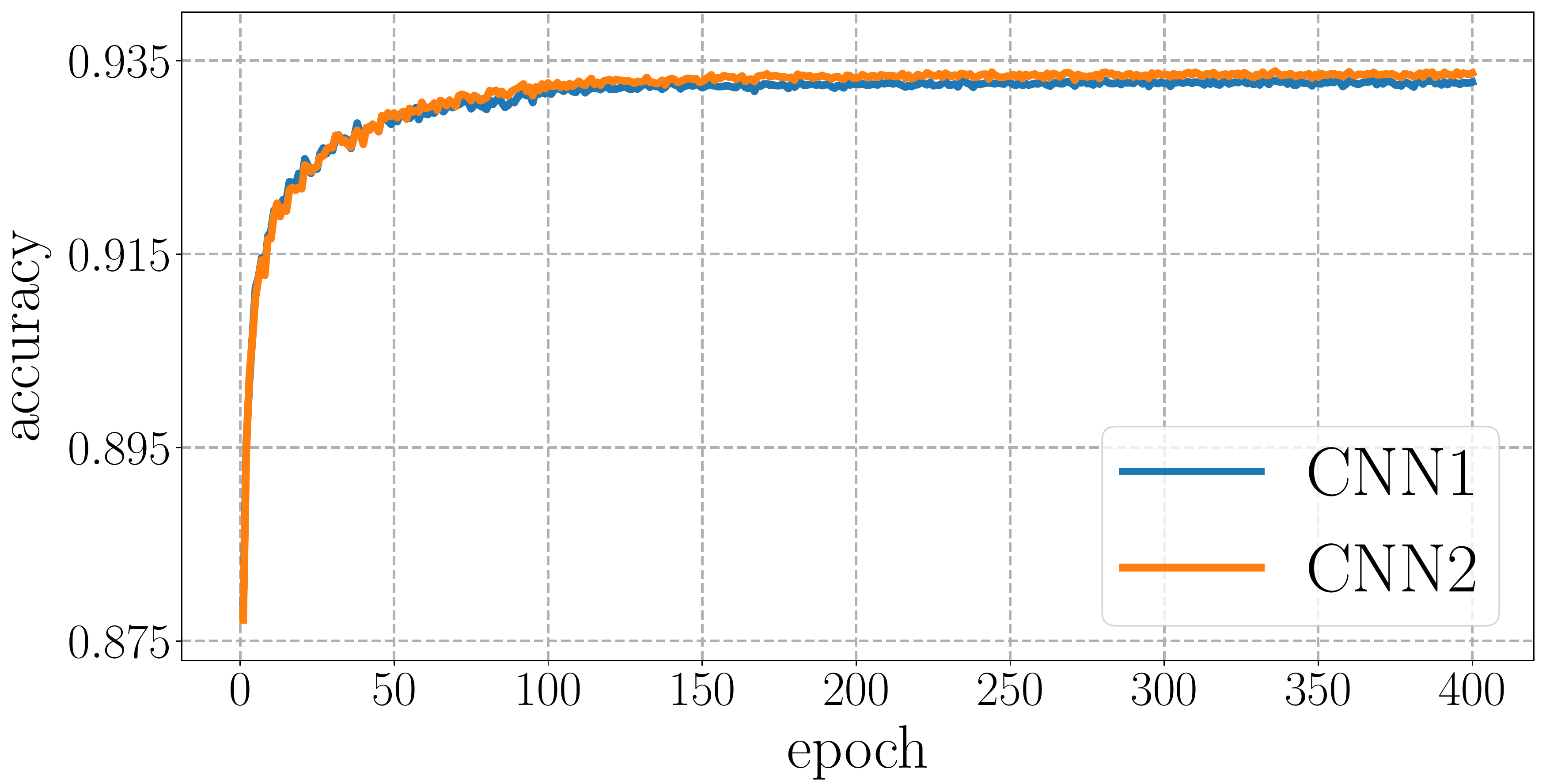}
		\subcaption{Epochs 1-400.}
	\end{subfigure}
		\hspace{3pt}
	\begin{subfigure}[b]{0.48\textwidth}
		\centering           \includegraphics[width=0.9768\textwidth]{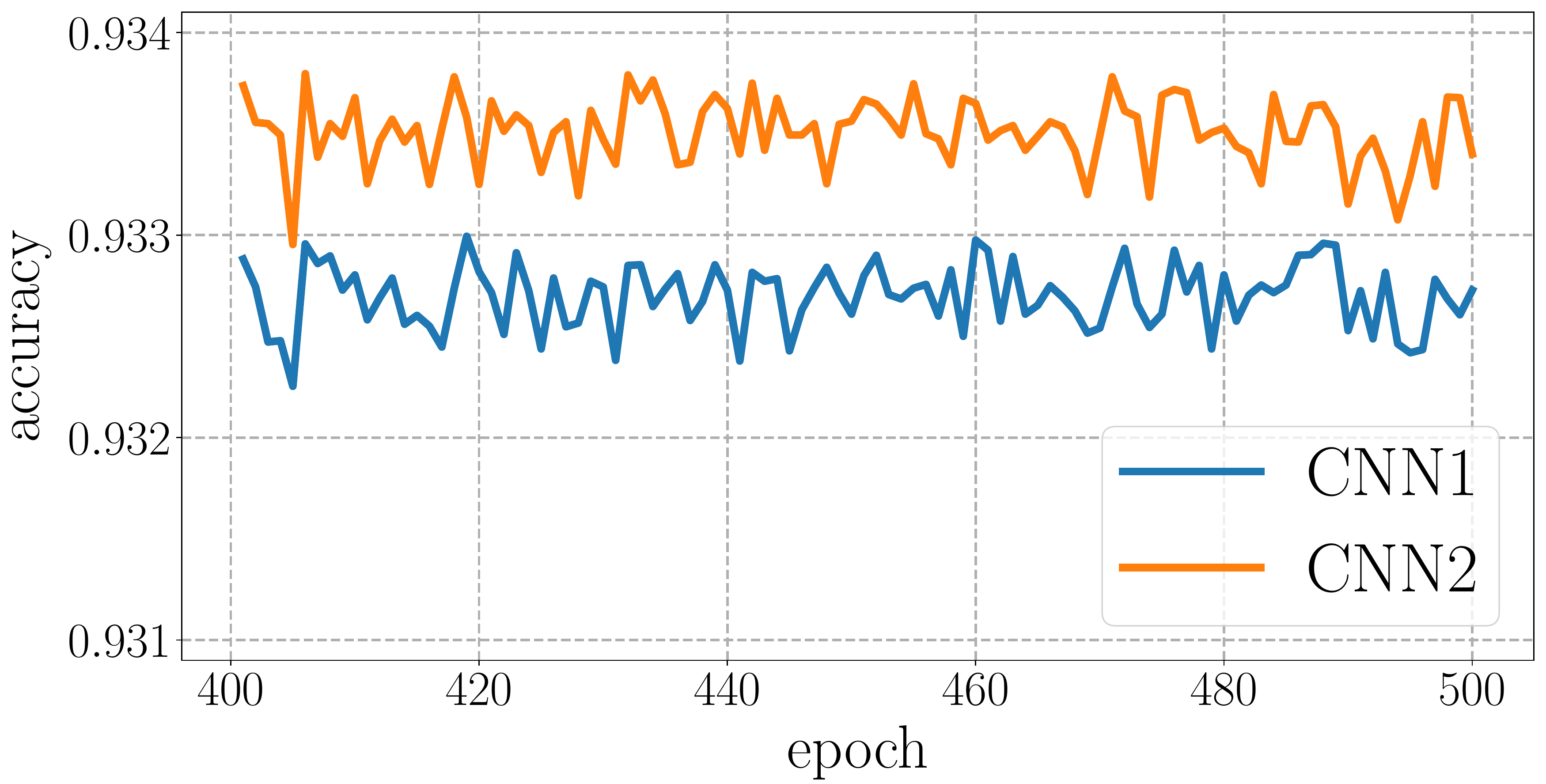}
		\subcaption{Epochs 401-500.}
	\end{subfigure}
	\caption{Test accuracy over epochs.}
	\label{fig:accuracy:comparison}
\end{figure}


\section{Other Examples of UAFs}
\label{sec:UAFs}

This section aims at designing new UAFs with additional properties such as smooth or sigmoidal functions. 
As discussed in the introduction and shown in the proof of our main theorem,  the construction of UAFs  mainly relies on three properties: high nonlinearity,  periodicity, and the capacity to reproduce step functions.  The EUAF $\sigma$ defined in Equation~\eqref{eq:def:sigma}  is a simple and typical example of UAFs satisfying these three properties. 
Indeed, having these properties plays an important role in our proof and is a necessary but not sufficient condition for designing a UAF. In other words, these properties are important, but cannot guarantee the successful construction of UAFs.

Here, we present another idea to design new UAFs, which mainly relies on the following observation: If a UAF $\varrho$ can be approximated by a fixed-size network activated by a new activation function $\tildevarrho$ within an arbitrary error on any bounded interval, then $\tildevarrho$ is also a UAF. Such an observation is a direct result of the lemma below.

\begin{lemma}\label{lem:net:varrho:to:tildevarrho}
    Let $\varrho,\tildevarrho:\R\to\R$ be two functions with $\varrho\in C(\R)$. For an arbitrary given  function  $f\in [a,b]^d\to \R$ and  any $\varepsilon>0$, suppose that the following two conditions hold:
    \begin{itemize}
        \item There exists a function $\phi_\varrho$ realized by a $\varrho$-activated network with width $N$ and depth $L$ such that 
        \begin{equation*}
             |\phi_\varrho(\bmx)-f(\bmx)|<\varepsilon/2 \quad \tn{for any $\bmx\in [a,b]^d$.}
        \end{equation*}
    
        \item For any $M>0$ and each $\delta\in (0,1)$, there exists a function $\varrho_\delta$ realized by a $\tildevarrho$-activated network with width $\tildeN$ and depth $\tildeL$ such that 
        \begin{equation*}
            \varrho_\delta(t)\rightrightarrows \varrho(t)\quad \tn{as}\quad \delta\to 0^+\quad \tn{for any $t\in [-M,M]$,}
        \end{equation*}
        where $\rightrightarrows$ denotes the uniform convergence.
    \end{itemize}
    Then, there exists a function $\phi=\phi_{\tildevarrho}$ generated by a $\tildevarrho$-activated network with width $N\cdot\tildeN$ and depth $L\cdot\tildeL$ such that
    \begin{equation*}
        |\phi(\bmx)-f(\bmx)|<\varepsilon\quad \tn{for any $\bmx\in [a,b]^d$.}
    \end{equation*}
\end{lemma}

The proof of Lemma~\ref{lem:net:varrho:to:tildevarrho} is placed in Section~\ref{sec:proof:lem:net:varrho:to:tildevarrho}.
Based on Lemma~\ref{lem:net:varrho:to:tildevarrho},
we will propose two UAFs with better mathematical properties. That is, the idea of designing a $C^s$ UAF is given in Section~\ref{sec:UAF:smooth}
and a sigmoidal UAF is constructed in detail in Section~\ref{sec:UAF:sig}.

\subsection{Smooth UAF}\label{sec:UAF:smooth}

The smoothness of a function is one of the most desired properties in mathematical modeling and computation. The EUAF $\sigma$ is continuous but not smooth. 
So we will show how to construct a $C^s$ UAF based on an existing one. The key point is the fact that the indefinite integral of a continuous function is continuously differentiable.

Suppose $\varrho$ is a continuous UAF. Define 
\begin{equation*}
    \tildevarrho(x)\coloneqq\int_{0}^{x} \varrho(t)\tn{d}t\quad \tn{ for any $x\in\R$.} 
\end{equation*}
For any $M>0$,
it holds that
\begin{equation*}
    \frac{\tildevarrho(x+\delta)-\tildevarrho(x)}{\delta}
    =\frac{1}{\delta}\int_{x}^{x+\delta} \varrho(t)\tn{d}t
    \rightrightarrows
    \varrho(x) \quad \tn{as} \quad \delta\to 0^+\quad \tn{for any $x\in[-M,M]$}.
\end{equation*}
This means $\varrho$ can be approximated by a one-hidden-layer $\tildevarrho$-activated network with width $2$ arbitrarily well on any bounded interval.
It follows that
$\tildevarrho$ is also a UAF. By repeated applications of the above idea, one could easily construct a $C^s$ UAF.

In particular, set $\varrho_0=\sigma$  and define $\varrho_1,\varrho_2,\cdots,\varrho_s$  by induction as follows.
\begin{equation}\label{eq:varrho:s}
    \varrho_{i+1}(x)\coloneqq \int_{0}^{x} \varrho_i(t)\tn{d} t\quad \tn{for any $x\in\R$ and $i\in \{0,1,\cdots,s-1\}$.}
\end{equation}
Then $\varrho_s$ is a $C^s$ UAF as shown in the following theorem.
    
\begin{theorem}\label{thm:main:smooth}
    Let $\varrho_s\in C^s(\R)$ be the function defined in Equation~\eqref{eq:varrho:s} for any $s\in\N^+$. Then,
    for any $f\in C([a,b]^d)$ and any $\varepsilon>0$, there exists a function $\phi$ generated by a $\varrho_s$-activated network with width $72sd(2d+1)$ and depth $11$ such that
    \begin{equation*}
        |\phi(\bmx)-f(\bmx)|<\varepsilon\quad \tn{for any $\bmx\in [a,b]^d$.}
    \end{equation*}
\end{theorem}
\begin{proof}
    For any $i\in \{0,1,\cdots,s-1\}$ and any $M>0$,
    it is easy to verify that
    \begin{equation*}
        \frac{\varrho_{i+1}(x+\delta)-\varrho_{i+1}(x)}{\delta}
        =\frac{1}{\delta}\int_{x}^{x+\delta} \varrho_i(t)\tn{d}t
        \rightrightarrows
        \varrho_i(x) \quad \tn{as} \quad \delta\to 0^+\quad \tn{for any $x\in [-M,M]$}.
    \end{equation*}
    This means $\varrho_{i}$ can be approximated by a one-hidden-layer $\varrho_{i+1}$-activated network with width $2$ arbitrarily well on any bounded interval. By induction, one could easily prove that
    $\varrho_{0}=\sigma$ can be approximated by a one-hidden-layer $\varrho_s$-activated network with width $2s$ arbitrarily well on any bounded interval. That is, for each $\delta\in (0,1)$, there exists a function $\sigma_{s,\delta}$ realized by a $\varrho_s$-activated network with width $2s$ and depth $1$ such that 
    \begin{equation*}
        \sigma_{s,\delta}(t)\rightrightarrows \sigma(t)\quad \tn{as}\quad \delta\to 0^+\quad \tn{for any $t\in [-M,M]$.}
    \end{equation*}
    By Theorem~\ref{thm:main}, there exists a function $\phi_\sigma$ generated by a $\sigma$-activated network with width $36d(2d+1)$ and depth $11$ such that 
    \begin{equation*}
        |\phi_\sigma(\bmx)-f(\bmx)|<\varepsilon/2 \quad \tn{for any $\bmx\in [a,b]^d$.}
    \end{equation*}
        Then, by Lemma~\ref{lem:net:varrho:to:tildevarrho}, there exists another function $\phi=\phi_{\varrho_s}$ realized by a $\varrho_s$-activated network with width $2s\times 36d(2d+1)=72sd(2d+1)$ and depth $1\times 11=11$ such that 
    \begin{equation*}
        |\phi(\bmx)-f(\bmx)|<\varepsilon \quad \tn{for any $\bmx\in [a,b]^d$.}
    \end{equation*}
    So we finish the proof.
\end{proof}

\subsection{Sigmoidal UAF}\label{sec:UAF:sig}

Many activation functions used in real-world applications are sigmoidal functions. Generally,  we say a function $g:\R\to\R$ is sigmoidal (or sigmoid, e.g., see \cite{10.1007/3-540-59497-3_175}) if it satisfies the following conditions.
\begin{itemize}
    \item Bounded: $\lim_{x\to \infty} g(x)=1$ and $\lim_{x\to -\infty} g(x)=-1$ (or $0$). 
    \item Differentiable: $g'(x)$ exists and continuous for all $x\in\R$.
    \item Increasing: $g'(x)$ is non-negative for all $x\in\R$.
\end{itemize}

Our goal is to construct a sigmoidal UAF.  To this end,
we need to design a new function $\tildesigma$ based on $\sigma$ such that $\sigma$ can be reproduced/approximated by a $\tildesigma$-activated network with a fixed size. Making $\tildesigma$ bounded and increasing is not difficult. 
The key is to make $\tildesigma$ continuously differentiable, which can be implemented by the fact that the indefinite integral of a continuous function is continuously differentiable. 
To be exact,
we can define $\tildesigma$ as follows. 
\begin{itemize}
    \item For $x\in (-\infty,0]$, define $\tildesigma(x)\coloneqq\sigma(x)=\tfrac{x}{-x+1}$.
    \item For $x\in (0,\infty)$, define
    \begin{equation*}
        \tildesigma(x)\coloneqq\int_{0}^{x} \frac{c\sigma(t)+1}{(2t+1)^2}\tn{d}t,\quad \tn{where}\quad c=\frac{1}{2\int_{0}^{\infty} \frac{\sigma(t)}{(2t+1)^2}\tn{d}t}\approx 2.554.
    \end{equation*}
\end{itemize}
We remark that there are many possible choices for the integrand in the above definition of $\tildesigma(x)$ for $x\in (0,\infty)$. Here, we just give a simple example.
See an illustration of $\tildesigma$ in Figure~\ref{fig:tildesigma}.

\begin{figure}[htbp!]        
	\centering          
	\includegraphics[width=0.6418\textwidth]{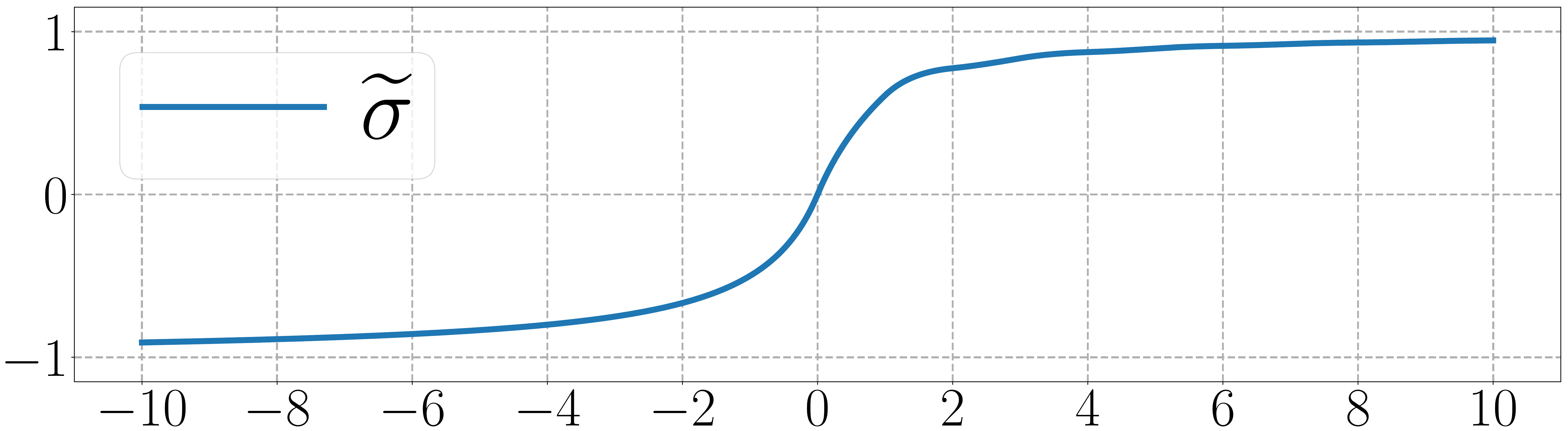}
	\caption{An illustration of $\tildesigma$ on $[-10,10]$.}
	\label{fig:tildesigma}
\end{figure}

Then $\tildesigma$ is a sigmoidal function as verified below.
\begin{itemize}
    \item Clearly, $\lim_{x\to-\infty}\tildesigma(x)=\lim_{x\to-\infty}\tfrac{x}{-x+1}=-1$. Moreover, 
    \begin{equation*}
        \lim_{x\to\infty} \tildesigma(x)
        = \int_{0}^{\infty} \frac{c\sigma(t)+1}{(2t+1)^2}\tn{d}t
        =\frac{1}{2}+\int_{0}^{\infty} \frac{1}{(2t+1)^2}\tn{d}t=1.
    \end{equation*}
    
    \item Obviously, $\tildesigma$ is continuously differentiable on $(-\infty,0)$ and $(0,\infty)$. 
    Meanwhile, we have $\tildesigma'(0)=1$ and 
        $\lim_{x\to 0}\tildesigma'(x)=1$.
    Therefore, we have $\tildesigma\in C^1(\R)$ as desired.
    
    \item For $x\in (-\infty,0)$, $\tildesigma'(x)=\tfrac{1}{(-x+1)^2}> 0$.  For $x=0$, $\tildesigma'(x)=1>0$. For $x\in (0,\infty)$, $\tildesigma'(x)=\tfrac{c\sigma(x)+1}{(2x+1)^2}>0$. Therefore, $\tildesigma'(x)>0$ for all $x\in \R$.
\end{itemize}
 
Based on Theorem~\ref{thm:main} corresponding to $\sigma$,
we establish a similar theorem for $\tildesigma$, Theorem~\ref{thm:main:tildesigma} below, showing that fixed-size  $\tildesigma$-activated networks can also approximate continuous functions within an arbitrary error on a hypercube. 

\begin{theorem}\label{thm:main:tildesigma}
    For any $f\in C([a,b]^d)$ and any $\varepsilon>0$, there exists a function $\phi$ generated by a $\tildesigma$-activated network with width $1800d(2d+1)$ and depth $66$ such that
    \begin{equation*}
        |\phi(\bmx)-f(\bmx)|<\varepsilon\quad \tn{for any $\bmx\in [a,b]^d$.}
    \end{equation*}
\end{theorem}

To prove this theorem based on Theorem~\ref{thm:main}, we only need to show
$\sigma$ can be approximated by  a fixed-size
$\tildesigma$-activated network within an arbitrary error on any pre-specified interval as presented in the following lemma.

\begin{lemma}\label{lem:sigma:by:varrho:net}
    For any $\varepsilon>0$ and any $M>0$, there exists a function $\phi$ realized by a $\tildesigma$-activated network with width $50$ and depth $6$ such that \begin{equation*}
        |\phi(x)-\sigma(x)|<\varepsilon\quad 
        \tn{for any $x\in[-M,M]$.} 
    \end{equation*}
\end{lemma}

The proof of Lemma~\ref{lem:sigma:by:varrho:net} can be found later.
By assuming Lemma~\ref{lem:sigma:by:varrho:net} is true, we can give the proof of Theorem~\ref{thm:main:tildesigma}.
\begin{proof}[Proof of Theorem~\ref{thm:main:tildesigma}]
    By Theorem~\ref{thm:main}, there exists a function $\phi_\sigma$ generated by a $\sigma$-activated network with width $36d(2d+1)$ and depth $11$ such that 
    \begin{equation*}
        |\phi_\sigma(\bmx)-f(\bmx)|<\varepsilon/2 \quad \tn{for any $\bmx\in [a,b]^d$.}
    \end{equation*} 
    By Lemma~\ref{lem:sigma:by:varrho:net}, for any $M>0$ and each $\delta\in (0,1)$, there exists a function $\sigma_{\delta}$ realized by a $\tildesigma$-activated network with width $50$ and depth $6$ such that 
    \begin{equation*}
        \sigma_{\delta}(t)\rightrightarrows \sigma(t)\quad \tn{as}\quad \delta\to 0^+\quad \tn{for any $t\in [-M,M]$.}
    \end{equation*}
    Then, by Lemma~\ref{lem:net:varrho:to:tildevarrho}, there exists another function $\phi=\phi_{\tildesigma}$ realized by a $\tildesigma$-activated network with width $50\times 36d(2d+1)=1800d(2d+1)$ and depth $6\times 11=66$ such that
    \begin{equation*}
        |\phi(\bmx)-f(\bmx)|<\varepsilon \quad \tn{for any $\bmx\in [a,b]^d$.}
    \end{equation*}
    So we finish the proof.
\end{proof}

Finally, let us present the detailed proof of Lemma~\ref{lem:sigma:by:varrho:net}.
\begin{proof}[Proof of Lemma~\ref{lem:sigma:by:varrho:net}]
    Since $1=\tildesigma'(0)=\lim_{x\to 0} \tfrac{\tildesigma(x)}{x}$, it is easy to show:
    For any $\scrE>0$ and any $R>0$, there exists a sufficiently small $w>0$ such that 
    \begin{equation*}
        \big|\tildesigma(wx)/w -x\big|<\scrE\quad \tn{for any $x\in [-R,R]$.}
    \end{equation*}
    Thus, we may assume the identity map is allowed to be the activation function in $\tildesigma$-activated networks. 
    Without loss of generality, we may assume $M\ge 2$ because $\widehat{M}=\max\{2,M\}$ implies $\widehat{M}\ge 2$ and $[-M,M]\subseteq [-\widehat{M},\widehat{M}]$.

    For simplicity, we denote
    $\tildescrH(N,L)$ as the (hypothesis) space of  functions generated by $\tildesigma$-activated networks with width $N$ and depth $L$. 
    Then the proof can be roughly divided into three steps as follows.
    \begin{enumerate}[(1)]
        \item Design $\Gamma\in \tildescrH(9,2)$ to reproduce $xy$ on $[-4\tildeM,4\tildeM]^2$, where $\tildeM=(M+1)^2$.
        
        \item Design $\psi_\delta\in \tildescrH(9,4)$ based on the first step to approximate $\sigma$ well on $[0,M]$.
        
        \item Design $\phi\in \tildescrH(50,6)$ based on the previous two steps to approximate $\sigma$ well on $[-M,M]$.
    \end{enumerate}
    
    The details of these three steps can be found below.
    \mystep{1}{Design $\Gamma\in \tildescrH(9,2)$ to reproduce $xy$ on $[-4\tildeM,4\tildeM]^2$.}
    
    	Observe that
	\begin{equation*}
		\tildesigma(y)+1=\frac{y}{|y|+1}+1=\frac{y}{-y+1}+1=\frac{1}{-y+1}\quad 
	\tn{for any $y\le 0$.}
	\end{equation*}
	For any $x\in[-4,4]$, we have  $-x-4\le 0$ and $-x-5\le 0$, implying
	\begin{equation*}
		\begin{split}
			\tildesigma(-x-4)-\tildesigma(-x-5)&=\Big(\tildesigma(-x-4)+1\Big)-\Big(\tildesigma(-x-5)+1\Big)\\
			&=\frac{1}{-(-x-4)+1}-\frac{1}{-(-x-5)+1}\\
			&=\frac{1}{x+5}-\frac{1}{x+6}=\frac{1}{(x+5)(x+6)}.
		\end{split}
	\end{equation*}
	It follows from  $1-\tfrac{90}{(x+5)(x+6)}\le 0$ for any $x\in[-4,4]$ that 
	\begin{equation*}
		\begin{split}
			\tildesigma\Big(1-\frac{90}{(x+5)(x+6)}\Big)+1
			=\frac{1}{-\big(1-\tfrac{90}{(x+5)(x+6)}\big)+1}
			=\frac{x^2+11x+30}{90},
		\end{split}
	\end{equation*}
	implying
	\begin{equation*}
		\begin{split}
			x^2&=90 \tildesigma\Big(1-\frac{90}{(x+5)(x+6)}\Big)+90-(11x+30)\\
			&=90 \tildesigma\Big(1- 90 \big(\tildesigma(-x-4)- \tildesigma(-x-5)\big)\Big)-11x+60\\
			& =90 \tildesigma\Big(1-90\tildesigma(-x-4)+90\tildesigma(-x-5)\Big)-11x+60.
		\end{split}
	\end{equation*}
    Thus, $x^2$ can be realized by a $\tildesigma$-activated network with width $3$ and depth $2$ on $[-4,4]$.
	Set $\tildeM=(M+1)^2$. Then, for any $x,y\in[-4\tildeM,4\tildeM]$, we have $\tfrac{x}{2\tildeM},\tfrac{y}{2\tildeM},\tfrac{x+y}{2\tildeM}\in [-4,4]$. Recall the fact
	\begin{equation*}
	    xy=2\tildeM^2\Big((\tfrac{x+y}{2\tildeM})^2
	    -(\tfrac{x}{2\tildeM})^2-(\tfrac{y}{2\tildeM})^2\Big).
	\end{equation*}
	Therefore, $xy$ can be realized by a $\tildesigma$-activated network with width $9$ and depth $2$  for any  $x,y\in [-4\tildeM,4\tildeM]$. That is, there exists $\Gamma\in \tildescrH(9,2)$ such that $\Gamma(x,y)=xy$ on $[-4\tildeM,4\tildeM]^2$.

    \mystep{2}{Design $\psi_\delta\in \tildescrH(9,4)$ to approximate $\sigma$ well on $[0,M]$.}
    
    Recall that $x^2$ can be realized by a $\tildesigma$-activated network with width $3$ and depth $2$ on $[-4,4]$. There exists $\psi_1\in \tildescrH(3,2)$ such that
    \begin{equation*}
        \psi_1(x)=\frac{(2x+1)^2}{(2M+1)^2}\quad \tn{for any $x\in [-M,M]$.}
    \end{equation*}

    For any small $\delta>0$, we define 
    \begin{equation*}
        \psi_{2,\delta}(x)\coloneqq\frac{\tildesigma(x+\delta)-\tildesigma(x)}{\delta} \quad \tn{for any $x\in\R$.}
    \end{equation*} 
    Then, we have $\psi_{2,\delta}\in \tildescrH(2,1)$ and   
    \begin{equation*}
        \psi_{2,\delta}(x)\coloneqq\frac{\tildesigma(x+\delta)-\tildesigma(x)}{\delta} \rightrightarrows \frac{\tn{d} }{\tn{d}x}\tildesigma(x)=\frac{c\sigma(x)+1}{(2x+1)^2}\quad \tn{as}\quad \delta\to 0^+
    \end{equation*}
    for any $x\in [0,M]$, where $c$ is a constant given by
    \begin{equation*}
        c=\frac{1}{2\int_{0}^{\infty} \frac{\sigma(t)}{(2t+1)^2}\tn{d}t}\approx 2.554.
    \end{equation*}
    
    For any small $\delta>0$, we define 
    \begin{equation*}
        \psi_\delta(x)\coloneqq \tfrac{(2M+1)^2}{c}\Gamma\Big(\psi_{1}(x), \psi_{2,\delta}(x)\Big)-\tfrac{1}{c}\quad \tn{for any $x\in\R.$}
    \end{equation*}
    Since $\Gamma\in \tildescrH(9,2)$, $\psi_1\in \tildescrH(3,2)$, and $\psi_{2,\delta}\in \tildescrH(2,1)$, we have $\psi_\delta\in \tildescrH(9,4)$.

    Clearly, for any $x\in [0,M]$, we have $\psi_1(x)=\tfrac{(2x+1)^2}{(2M+1)^2}\in [0,1]$ and $\psi_{2,\delta}(x)\rightrightarrows \tfrac{c\sigma(x)+1}{(2x+1)^2}\in [0,c+1]\subseteq [0,3.6]$, implying
    $\psi_{1}(x),\psi_{2,\delta}(x)\in [-4,4]\subseteq [-4\tildeM,4\tildeM]$ for any small $\delta>0$.
    Thus, for any $x\in [0,M]$, as $\delta$ goes to $0^+$, we have
    \begin{equation*}
        \begin{split}
            \psi_\delta(x)
            &=\tfrac{(2M+1)^2}{c}\Gamma\Big(\psi_{1}(x), \psi_{2,\delta}(x)\Big)-\tfrac{1}{c}
            =\tfrac{(2M+1)^2}{c}\cdot\psi_{1}(x)\cdot \psi_{2,\delta}(x)-\tfrac{1}{c}\\
            &\rightrightarrows \tfrac{(2M+1)^2}{c}\cdot \tfrac{(2x+1)^2}{(2M+1)^2}\cdot \tfrac{c\sigma(x)+1}{(2x+1)^2}-\tfrac{1}{c}=\sigma(x).
        \end{split}
    \end{equation*}
    That is, for any $x\in [0,M]$, 
    \begin{equation*}
        \psi_\delta(x)\rightrightarrows \sigma(x)\quad \tn{as} \quad \delta\to 0^+.
    \end{equation*}
    
    \mystep{3}{Design $\phi\in \tildescrH(50,6)$ to approximate $\sigma$ well on $[-M,M]$.}
    
    Note that $\tildesigma(x)=\sigma(x)$ for all $x\in [-M,0)$ and $\psi_\delta(x)$ approximates $\sigma(x)$ well for all $x\in [0,M]$. Then, we have
    \begin{equation*}
        \psi_\delta(x)\cdot \one_{\{x\in [0,M]\}}+\tildesigma(x)\cdot \one_{\{x\in [-M,0)\}}
    \end{equation*}
    approximates $\sigma(x)$ well for all $x\in [-M,M]$.
However, it is impossible to approximate $\one_{\{x\in [0,M]\}}$ well by a $\tildesigma$-activated network due to the continuity of $\tildesigma$. To address this gap, we will construct a continuous function $g$  to replace $\one_{\{x\in [0,M]\}}$ such that
\begin{equation}\label{eq:phi:delta+tsigma}
    \psi_\delta(x)\cdot g(x)+\tildesigma(x)\cdot \big(1-g(x)\big)
\end{equation}
can also approximate $\sigma(x)$ well for all $x\in [-M,M]$.

        By the continuity of $\tildesigma$ and $\sigma$, there exists a small $\eta_0\in (0,1)$ such that 
    \begin{equation}\label{eq:eta0:ineq}
        |\tildesigma(x)|<\varepsilon/6
        \quad \tn{and }\quad 
        |\sigma(x)|<\varepsilon/6
        \quad \tn{for any $x\in [0,\eta_0]$}.
    \end{equation}
    Then we define
    \begin{equation*}
        g(x)\coloneqq \frac{\tn{ReLU}(x)-\tn{ReLU}(x-\eta_0)}{\eta_0},\quad
        \tn{where} \  \tn{ReLU}(x)=\max\{0,x\}\quad  \tn{for any $x\in\R$}.
    \end{equation*}
    See Figure~\ref{fig:g} for an illustration of $g$.
    
    \begin{figure}[htbp!]        
    	\centering          
    	\includegraphics[width=0.6418\textwidth]{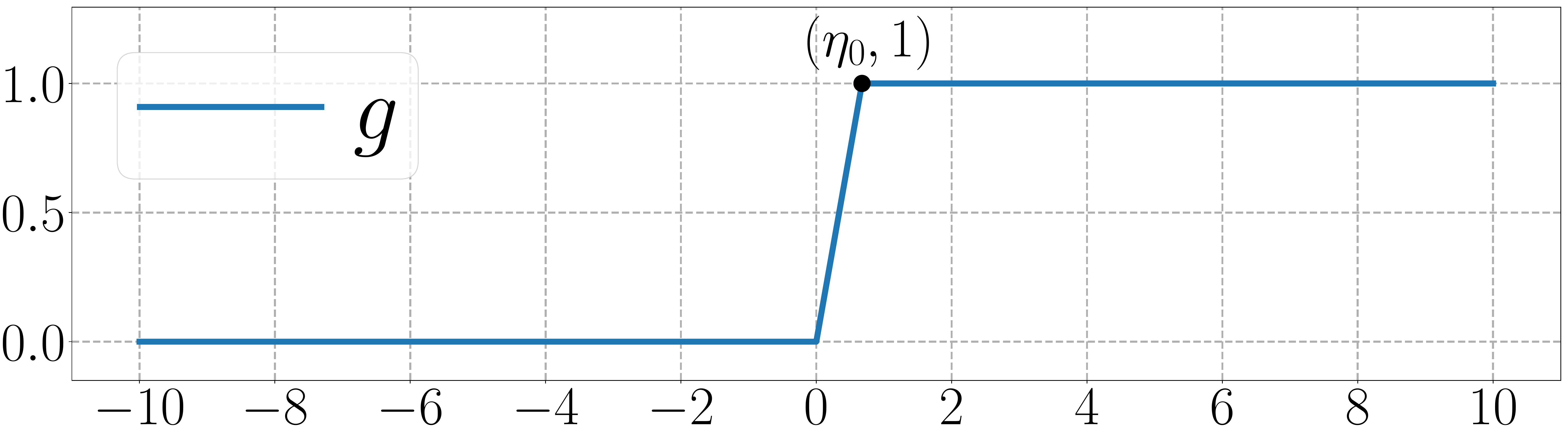}
    	\caption{An illustration of $g$ on $[-10,10]$.}
    	\label{fig:g}
    \end{figure}
    
    We will construct a $\tildesigma$-activated network to approximate $g$ well. To this end, we first design a $\tildesigma$-activated network to approximate the ReLU function well.
    For any $x\in [-M-1,M+1]$, we have $\tfrac{x}{M+1}+1\in [0,2]\subseteq [0,M]$, implying
    \begin{equation*}
       1-\psi_\delta(\tfrac{x}{M+1}+1) \rightrightarrows 1-\sigma(\tfrac{x}{M+1}+1)=|\tfrac{x}{M+1}|\quad \tn{as} \quad \delta\to 0^+,
    \end{equation*}
    where the last equality comes from $1-\sigma(y)=|y-1|$ for any $y\in [0,2]$.
    Recall that 
    \begin{equation*}
        \tn{ReLU}(x)=\tfrac{x}{2}+\tfrac{|x|}{2}=\tfrac{x}{2}+\tfrac{M+1}{2}\cdot|\tfrac{x}{M+1}|
    \end{equation*}
    for any $x\in [-M-1,M+1]$.
    For any small $\delta>0$, we define 
    \begin{equation*}
        \tildeg_\delta(x)\coloneqq \tfrac{x}{2}+\tfrac{M+1}{2}\Big(
        1-\psi_\delta(\tfrac{x}{M+1}+1) \Big)\quad \tn{for any $x\in \R$}.
    \end{equation*}
    Then, $\psi_\delta\in \tildescrH(9,4)$ implies
    $\tildeg_\delta\in \tildescrH(10,4)$.
    Moreover, for any $x\in [-M-1,M+1]$,
    \begin{equation*}
       \tildeg_\delta(x) \rightrightarrows \tfrac{x}{2}+\tfrac{M+1}{2}\cdot|\tfrac{x}{M+1}|
       =\tn{ReLU}(x) \quad \tn{as} \quad \delta\to 0^+.
    \end{equation*}
        Define 
    \begin{equation*}
        g_\delta(x)\coloneqq \frac{\tildeg_\delta(x)-\tildeg_\delta(x-\eta_0)}{\eta_0}\quad  \tn{for any $x\in\R$}.
    \end{equation*}
    Clearly, $\tildeg_\delta\in \tildescrH(10,4)$ implies $g_\delta\in \tildescrH(20,4)$. For any $x\in [-M,M]$,  we have $x,x-\eta_0\in [-M-1,M+1]$, implying
    \begin{equation*}
       g_\delta(x) =
       \frac{\tildeg_\delta(x)-\tildeg_\delta(x-\eta_0)}{\eta_0}
     \rightrightarrows
       \frac{\tn{ReLU}(x)-\tn{ReLU}(x-\eta_0)}{\eta_0}
       =g(x) \quad \tn{as}
        \quad \delta\to 0^+.
    \end{equation*}
    
    Next, motivated by Equation~\eqref{eq:phi:delta+tsigma}, we can define $\phi_\delta$ to approximate $\sigma$ well on $[-M,M]$. The definition of $\phi_\delta$ is given by
    \begin{equation*}
        \phi_\delta(x)\coloneqq \Gamma\Big(\psi_\delta(x),g_\delta(x)\Big)
        +\Gamma\Big(\tildesigma(x),1-g_\delta(x)\Big)\quad  \tn{for any $x\in \R$.}
    \end{equation*}
    Since $\Gamma\in \tildescrH(9,2)$, $\psi_\delta\in \tildescrH(9,4)$, and   $g_\delta,1-g_\delta\in \tildescrH(20,4)$, we have
    \begin{equation*}
        \phi_\delta\in          \tildescrH(9+20+1+20,4+2)     =\tildescrH(50,6).
    \end{equation*}

    Clearly, $\tildesigma(x)$, $g_\delta(x)$, and $1-g_\delta(x)$ are all in $[-4\tildeM,4\tildeM]$ for any small $\delta>0$ and all $x\in [-M,M]$.
    We will show $\psi_\delta(x)\in [-4\tildeM,4\tildeM]$ for any small $\delta>0$ and all $x\in [-M,M]$ via two cases as follows.
    \begin{itemize}
        \item For any $x\in [0,M]$, $\psi_\delta(x) \rightrightarrows \sigma(x)$ implies $\psi_\delta(x)\in [-4\tildeM,4\tildeM]$ for any small $\delta>0$.
        
        \item  For any $x\in [-M,0)$, we have $\psi_1(x)=\tfrac{(2x+1)^2}{(2M+1)^2}\in [0,1]$ and
        \begin{equation*}
        \psi_{2,\delta}(x)=\tfrac{\tildesigma(x+\delta)-\tildesigma(x)}{\delta} \rightrightarrows \tfrac{\tn{d} }{\tn{d}x}\tildesigma(x)=\tfrac{1}{(-x+1)^2}\quad \tn{as}\quad \delta\to 0^+.
    \end{equation*}
        Thus, for any $x\in [-M,0)$, as $\delta$ goes to $0^+$, we get
    \begin{equation*}
        \begin{split}
            \psi_\delta(x)
            &=\tfrac{(2M+1)^2}{c}\Gamma\Big(\psi_{1}(x), \psi_{2,\delta}(x)\Big)-\tfrac{1}{c}
            =\tfrac{(2M+1)^2}{c}\cdot\psi_{1}(x)\cdot \psi_{2,\delta}(x)-\tfrac{1}{c}\\
            &\rightrightarrows \tfrac{(2M+1)^2}{c}\cdot \tfrac{(2x+1)^2}{(2M+1)^2}\cdot \tfrac{1}{(-x+1)^2}-\tfrac{1}{c}
            =\tfrac{(2x+1)^2-1}{c(-x+1)^2}.
        \end{split}
    \end{equation*}
    For all $x\in [-M,0)$, we have $c(-x+1)^2\ge 1$, implying $\tfrac{(2x+1)^2-1}{c(-x+1)^2}\ge \tfrac{-1}{c(-x+1)^2}\ge -1$ and 
    \begin{equation*}
        \begin{split}
        \tfrac{(2x+1)^2-1}{c(-x+1)^2}
        \le 
        \tfrac{(2|x|+1)^2-1}{c(-x+1)^2}
        \le (2|x|+1)^2-1
        &= 4(|x|+1/2)^2-1\\
        &
        \le 4(M+1)^2-1=4\tildeM-1.
        \end{split}
    \end{equation*}
    That is, $\tfrac{(2x+1)^2-1}{c(-x+1)^2}\in  [-1,4\tildeM-1]$ for all $x\in [-M,0)$,  implying
    $\psi_{\delta}(x)\in [-4\tildeM,4\tildeM]$ for any small $\delta>0$.
    \end{itemize}
    Hence, for any $x\in [\eta_0,M]$,  we have  $1-g(x)=0$, implying
    \begin{equation*}
        \phi_\delta(x)=\psi_\delta(x)\cdot g_\delta(x)+\tildesigma(x)\cdot \big(1-g_\delta(x)\big)
        \rightrightarrows \sigma(x)\cdot g(x)+0=\sigma(x) \quad \tn{as}
        \quad \delta\to 0^+.
    \end{equation*}
    Similarly, for any $x\in [-M,0]$,  we have  $g(x)=0$, implying
    \begin{equation*}
        \phi_\delta(x)=\psi_\delta(x)\cdot g_\delta(x)+\tildesigma(x)\cdot \big(1-g_\delta(x)\big)
        \rightrightarrows 0 + \tildesigma(x)\cdot \big(1-g(x)\big) =\sigma(x) \quad \tn{as}
        \quad \delta\to 0^+.
    \end{equation*}
    Therefore, there exists a small $\delta_0>0$ such that  
    \begin{equation*}
        |\phi_{\delta_0}(x)-\sigma(x)|<\varepsilon \quad \tn{for any $x\in [-M,0]\bigcup [\eta_0,M]$},
    \end{equation*}
    $\|g_{\delta_0}\|_{L^\infty([0,\eta_0])}\le 2$,\quad 
    $\|1-g_{\delta_0}\|_{L^\infty([0,\eta_0])}\le  2$,\quad
    and
    \[\|\psi_{\delta_0}\|_{L^\infty([0,\eta_0])}\le 
    \|\sigma\|_{L^\infty([0,\eta_0])} +\varepsilon/12, \]
    where the above inequality comes from the fact $\psi_\delta(x)$ uniformly converges to $\sigma(x)$ for any $x\in [0,\eta_0]\subseteq [0,M]$.

    Clearly, 
    for any $x\in [0,\eta_0]$, by Equation~\eqref{eq:eta0:ineq}, we have 
    \begin{equation*}
        \begin{split}
            |\phi_{\delta_0}(x)-\sigma(x)|
            \le 
            |\phi_{\delta_0}(x)|+|\sigma(x)|
            &<\Big|\psi_{\delta_0}(x)\cdot g_{\delta_0}(x)+\tildesigma(x)\cdot \big(1-g_{\delta_0}(x)\big)\Big| +\varepsilon/6\\
            &\le \big|\psi_{\delta_0}(x)\big|\cdot\big| g_{\delta_0}(x)\big|+\big|\tildesigma(x)\big|\cdot \big|1-g_{\delta_0}(x)\big| +\varepsilon/6\\
            &\le  \big(\|\sigma\|_{L^\infty([0,\eta_0])} +\frac{\varepsilon}{12}\big) \cdot 2 + \frac{\varepsilon}{6}\cdot 2 +\frac{\varepsilon}{6}\\
            &\le \big( \frac{\varepsilon}{6}+\frac{\varepsilon}{12}\big) \cdot 2 + \frac{\varepsilon}{6}\cdot 2 +\frac{\varepsilon}{6}=\varepsilon.
        \end{split}
    \end{equation*}
    By setting $\phi=\phi_{\delta_0}$, we have 
    $\phi=\phi_{\delta_0}\in \tildescrH(50,6)$ and 
        \begin{equation*}
        |\phi(x)-\sigma(x)|=|\phi_{\delta_0}(x)-\sigma(x)|<\varepsilon \quad \tn{for any $x\in [-M,M]$}.
    \end{equation*}
    So we finish the proof.
\end{proof}

\subsection{Proof of Lemma~\ref{lem:net:varrho:to:tildevarrho}}
\label{sec:proof:lem:net:varrho:to:tildevarrho}

Let the activation function be applied to a vector elementwisely.
Then $\phi_\varrho$ can be represented in a form of function compositions as follows:
\begin{equation*}
    \phi_\varrho(\bmx) =\bmcalL_L\circ\varrho\circ\bmcalL_{L-1}\circ  \ \cdots \  \circ \varrho\circ\bmcalL_1\circ\varrho\circ\bmcalL_0(\bmx)\quad \tn{for any $\bmx\in\R^d$},
\end{equation*}
where $N_0=d$, $N_1,N_2,\cdots,N_L\in\N^+$,  $N_{L+1}=1$, $\bm{A}_\ell\in \R^{N_{\ell+1}\times N_{\ell}}$ and $\bm{b}_\ell\in \R^{N_{\ell+1}}$ are the weight matrix and the bias vector in the $\ell$-th affine linear transform $\bmcalL_\ell:\bmy \mapsto \bmA_\ell\bmy+\bmb_\ell$ for each $\ell\in \{0,1,\cdots,L\}$.
Define 
\begin{equation*}
    \phi_{\varrho_\delta}(\bmx) \coloneqq\bmcalL_L\circ{\varrho_\delta}\circ\bmcalL_{L-1}\circ  \ \cdots \  \circ {\varrho_\delta}\circ\bmcalL_1\circ\varrho_\delta\circ\bmcalL_0(\bmx)\quad \tn{for any $\bmx\in\R^d$}.
\end{equation*}
Recall that $\varrho_\delta$ can be realized by a $\tildevarrho$-activated network with width $\tildeN$ and depth $\tildeL$. Thus, $\phi_{\varrho_\delta}$ can be realized by a $\tildevarrho$-activated network with width $N\cdot\tildeN$ and depth $L\cdot\tildeL$.
We will prove 
\begin{equation*}
    \phi_{\varrho_\delta}(\bmx) 
    \rightrightarrows
    \phi_{\varrho}(\bmx)
    \quad \tn{as}\quad \delta\to 0^+
    \quad \tn{for any $\bmx\in[a,b]^d$}.
\end{equation*}

For any $\bmx\in\R^d$ and each $\ell\in \{1,2,\cdots,L+1\}$, define
\begin{equation*}
    \bmh_\ell(\bmx)
    \coloneqq \bmcalL_{\ell-1}\circ\varrho\circ\bmcalL_{\ell-2}\circ  \ \cdots \  \circ \varrho\circ\bmcalL_1\circ\varrho\circ\bmcalL_0(\bmx)
\end{equation*}
and 
\begin{equation*}
    \bmh_{\ell,\delta}(\bmx)
    \coloneqq \bmcalL_{\ell-1}\circ\varrho_\delta\circ\bmcalL_{\ell-2}\circ  \ \cdots \  \circ \varrho_\delta\circ\bmcalL_1\circ\varrho_\delta\circ\bmcalL_0(\bmx).
\end{equation*}
Note that $\bmh_{\ell}$ and $\bmh_{\ell,\delta}$ are two maps from $\R^d$ to $\R^{N_\ell}$ for each $\ell$.

We will prove by induction that
\begin{equation}\label{eq:induction:h:ell}
    \bmh_{\ell,\delta}(\bmx)\rightrightarrows \bmh_{\ell}(\bmx) \quad \tn{as}\quad  \delta\to 0^+
\end{equation}
for any $\bmx\in [a,b]^d$ and each $\ell\in \{1,2,\cdots,L+1\}$.

First, we consider the case $\ell=1$. Clearly,
\begin{equation*}
    \bmh_{1,\delta}(\bmx)=\bmcalL_0(\bmx)= \bmh_{1}(\bmx) \quad \tn{as}\quad \delta\to 0^+\quad \tn{for any $\bmx\in[a,b]^d$.}
\end{equation*}
This means Equation~\eqref{eq:induction:h:ell} holds for $\ell=1$.

Next, suppose Equation~\eqref{eq:induction:h:ell} holds for $\ell=i\in \{1,2,\cdots,L\}$. Our goal is to prove that it also holds for $\ell=i+1$. Determine $M>0$ by defining
\begin{equation*}
    M\coloneqq  \sup \Big\{\|\bmh_j(\bmx)\|_{\infty}+1:
    \bmx\in [a,b]^d,\quad j=1,2,\cdots,L+1\Big\},
\end{equation*}
where the continuity of $\varrho$ guarantees the above supremum is finite, i.e., $M\in (1,\infty)$.
By the induction hypothesis, we have
\begin{equation*}
    \bmh_{i,\delta}(\bmx)\rightrightarrows \bmh_{i}(\bmx) \quad \tn{as}\quad  \delta\to 0^+\quad \tn{for any $\bmx\in [a,b]^d$.}
\end{equation*}
Clearly, for any $\bmx\in [a,b]^d$, we have $\|\bmh_{i}(\bmx)\|_{\infty}\le M$ and  $\|\bmh_{i,\delta}(\bmx)\|_{\infty} \le \|\bmh_{i}(\bmx)\|_{\infty}+1\le  M$  for any small $\delta>0$.

 Recall the fact $\varrho_\delta(t)\rightrightarrows \varrho(t)$ as $\delta\to 0^+$ for any $t\in [-M,M]$. Then, we have 
\begin{equation*}
    \varrho_\delta\circ \bmh_{i,\delta}(\bmx)
    -\varrho\circ \bmh_{i,\delta}(\bmx)
    \rightrightarrows \bmzero\quad \tn{as}\quad \delta\to 0^+\quad \tn{for any $\bmx\in [a,b]^d$.}
\end{equation*}
The continuity of $\varrho$ implies
the uniform continuity of $\varrho$ on $[-M,M]$, from which we deduce
\begin{equation*}
    \varrho\circ \bmh_{i,\delta}(\bmx)
    -
    \varrho\circ\bmh_{i}(\bmx) 
    \rightrightarrows \bmzero
    \quad \tn{as}\quad  \delta\to 0^+\quad \tn{for any $\bmx\in [a,b]^d$.}
\end{equation*}

Therefore, for any $\bmx\in [a,b]^d$, as $\delta\to 0^+$, we have
\begin{equation*}
    \varrho_\delta\circ \bmh_{i,\delta}(\bmx)
    -
    \varrho\circ\bmh_{i}(\bmx) 
    =
    \underbrace{
    \varrho_\delta\circ \bmh_{i,\delta}(\bmx)
    -\varrho\circ \bmh_{i,\delta}(\bmx)
    }_{\rightrightarrows \bmzero}
    +
    \underbrace{\varrho\circ \bmh_{i,\delta}(\bmx)
    -
    \varrho\circ\bmh_{i}(\bmx) 
    }_{\rightrightarrows \bmzero}
    \rightrightarrows \bmzero,
\end{equation*}
implying
\begin{equation*}
    \bmh_{i+1,\delta}(\bmx)=\bmcalL_i\circ\varrho_\delta\circ \bmh_{i,\delta}(\bmx)
    \rightrightarrows
    \bmcalL_i\circ\varrho\circ \bmh_{i}(\bmx)=\bmh_{i+1}(\bmx).
\end{equation*}
This means Equation~\eqref{eq:induction:h:ell} holds for $\ell=i+1$. So we complete the inductive step.

By the principle of induction,  we have
\begin{equation*}
    \phi_{\varrho_{\delta}}(\bmx) =\bmh_{L+1,\delta}(\bmx)
    \rightrightarrows
    \bmh_{L+1}(\bmx)=
    \phi_{\varrho}(\bmx)
    \quad \tn{as}\quad \delta\to 0^+
    \quad \tn{for any $\bmx\in[a,b]^d$}.
\end{equation*}
There exists a small $\delta_0>0$ such that 
\begin{equation*}
    \big|\phi_{\varrho_{\delta_0}}(\bmx) -
    \phi_{\varrho}(\bmx)\big|<\varepsilon/2
    \quad \tn{for any $\bmx\in[a,b]^d$}.
\end{equation*}
By defining $\phi\coloneqq\phi_{\varrho_{\delta_0}}$, 
we have
\begin{equation*}
    \big|\phi(\bmx) -
    f(\bmx)\big|
    \le \big|\phi_{\varrho_{\delta_0}}(\bmx) -
    \phi_{\varrho}(\bmx)\big|
    +\big|\phi_{\varrho}(\bmx) -
    f(\bmx)\big|
    <\varepsilon/2+\varepsilon/2=\varepsilon
\end{equation*}
for any $\bmx\in[a,b]^d$.
Moreover, $\phi=\phi_{\varrho_{\delta_0}}$ can be generated by a $\tildevarrho$-activated network with width $N\cdot\tildeN$ and depth $L\cdot\tildeL$.
So we finish the proof.

\section{Detailed Proofs of Theorems~\ref{thm:main} and \ref{thm:main:classification}}
\label{sec:proof:main:and:classification}

In this section, we will give the detailed proofs of Theorems~\ref{thm:main} and \ref{thm:main:classification}. First, we prove Theorem~\ref{thm:main} based on Theorem~\ref{thm:main:d=1}, which will be proved in Section~\ref{sec:proof:main:d=1}. Next, we apply Theorem~\ref{thm:main} to prove Theorem~\ref{thm:main:classification}.

\subsection{Proof of Theorem~\ref{thm:main}}\label{sec:proof:main}

The detailed proof  of Theorem~\ref{thm:main} converts the above ideas mentioned in Section~\ref{sec:key:ideas} to implementations using neural networks with fixed sizes. The whole construction procedure can be divided into three steps.
\begin{enumerate}[(1)] 
    \item Apply KST to reduce dimension, i.e.,  represent $f\in C([a,b]^d)$ by the compositions and combinations of univariate continuous functions.
    \item Apply Theorem~\ref{thm:main:d=1} to design sub-networks to approximate the univariate continuous functions in the previous step within the desired error. 
    \item Integrate the sub-networks to form the final network and estimate its size.
\end{enumerate}

The details of these three steps can be found below.
\mystep{1}{  Apply KST to reduce dimension.    }
To apply KST,	we define a linear function $\calL_1(t)=(b-a)t+a$ for any $t\in [0,1]$. 
	Clearly, $\calL_1$ is a bijection from $[0,1]$ to $[a,b]$.
	Define
	 \begin{equation*}
	 	\tildef(\bmy)\coloneqq  f\big(\calL_1(y_1),\calL_1(y_2),\cdots,\calL_1(y_d)\big)\quad\tn{ for any  } \bmy=[y_1,y_2,\cdots,y_d]^T\in [0,1]^d.
	 \end{equation*}
	Then, $\tildef:[0,1]^d\to\R$ is a continuous function since $f\in C([a,b]^d)$.
	By Theorem~\ref{thm:kst}, there exists $\tildeh_{i,j}\in C([0,1])$ and $\tildeg_i\in C(\R)$ for $i=0,1,\cdots,2d$ and $j=1,2,\cdots,d$ such that 
	\begin{equation*}
		\tildef(\bmy)=\sum_{i=0}^{2d}\tildeg_i \Big(\sum_{j=1}^d \tildeh_{i,j}(y_j)\Big)\quad \tn{for any $\bmy=[y_1,y_2,\cdots,y_d]^T\in [0,1]^d$.}
	\end{equation*}
Let $\caltildeL_1$ be the inverse of $\calL_1$, i.e., $\caltildeL_1(t)=(t-a)/(b-a)$ for any $t\in [a,b]$. Then, for any $x_j\in [a,b]$, there exists a unique $y_j\in [0,1]$ such that $\calL_1(y_j)=x_j$ and $y_j=\caltildeL_1(x_j)$ for any $j=1,2,\cdots,d$, which implies
\begin{equation*}
	\begin{split}
		f(\bmx)&=f(x_1,x_2,\cdots,x_d)=f\big(\calL_1(y_1),\calL_1(y_2),\cdots,\calL_1(y_d)\big)=\tildef(\bmy)\\
		& =\sum_{i=0}^{2d}\tildeg_i \Big(\sum_{j=1}^d \tildeh_{i,j}(y_j)\Big)
		=\sum_{i=0}^{2d}\tildeg_i \Big(\sum_{j=1}^d \tildeh_{i,j}\big(\caltildeL_1(x_j)\big)\Big)
		=\sum_{i=0}^{2d}\tildeg_i \Big(\sum_{j=1}^d \tildeh_{i,j}\circ\caltildeL_1(x_j)\Big).
	\end{split}
\end{equation*}
It follows that
\begin{equation*}
	\begin{split}
		f(\bmx)=\sum_{i=0}^{2d}\tildeg_i \Big(\sum_{j=1}^d \tildeh_{i,j}\circ\caltildeL_1(x_j)\Big)
		=\sum_{i=0}^{2d}\tildeg_i \circ \hath_i(\bmx)\quad\tn{for any $\bmx\in [a,b]^d$,}
	\end{split}
\end{equation*}
where
\begin{equation}\label{eq:tildehi}
 	\hath_i(\bmx)=\sum_{j=1}^d \tildeh_{i,j}\circ\caltildeL_1(x_j)\quad\tn{for any $\bmx=[x_1,x_2,\cdots,x_d]^T\in[a,b]^d$.}
\end{equation}

Set 
\begin{equation*}
	M=\max_{i\in \{0,1,\cdots,2d\}} \|\hath_i\|_{L^\infty([a,b]^d)}+1>0.
\end{equation*}
Define $\calL_2(t)=(t+2M)/4M$ and $\caltildeL_2(t)=4Mt-2M$ for any $t\in \R$.
Then, $\calL_2$ is a bijection from $[-M,M]$ to $[\tfrac{1}{4},\tfrac{3}{4}]$ and  $\caltildeL_2$ is the inverse of $\calL_2$.
Clearly,
$\caltildeL_2\circ\calL_2(t)=t$ for any $t\in [-M,M]$, which implies $\hath_i(\bmx)=\caltildeL_2\circ\calL_2\circ  \hath_i(\bmx)$ for any $\bmx\in [a,b]^d$. Therefore, for any $\bmx\in [a,b]^d$, we have
\begin{equation*}
	\begin{split}
		f(\bmx)
		=\sum_{i=0}^{2d}\tildeg_i \circ\hath_i(\bmx)
		=\sum_{i=0}^{2d}\tildeg_i \circ\caltildeL_2\circ\calL_2\circ  \hath_i(\bmx)
		=\sum_{i=0}^{2d}g_i\circ  h_i(\bmx),		
	\end{split}
\end{equation*}
where 
\begin{equation}\label{eq:gi:hi}
    g_i=\tildeg_i\circ \caltildeL_2 \tn{\quad  and\quad } h_i=\calL_2\circ \hath_i\tn{\quad for $i=0,1,\cdots,2d$.}
\end{equation}
Clearly, $\calL_2(t)\in [\tfrac{1}{4},\tfrac{3}{4}]$ for any $t\in [-M,M]$, which implies
\begin{equation*}
	h_i(\bmx)=\calL_2\circ \hath_i(\bmx)\in [\tfrac{1}{4},\tfrac{3}{4}]\quad\tn{for any $\bmx\in [a,b]^d$ and $i=0,1,\cdots,2d$.}
\end{equation*}

\mystep{2}{  Design sub-networks to approximate $g_i$ and $h_i$.    }

Next, we will construct sub-networks to approximate $g_i$ and $h_i$ for each $i$. 
Obviously, $g_i=\tildeg_i\circ \caltildeL_2$ is continuous on $\R$ and hence uniformly continuous on $[0,1]$ for each $i$. Thus, for  $i=0,1,\cdots,2d$,
there exists $\delta_i>0$ such that
\begin{equation*}
	|g_i(z_1)-g_i(z_2)|<\varepsilon/(4d+2)\quad \tn{for any $z_1,z_2\in [0,1]$ with $|z_1-z_2|<\delta_i$.}
\end{equation*}
Set $\delta=\min \big(\{\delta_i:i=0,1,\cdots,2d\}\bigcup\{\tfrac{1}{4}\}\big)$. Then, for $i=0,1,\cdots,2d$, we have 
\begin{equation}\label{eq:error:g:1}
	|g_i(z_1)-g_i(z_2)|<\varepsilon/(4d+2)\quad \tn{for any $z_1,z_2\in [0,1]$ with $|z_1-z_2|<\delta$.}
\end{equation}

For each $i\in\{0,1,\cdots,2d\}$, by Theorem~\ref{thm:main:d=1}, there exists a function $\phi_i$ generated by an EUAF network with width $36$ and depth $5$ such that 
\begin{equation}\label{eq:error:g:2}
	|g_i(z)-\phi_i(z)|<\varepsilon/(4d+2)\quad \tn{for any $z\in [0,1]$.}
\end{equation}

Fix $i\in\{0,1,\cdots,2d\}$, we will design an EUAF network to generate a function $\psi_i:[a,b]^d\to\R$ satisfying
\begin{equation*}
	|h_i(\bmx)-\psi_i(\bmx)|<\delta\quad \tn{for any $\bmx\in[a,b]^d$.}
\end{equation*}
For any $\bmx=[x_1,x_2,\cdots,x_d]^T\in [a,b]^d$, by Equations~\eqref{eq:tildehi} and \eqref{eq:gi:hi}, we have
\begin{equation*}
\begin{split}
		h_i(\bmx)=\calL_2\circ \hath_i(\bmx)
		=\calL_2\Big( \sum_{j=1}^d \tildeh_{i,j}\circ\caltildeL_1(x_j)\Big)
		&=\frac{\Big( \sum_{j=1}^d \tildeh_{i,j}\circ\caltildeL_1(x_j)\Big)+2M}{4M}\\
		&=\sum_{j=1}^d \Big(\frac{\tildeh_{i,j}\circ\caltildeL_1(x_j)}{4M}+\frac{1}{2d}\Big)
		= \sum_{j=1}^d h_{i,j}(x_j),
\end{split}
\end{equation*}
where 
\begin{equation*}
	h_{i,j}(t)\coloneqq \frac{\tildeh_{i,j}\circ\caltildeL_1(t)}{4M}+\frac{1}{2d}\quad \tn{for any $t\in [a,b]$, $i=0,1,\cdots,2d$, and $j=1,2,\cdots,d$.}
\end{equation*}

It is easy to verify that $h_{i,j}\in C([a,b]^d)$  each $i\in \{0,1,\cdots,2d\}$ and each $j\in \{1,2,\cdots,d\}$, 
from which we deduce by Theorem~\ref{thm:main:d=1} that there exists a function $\psi_{i,j}$ generated by an EUAF network with width $36$ and depth $5$ such that
\begin{equation*}
	|h_{i,j}(t)-\psi_{i,j}(t)|<\delta/d\quad \tn{for any $t\in [a,b]$.}
\end{equation*}

For each $i\in \{0,1,\cdots,2d\}$, we define 
\begin{equation*}
    \psi_i(\bmx)\coloneqq \sum_{j=1}^d \psi_{i,j}(x_j)\quad \tn{ for any $\bmx=[x_1,x_2,\cdots,x_d]^T\in[a,b]^d$.}
\end{equation*}
Then, for any $\bmx=[x_1,x_2,\cdots,x_d]^T\in[a,b]^d$ and  each $i\in \{0,1,\cdots,2d\}$, we have
\begin{equation*}
	\begin{split}
		|h_i(\bmx)-\psi_i(\bmx)|=\Big| \sum_{j=1}^d h_{i,j}(x_j)-\sum_{j=1}^{d}  \psi_{i,j}(x_j)   \Big|
		=\sum_{j=1}^d\Big|  h_{i,j}(x_j)-  \psi_{i,j}(x_j)   \Big|<\sum_{j=1}^d \delta/d=\delta.
	\end{split}
\end{equation*}

\mystep{3}{Integrate sub-networks.}

Finally, we build an integrated network with the desired size to approximate the target function $f$. 
The desired function $\phi$ can be defined as 
\begin{equation*}
	\phi(\bmx)\coloneqq\sum_{i=0}^{2d}\phi_i\circ\psi_i(\bmx)=\sum_{i=0}^{2d}\phi_i\Big(\sum_{j=1}^d \psi_{i,j}(x_j)\Big)\quad \tn{for any $\bmx=[x_1,x_2,\cdots,x_d]^T\in[a,b]^d$.}
\end{equation*}
Let us first estimate the approximation error and then determine the size of the target network realizing $\phi$. See Figure~\ref{fig:phi:decomposition} for an illustration of the target network realizing $\phi$ for the case $d=2$. 

\begin{figure}[htbp!]
	\centering
	\includegraphics[width=0.8945\textwidth]{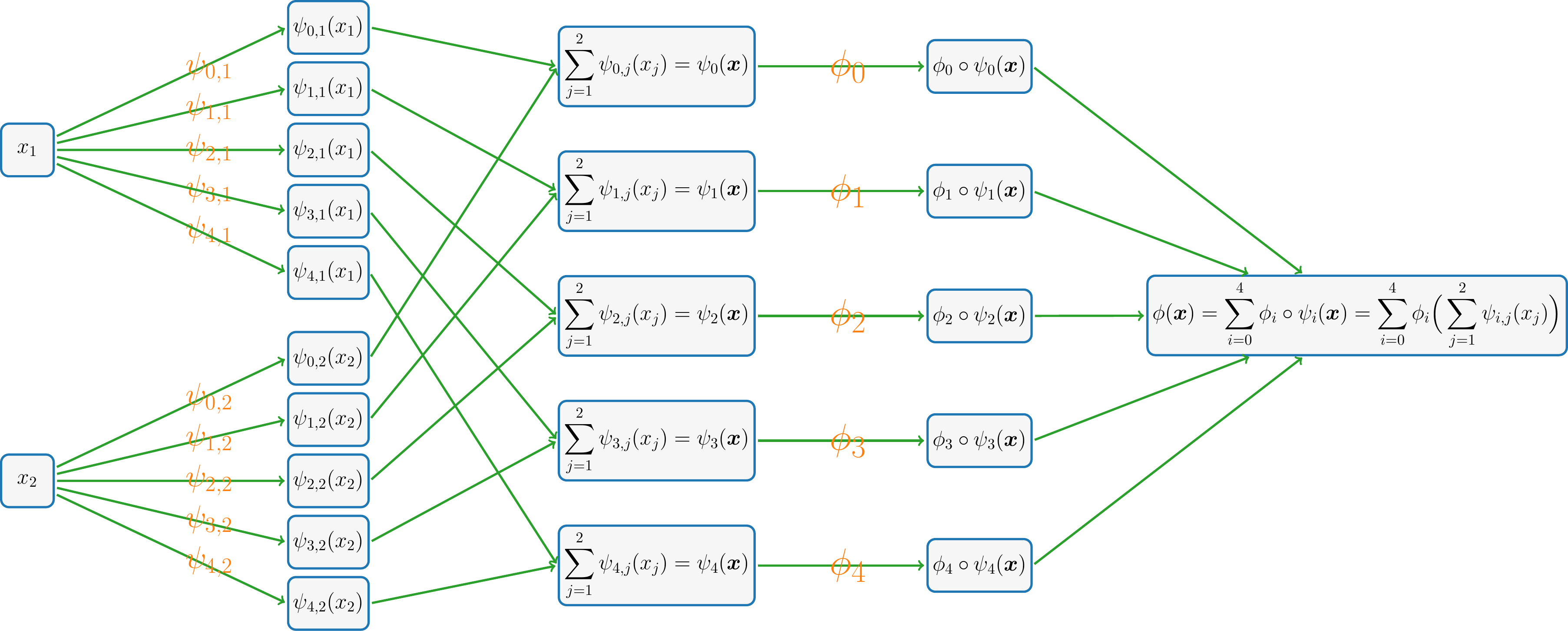}
	\caption{An illustration of the target network realizing  $\phi$ for any $\bmx\in [a,b]^d$ in the case of $d=2$. This network contains $(2d+1)d+(2d+1)=(d+1)(2d+1)$ sub-networks that realize $\psi_{i,j}$ and $\phi_i$ for $i=0,1,\cdots,2d$ and $j=1,2,\cdots,d$.}
	\label{fig:phi:decomposition}
\end{figure}

Fix $\bmx\in [a,b]^d$ and $i\in\{0,1,\cdots,2d\}$.
Recall that $h_i(\bmx)\in [\tfrac{1}{4},\tfrac{3}{4}]$ and 
\begin{equation*}
    |h_i(\bmx)-\psi_i(\bmx)|<\delta\le \tfrac{1}{4},
\end{equation*}
implying $\psi_i(\bmx)\in [0,1]$. Then, by Equation~\eqref{eq:error:g:1} (set $z_1=h_i(\bmx)$ and $z_2=\psi_i(\bmx)$ therein), we have
\begin{equation*}
	\Big|  g_i\circ  h_i(\bmx)- g_i\circ \psi_i(\bmx)\Big|
	=\Big|  g_i\big(h_i(\bmx)\big)- g_i\big(\psi_i(\bmx)\big)\Big|
	<\varepsilon/(4d+2).
\end{equation*}
By Equation~\eqref{eq:error:g:2} (set $z=\psi_i(\bmx)\in [0,1]$ therein), we have
\begin{equation*}
	\Big|  g_i\circ  \psi_i(\bmx)- \phi_i\circ \psi_i(\bmx)\Big|
	=\Big|  g_i\big(\psi_i(\bmx)\big)- \phi_i\big(\psi_i(\bmx)\big)\Big|
	<\varepsilon/(4d+2).
\end{equation*}
Therefore, for any $\bmx\in [a,b]^d$, we have
\begin{equation*}
	\begin{split}
		\big|f(\bmx)-\phi(\bmx)\big|
		&=\Big|\sum_{i=0}^{2d}g_i\circ  h_i(\bmx)-\sum_{i=0}^{2d}\phi_i\circ\psi_i(\bmx)\Big|
		=\sum_{i=0}^{2d}\Big|g_i\circ  h_i(\bmx)-\phi_i\circ\psi_i(\bmx)\Big|\\
		&\le  \sum_{i=0}^{2d}\bigg(\Big|g_i\circ  h_i(\bmx)-g_i\circ\psi_i(\bmx)\Big|+\Big|g_i\circ  \psi_i(\bmx)-\phi_i\circ\psi_i(\bmx)\Big|\bigg)\\
		&< \sum_{i=0}^{2d}\Big(\varepsilon/(4d+2)+\varepsilon/(4d+2)\Big)=\varepsilon.
	\end{split}
\end{equation*}
It remains to show $\phi$ can be generated by an EUAF network with the desired size.
Recall that, for each $i\in \{0,1,\cdots,2d\}$ and each $j\in \{1,2,\cdots,d\}$,  $\psi_{i,j}$ can be generated by an EUAF network with width $36$, depth $5$, and therefore at most 
\begin{equation*}
	(1\times 36+36)\   +\   (36\times 36+36)\times 4\  
	+\    (36\times 1+1)=5437
\end{equation*} 
nonzero parameters. Hence, for each $i\in \{0,1,\cdots,2d\}$, $\psi_i$, given by $\psi_i(\bmx)=\sum_{j=1}^d \psi_{i,j}(x_j)$, 
can be generated by an EUAF network with width $36d$, depth $5$, and at most $5437d$ nonzero parameters. 

Since $\psi_i(\bmx)\in [0,1]$ for any $\bmx\in [a,b]^d$ and $i=0,1,\cdots,2d$, we have $\sigma\big(\psi_i(\bmx) \big)= \psi_i(\bmx)$ for any $\bmx\in [a,b]^d$. Hence, $\phi_i\circ\psi_i$ can be generated by an EUAF network as visualized in Figure~\ref{fig:phii:circ:psii}.

\begin{figure}[htbp!]
	\centering
	\includegraphics[width=0.677\textwidth]{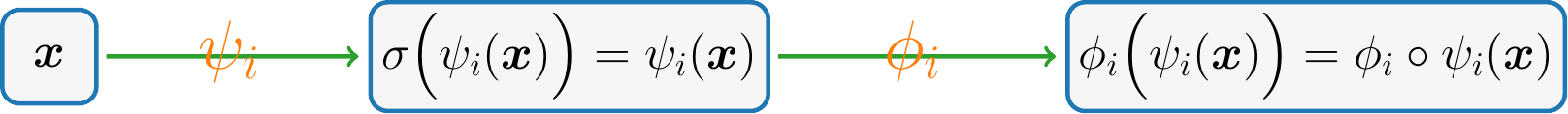}
	\caption{An illustration of the target EUAF network generating  $\phi_i\circ\psi_i(\bmx)$ for any $\bmx\in [a,b]^d$ and $i=0,1,\cdots,2d$.}
	\label{fig:phii:circ:psii}
\end{figure}

Recall that $\phi_i$ can be generated by an EUAF network with width $36$ and depth $5$. Hence, the network generating $\phi_i$ has at most $5437$ nonzero parameters.
As we can see from Figure~\ref{fig:phii:circ:psii}, $\phi_i\circ\psi_i$ can be generated by an EUAF network with width $\max\{36d,36\}=36d$, depth $5+1+5=11$, and at most $5437d+5437=5437(d+1)$ nonzero parameters. This means $\phi=\sum_{i=0}^{2d}\phi_i\circ\psi_i$ can be generated by an EUAF network with width $36d(2d+1)$, depth $11$, and therefore at most $5437(d+1)(2d+1)$ nonzero parameters as desired. So we finish the proof.

\subsection{Proof of Theorem~\ref{thm:main:classification}}
\label{sec:proof:thm:main:classification}

The proof of Theorem~\ref{thm:main:classification} relies on a basic result of real analysis given in the following lemma.
\begin{lemma}\label{lem:continuous:indicator}
Suppose $A,B\subset \R^d $ are two disjoint bounded closed sets. Then, there exists a continuous function $f\in C(\R^d)$ such that
	$f(\bmx)=1$ for any $\bmx\in A$ and $f(\bmy)=0$ for any $\bmy\in B$.
\end{lemma}
\begin{proof}
	Define $\tn{dist}(\bmx,A)=\inf\{\|\bmx-\bmy\|_2:\bmy\in A\}$ and $\tn{dist}(\bmx,B)=\inf\{\|\bmx-\bmy\|_2:\bmy\in B\}$ for any $\bmx\in\R^d$. It is easy to verify that $\tn{dist}(\bmx,A)$ and $\tn{dist}(\bmx,B)$ are continuous in $\bmx\in\R^d$. 
	Since $A,B\in\R^d $ are two disjoint bounded closed subsets, we have $\tn{dist}(\bmx,A)+\tn{dist}(\bmx,B)>0$ for any $\bmx\in \R^d$. 
	Finally, define 
	\begin{equation*}
		f(\bmx)\coloneqq\frac{\tn{dist}(\bmx,B)}{\tn{dist}(\bmx,A)+\tn{dist}(\bmx,B)}\quad \tn{for any $\bmx\in\R^d$.}
	\end{equation*} 
	Then $f$ meets the requirements. So we finish the proof.
\end{proof}

With Lemma~\ref{lem:continuous:indicator}, we can prove Theorem~\ref{thm:main:classification}.
\begin{proof}[Proof of Theorem~\ref{thm:main:classification}]
For any $f=\sum_{j=1}^J r_j\cdot \one_{E_j}\in \scrC_d(E_1,E_2,\cdots,E_J)$, our goal is to construct a function $\phi$ generated by a $\sigma$-activated network  such that $\phi(\bmx)=f(\bmx)$ for any $\bmx\in \bigcup_{j=1}^J E_j$, where $E_1,E_2,\cdots,E_J$ are pairwise disjoint bounded closed subsets of $\R^d$.
Define  $E\coloneqq \bigcup_{j=1}^J E_j$ and choose $a,b\in\R$ properly such that $E\subseteq [a,b]^d$. 
	
	 For each $j\in \{1,2,\cdots,J\}$,
	$E_j$ and $\tildeE_j\coloneqq  E \big\backslash E_j$ are two disjoint bounded closed subsets. Then, for each $j$, by Lemma~\ref{lem:continuous:indicator}, there exists $g_j\in C(\R^d)$ such that $g_j(\bmx)=1$ for any $\bmx\in E_j$ and $g_j(\bmy)=0$ for any $\bmy\in \tildeE_j=E\backslash  E_j$.
	By defining $g\coloneqq \sum_{j=1}^{J}r_j\cdot g_j \in C(\R^d)$, we have $g(\bmx)=\sum_{j=1}^J r_j\cdot \one_{E_j}(\bmx)=f(\bmx)$ for any $\bmx \in E=\bigcup_{j=1}^J E_j$.

	Since $r_1,r_2,\cdots,r_J$ are rational numbers and $g:[a,b]^d\to \R$ is continuous, there exist $n_1,n_2\in \Z\backslash \{0\}$ such that 
	\begin{itemize}
	    \item $n_1\cdot r_j+n_2\in\N^+$ for $j=1,2,\cdots,J$;
	    \item $n_1\cdot   g(\bmx)+n_2\ge 0$ for any $\bmx\in [a,b]^d$.
	\end{itemize}
	
	By applying Theorem~\ref{thm:main} to $2(n_1\cdot  g+n_2)+1\in C([a,b]^d)$, there exists a function $\phi_1$ generated by an EUAF network with width $36d(2d+1)$, depth $11$, and at most $5437(d+1)(2d+1)$ nonzero parameters such that
	\begin{equation}\label{eq:fJ12:phi}
		\Big|2\big(n_1\cdot  g(\bmx)+n_2\big)+1-\phi_1(\bmx)\Big|\le 1/2\quad \tn{for any $\bmx\in [a,b]^d$.}
	\end{equation}
It follows that
	\begin{equation*}
		\Big|2\big(n_1\cdot  \sum_{j=1}^J r_j\cdot \one_{E_j}(\bmx)+n_2\big)+1-\phi_1(\bmx)\Big|\le 1/2\quad \tn{for any $\bmx\in E=\bigcup_{j=1}^J  E_j$.}
	\end{equation*}

	Since $E_1,E_2,\cdots,E_J$ are pairwise disjoint, we have
	\begin{equation}\label{eq:odd:half:ball}
			\Big|2(n_1\cdot  r_j +n_2)+1-\phi_1(\bmx)\Big|\le 1/2\quad \tn{for any $\bmx\in E_j$ and each $j\in\{1,2,\cdots,J\}$.}
	\end{equation}
Define $\phi_2(x)=x+1/2- \sigma(x+3/2)$ for any $x\in \R$. See Figure~\ref{fig:phi2} for an illustration.

\begin{figure}[htbp!]
	\centering
	\includegraphics[width=0.63045\textwidth]{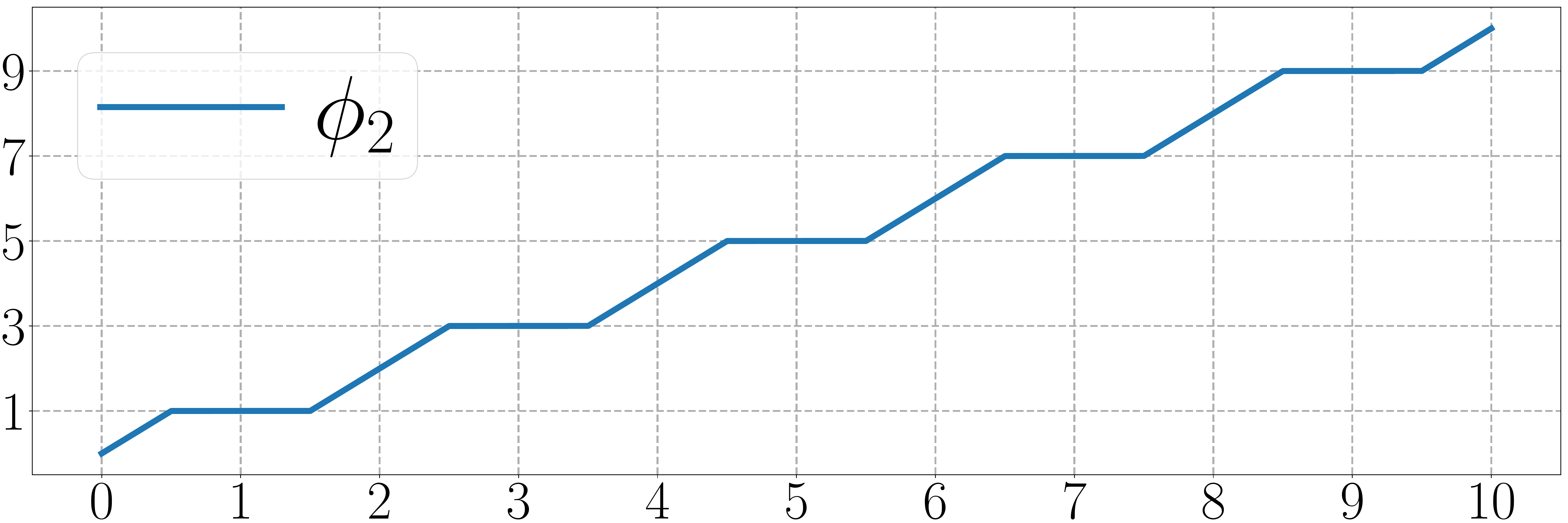}
	\caption{An illustration of $\phi_2$ on $[0,10]$.}
	\label{fig:phi2}
\end{figure}

It is easy to verify that 
\begin{equation}\label{eq:phi2:2k+1}
    \phi_2(y)=2k+1\quad \tn{ for any $y$ and $k\in \N^+$ with $|2k+1-y|\le 1/2$.}
\end{equation}
Then, by Equations~\eqref{eq:odd:half:ball} and $\eqref{eq:phi2:2k+1}$ (set $y=\phi_1(\bmx)$ and $k=n_1\cdot r_j+n_2$ therein), we have
\begin{equation*}
    \phi_2\circ\phi_1(\bmx)=\phi_2(y)=2k+1=2(n_1\cdot  r_j+n_2)+1 
\end{equation*}
for any $\bmx\in E_j$ and any $j\in\{1,2,\cdots,J\}$, which implies 
\begin{equation*}
	\frac{\phi_2\circ\phi_1(\bmx)-2n_2-1}{2n_1}=r_j\quad \tn{for any $\bmx\in E_j$ and any $j\in\{1,2,\cdots,J\}$}.
\end{equation*}
Define 
\begin{equation*}
    \phi(\bmx)\coloneqq \frac{\phi_2\circ\phi_1(\bmx)-2n_2-1}{2n_1}\quad \tn{ for any $\bmx\in[a,b]^d$. }
\end{equation*}
Clearly, we have 
$\phi(\bmx)=r_j$ for any $\bmx\in E_j$ and each $j\in\{1,2,\cdots,J\}$, which implies 
\begin{equation*}
    \phi(\bmx)=\sum_{j=1}^J r_j\cdot \one_{E_j}(\bmx)=f(\bmx)\quad \tn{ for any $\bmx\in E=\bigcup_{j=1}^J E_j$.}
\end{equation*}

It remains to show that $\phi$ can be generated by an EUAF network with the desired size.
Set $M=2\|n_1 g+n_2\|_{L^\infty([a,b]^d)}+3/2>0$. By Equation~\eqref{eq:fJ12:phi} and  the fact $n_1\cdot g(\bmx)+n_2\ge 0$ for any $\bmx\in [a,b]^d$,  we have 
\begin{equation*}
    \phi_1(\bmx)\in \Big[1/2,\   2\|n_1 g+n_2\|_{L^\infty([a,b]^d)}+1+1/2\Big]\subseteq [0,M]\quad \tn{for any $\bmx\in [a,b]^d$.}
\end{equation*}
Then, for any $\bmx\in [a,b]^d$, we have 
\begin{equation*}
    \begin{split}
        	\phi_2\circ\phi_1(\bmx)
        	&=\phi_1(\bmx)+1/2 -\sigma\big(\phi_1(\bmx)+3/2\big)\\
        	&= M\sigma\big(\phi_1(\bmx)/M\big)+1/2 -\sigma\big(\phi_1(\bmx)+3/2\big).
    \end{split}
\end{equation*}
It follows that
\begin{equation*}
	\phi(\bmx)=\frac{\phi_2\circ\phi_1(\bmx)-2n_2-1}{2n_1}= \frac{M\sigma\big(\phi_1(\bmx)/M\big) -\sigma\big(\phi_1(\bmx)+3/2\big)-2n_2-1/2}{2n_1},
\end{equation*}
for any $\bmx\in [a,b]^d$. 
That means the network realizing $\phi$ has just one more hidden layer with $2$ neurons, compared to the network realizing $\phi_1$.
Recall that $\phi_1$ can be generated by an EUAF network with width $36d(2d+1)$, depth $11$, and at most $5437(d+1)(2d+1)$ nonzero parameters. Therefore, $\phi$, limited on $[a,b]^d$, can be generated by an EUAF network with width $36d(2d+1)$, depth $12$, and at most 
\begin{equation*}
	5437(d+1)(2d+1)\  +\   \underbrace{2\times 36d(2d+1) +2\   +\   2\times  1+1}_{\tn{all possible new parameters}} \le 5509(d+1)(2d+1)
\end{equation*}
nonzero parameters. So we finish the proof. 
\end{proof}

\section{Proof of Theorem~\ref{thm:main:d=1}}
\label{sec:proof:main:d=1}

To prove Theorem~\ref{thm:main:d=1}, we need to introduce two auxiliary theorems, Theorems~\ref{thm:main:d=1:half} and \ref{thm:main:d=1:small:region}, which serve as two important intermediate steps.
\begin{theorem}\label{thm:main:d=1:half}
    Let $f\in C([0,1])$ be a continuous function.
	Given any $\varepsilon>0$, if $K$ is a positive integer satisfying
	\begin{equation}\label{eq:f:error:1/K}
		|f(x_1)-f(x_2)|<\varepsilon/2\quad \tn{for any $x_1,x_2\in[0,1]$ with $|x_1-x_2|<1/K$,}
	\end{equation}
	then there exists a function $\phi$ generated by an EUAF network  with width $2$ and depth $3$ such that 
	$\|\phi\|_{L^\infty([0,1])}\le \|f\|_{L^\infty([0,1])}+1$ and
	\begin{equation*}
		|\phi(x)-f(x)|<\varepsilon\quad \tn{for any}\  x\in \bigcup_{k=0}^{K-1}\big[\tfrac{2k}{2K},\tfrac{2k+1}{2K}\big].
	\end{equation*}
\end{theorem}

\begin{theorem}\label{thm:main:d=1:small:region}
	Let $f\in C([0,1])$ be a continuous function.  Then, for any  $\varepsilon>0$, 
	there exists a function $\phi$ generated by an EUAF network  with width $36$ and depth $5$ such that
	\begin{equation*}
		|\phi(x)-f(x)|<\varepsilon\quad \tn{for any $x\in [0,\tfrac{9}{10}]$. }
	\end{equation*}
\end{theorem}

To prove Theorem~\ref{thm:main:d=1:half}, we only need to care about the approximation on one ``half'' of $[0,1]$. It is not necessary to care about the approximation on the other ``half'' of $[0,1]$. Such an idea is similar to the ``trifling region'' in \cite{shijun3,shijun:thesis}. As we shall see later, the proof of Theorem~\ref{thm:main:d=1:half} can eventually be converted to a point-fitting problem, which can be solved by applying Proposition~\ref{prop:dense}.

The key idea to prove Theorem~\ref{thm:main:d=1:small:region} is to  apply Theorem~\ref{thm:main:d=1:half} to several horizontally shifted variants of the target function. Then a good approximation can be constructed via the combinations and multiplications of these variants, similar to the idea of \cite{shijun3,shijun:thesis} to obtain an error estimation with the $L^\infty$-norm from a result with the $L^p$-norm for $p\in[1,\infty)$.

The proofs of Theorems~\ref{thm:main:d=1:half} and \ref{thm:main:d=1:small:region}  will be presented in Sections~\ref{sec:proof:main:d=1:half} and \ref{sec:proof:main:d=1:small:region}, respectively. Let us first prove Theorem~\ref{thm:main:d=1} by assuming Theorem~\ref{thm:main:d=1:small:region} is true.
\begin{proof}[Proof of Theorem~\ref{thm:main:d=1}]
	Define a linear function $\calL$ by 
	$\calL(x)=a+\tfrac{10(b-a)}{9}x$ for any $x\in [0,\tfrac{9}{10}]$.
	Then $\calL$ is a bijection  from $[0,\tfrac{9}{10}]$ to $[a,b]$. It follows that $f\circ\calL$ is a continuous function on $[0,\tfrac{9}{10}]$. By Theorem~\ref{thm:main:d=1:small:region}, there exists  a function $\tildephi$ generated by an EUAF network with width $36$ and depth $5$ such that 
	\begin{equation*}
		|f\circ\calL(x)-\tildephi(x)|<\varepsilon\quad\tn{for any $x\in [0,\tfrac{9}{10}]$.}
	\end{equation*}
	
Define  $\widetilde{\calL}(y)=\tfrac{9(y-a)}{10(b-a)}$ for any $y\in [a,b]$. Clearly, it is the inverse of $\calL$, i.e., $\calL\circ\widetilde{\calL}(y)=y$ for any $y\in [a,b]$. Therefore, for any $y\in [a,b]$, we have $x=\widetilde{\calL}(y)\in [0,\tfrac{9}{10}]$, which implies 
\begin{equation*}
	\begin{split}
		|f(y)-\tildephi\circ\widetilde{\calL}(y)|
		&=\big|f\circ\calL\circ\widetilde{\calL}(y)-\tildephi\circ\widetilde{\calL}(y)\big|\\
		&=\Big|f\circ\calL\big(\widetilde{\calL}(y)\big)-\tildephi\big(\widetilde{\calL}(y)\big)\Big|= |f\circ\calL(x)-\tildephi(x)|<\varepsilon.
	\end{split}
\end{equation*}
By defining $\phi\coloneqq \tildephi\circ\widetilde{\calL}$, we have $|f(y)-\phi(y)|<\varepsilon$ for any $y\in [a,b]$ as desired. 

Note that $\tildephi$ can be realized by an EUAF network with width $36$ and depth $5$. We can compose $\widetilde{\calL}$ and the affine linear map of the network $\tildephi$ that connects the input layer and the first hidden layer.
Therefore,
$\phi= \tildephi\circ\widetilde{\calL}$ can also be realized by an EUAF network with width $36$ and depth $5$. So we finish the proof.
\end{proof}

\subsection{Proof of Theorem~\ref{thm:main:d=1:half}}
\label{sec:proof:main:d=1:half}

Partition $[0,1]$ into $2K$ small intervals $\calI_k$ and $\caltildeI_k$ for $k=1,2,\cdots, K$, i.e.,
\begin{equation*}
	\calI_k=\big[\tfrac{2k-2}{2K},\tfrac{2k-1}{2K}\big]\quad 
	\tn{and}\quad \caltildeI_k=\big[\tfrac{2k-1}{2K},\tfrac{2k}{2K}\big].
\end{equation*}
Clearly, $[0,1]=\bigcup_{k=1}^K(\calI_k\cup\caltildeI_k)$.
Let $x_k$ be the right endpoint of $\calI_k$, i.e., $x_k=\tfrac{2k-1}{2K}$ for $k=1,2,\cdots,K$.
 See an illustration of $\calI_k$, $\caltildeI_k$, and $x_k$ in Figure~\ref{fig:calI} for the case $K=5$.

\begin{figure}[H]
	\centering
	\includegraphics[width=0.775\textwidth]{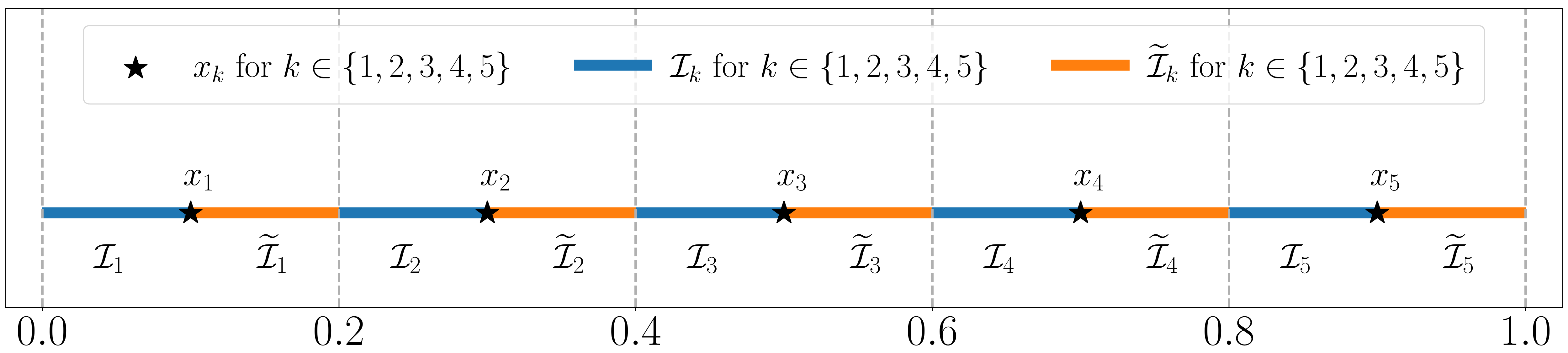}
	\caption{An illustration of $\calI_k$ and $\caltildeI_k$ for $k\in\{1,2,\cdots,K\}$ with $K=5$.}
	\label{fig:calI}
\end{figure}

Our goal is to construct a function $\phi$ generated by an EUAF network with the desired size to approximate $f$ well on 
 $\calI_k$ for  $k=1,2,\cdots,K$. It is not necessary to care about the values of $\phi$ on $\caltildeI_k$ for all $k$. In other words, we only need to care about the approximation on one ``half'' of $[0,1]$,  which is the key for our proof.

Define $\psi(x)\coloneqq  x-\sigma(x)$ for any $x\in\R$, where $\sigma$ is defined in Equation~\eqref{eq:def:sigma}. See Figure~\ref{fig:psi} for an illustration of $\psi$. 

\begin{figure}[htbp!]
	\centering
	\includegraphics[width=0.63045\textwidth]{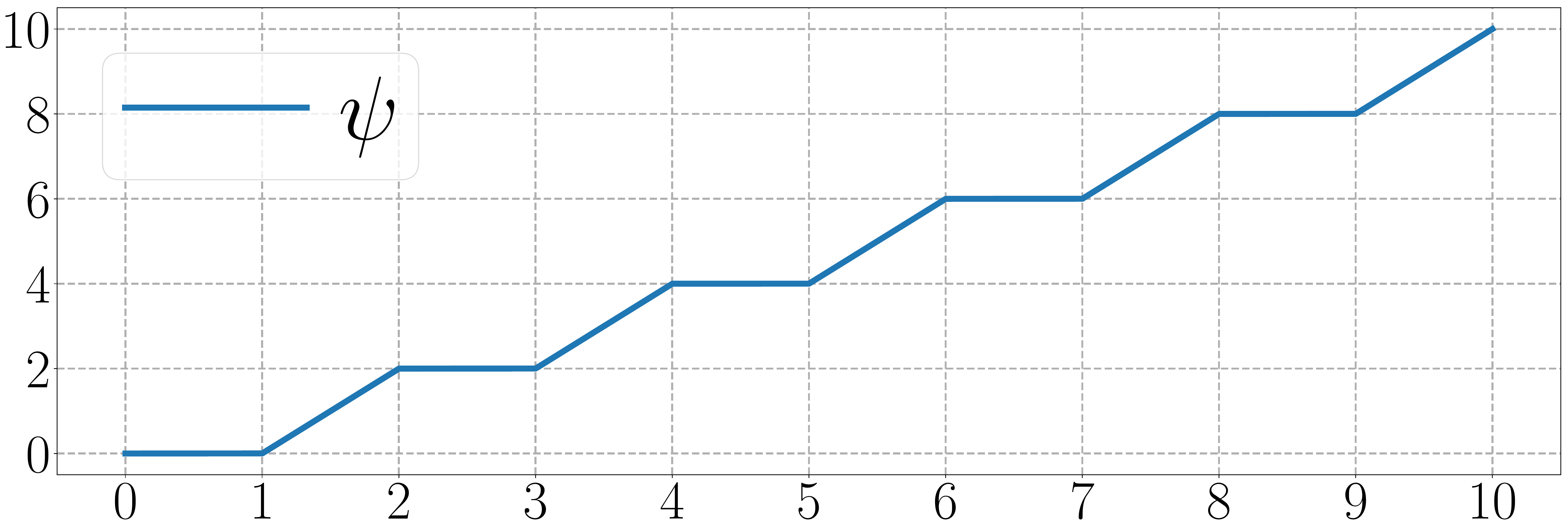}
	\caption{An illustration of $\psi$ on $[0,10]$.}
	\label{fig:psi}
\end{figure}

It is easy to verify that 
\begin{equation*}
	\psi(y)=2k-2\quad \tn{ for any $y\in [2k-2,2k-1]$ and each $k\in \{1,2,\cdots,K\}$.}
\end{equation*}
 It follows that
\begin{equation}\label{eq:psi:return:k}
	\psi(2Kx)/2+1=k\quad \tn{for any $x\in [\tfrac{2k-2}{2K},\tfrac{2k-1}{2K}]=\calI_k$ and each $k\in \{1,2,\cdots,K\}$.}
\end{equation}

Recall that $x_k$ is the right endpoint of $\calI_k$ for $k=1,2,\cdots,K$.
Set $M=\|f\|_{L^\infty([0,1])}+1$ and define
\begin{equation*}
	 \xi_k \coloneqq \tfrac{f(x_k)+M}{2M}\in [0,1]\quad \tn{ for $k=1,2,\cdots,K.$} 
\end{equation*}
 Then $[\xi_1,\xi_2,\cdots,\xi_K]^T$ is in $[0,1]^K$.
By Proposition~\ref{prop:dense}, there exists $w_0\in\R$ such that
\begin{equation*}
	\big|\sigma_1(\tfrac{w_0}{\pi+k})-\xi_k\big| <\varepsilon/(4M)\quad \tn{for $k=1,2,\cdots,K$.}
\end{equation*}

Let $m_0$ be an integer larger than $|w_0|$, e.g.,   set $m_0=\lfloor|w_0|\rfloor+1 $. It is easy to verify that
\begin{equation*}
 \tfrac{w_0}{\pi+k}+2m_0\ge  0\quad \tn{ for any $x\in [0,1]$.}
\end{equation*}
Since $\sigma(x)=\sigma_1(x)$ for any $x\ge 0$ and $\sigma_1$ is periodic with period $2$, we have
\begin{equation*}
	\big|\sigma(\tfrac{w_0}{\pi+k}+2m_0)-\xi_k\big| =\big|\sigma_1(\tfrac{w_0}{\pi+k}+2m_0)-\xi_k\big|= \big|\sigma_1(\tfrac{w_0}{\pi+k})-\xi_k\big|<\varepsilon/(4M),
\end{equation*}
for $k=1,2,\cdots,K$.
It follows that 
\begin{equation}\label{eq:f:error:dense}
\begin{split}
		 \Big|2M\sigma(\tfrac{w_0}{\pi+k}+2m_0)-M-f(x_k)\Big|&=\Big|2M\sigma(\tfrac{w_0}{\pi+k}+2m_0)-M-(2M\xi_k-M)\Big|\\
		 &=	2M\big|\sigma(\tfrac{w_0}{\pi+k}+2m_0)-\xi_k\big|
		  <2M\cdot\tfrac{\varepsilon}{4M}=\varepsilon/2,
\end{split}
\end{equation}
for $k=1,2,\cdots,K$.

The desired $\phi$ is defined as 
\begin{equation*}
	\phi(x)\coloneqq 2M\sigma\big(\tfrac{w_0}{\pi+\psi(2Kx)/2+1}+2m_0\big)-M
	\quad \tn{for any $x\in [0,1]$}.
\end{equation*}
Recall that $m_0\ge |w_0|$ and $\psi(x)\ge 0$ for any $x\ge 0$, which implies 
\begin{equation*}
    \tfrac{w_0}{\pi+\psi(2Kx)/2+1}+2m_0\ge  0\quad \tn{ for any $x\in [0,1]$.}
\end{equation*}
It follows that $\|\phi\|_{L^\infty([0,1])}\le M=\|f\|_{L^\infty([0,1])}+1$ since $0\le \sigma(y)\le 1$ for any $y\ge  0$.

For any $x\in \calI_k$ and each $k\in\{1,2,\cdots,K\}$, by Equation~\eqref{eq:psi:return:k}, we have $\psi(2Kx)/2+1=k$, which implies
\begin{equation*}
	\phi(x)=2M\sigma\big(\tfrac{w_0}{\pi+\psi(2Kx)/2+1}+2m_0\big)-M=2M\sigma\big(\tfrac{w_0}{\pi+k}+2m_0\big)-M.
\end{equation*}
Clearly, 
for any $x\in \calI_k$ and each $k\in\{1,2,\cdots,K\}$,
 we have $|x_k-x|<1/K$. Then,   by Equation~\eqref{eq:f:error:1/K}, we get
 \begin{equation*}
    |f(x_k)-f(x)|<\varepsilon/2\quad \tn{for any $x\in \calI_k$ and each $k\in\{1,2,\cdots,K\}$.}
 \end{equation*}

Therefore, by Equation~\eqref{eq:f:error:dense}, we have
\begin{equation*}
	\begin{split}
		|\phi(x)-f(x)|&= \Big|2M\sigma\big(\tfrac{w_0}{\pi+k}+2m_0\big)-M-f(x)\Big|\\
		&\le \Big|2M\sigma\big(\tfrac{w_0}{\pi+k}+2m_0\big)-M-f(x_k)\Big|+\big|f(x_k)-f(x)\big|
		< \varepsilon/2+\varepsilon/2=\varepsilon
	\end{split}
\end{equation*}
for any $x\in \calI_k$ and each $k\in\{1,2,\cdots,K\}$.
It follows that
\begin{equation*}
	\begin{split}
		|\phi(x)-f(x)|<\varepsilon\quad \tn{for any $x\in \bigcup_{j=1}^K \calI_j=\bigcup_{j=1}^{K}\big[\tfrac{2j-2}{2K},\tfrac{2j-1}{2K}\big]=\bigcup_{k=0}^{K-1}\big[\tfrac{2k}{2K},\tfrac{2k+1}{2K}\big] $}.
	\end{split}
\end{equation*}

It remains to show that $\phi$ can be generated by an EUAF network with the desired size. Observe that 
\begin{equation*}
	\sigma(y)+1=\frac{y}{|y|+1}+1=\frac{y}{-y+1}+1=\frac{1}{-y+1}\quad \tn{for any $y\le 0$.}
\end{equation*}
By setting $y=-\pi-\psi(2Kx)/2\le 0$ for any $x\in[0,1]$, we have 
\begin{equation*}
\begin{split}
    	\frac{1}{\pi+\psi(2Kx)/2+1}= \frac{1}{-y+1}=\sigma(y)+1
    	&=\sigma\big(-\pi-\psi(2Kx)/2
	\big)+1\\
	&=\sigma\Big(-\pi-\big(2Kx-\sigma(2Kx)\big)/2
	\Big)+1\\
	&=\sigma\big(-\pi-Kx+\sigma(2Kx)/2
	\big)+1,
\end{split}
\end{equation*}
where the second-to-last equality comes from $\psi(z)=z-\sigma(z)$ for any $z\in\R$.
Therefore, we get
\begin{equation}\label{eq:phi:d=1:half}
	\begin{split}
		\phi(x) &=  2M\sigma\big(\tfrac{w_0}{\pi+\psi(2Kx)/2+1}+2m_0\big)-M\\
		&=  2M\sigma\Big(w_0\sigma\big(-\pi-Kx+\sigma(2Kx)/2
		\big)+w_0+2m_0\Big)-M.
	\end{split}
\end{equation}

\begin{figure}[htbp!]        
	\centering
	\includegraphics[width=0.8945\textwidth]{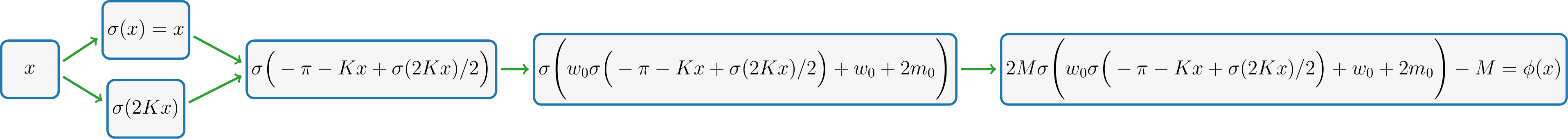}
	\caption{An illustration of the target EUAF network realizing $\phi(x)$ for $x\in [0,1]$ based on Equation~\eqref{eq:phi:d=1:half}. }
	\label{fig:phi:d=1:half}
\end{figure}

Thus, the desired EUAF network realizing $\phi$ is shown in Figure~\ref{fig:phi:d=1:half}.
Clearly, the network in Figure~\ref{fig:phi:d=1:half} has width $2$  and
depth $3$ as desired. It is easy to verify that the network architecture corresponding $\phi$ is independent of the target function $f$ and the desired error $\varepsilon$. That is, we can fix the architecture and only adjust parameters to achieve the desired approximation error. So we finish the proof.

\subsection{Proof of Theorem~\ref{thm:main:d=1:small:region}}
\label{sec:proof:main:d=1:small:region}

The key idea of proving Theorem~\ref{thm:main:d=1:small:region} is to apply Theorem~\ref{thm:main:d=1:half} to several horizontally shifted variants of the target function. Then a good approximation can be expected via combinations and multiplications of these variants. 
Thus, we need to reproduce $f(x,y)=xy$ locally via an EUAF network as shown in the following lemma.

\begin{lemma}\label{lem:xy}
	For any $M>0$, there exists a function $\phi$ generated by an EUAF network with width $9$ and depth $2$ such that
	\begin{equation*}
		\phi(x,y)=xy\quad \tn{for any $x,y\in[-M,M]$.}
	\end{equation*}
\end{lemma}

The proof of this lemma is given in Section~\ref{sec:proof:lemma:xy}. Now let us first prove Theorem~\ref{thm:main:d=1:small:region} by assuming this lemma is true.
\begin{proof}[Proof of Theorem~\ref{thm:main:d=1:small:region}]
Set $\widetilde{\varepsilon}=\varepsilon/4$ and
extend $f$ from $[0,1]$ to $[-1,1]$ by defining $f(x)=f(0)$ for any $x\in [-1,0)$. Then $f$ is continuous on $[-1,1]$ and therefore uniformly continuous. Thus, there exists $K=K(f,\varepsilon)\in \N^+$ with $K\ge 10$ such that
\begin{equation*}
	|f(x_1)-f(x_2)|<\widetilde{\varepsilon}/2\quad \tn{for any $x_1,x_2\in[-1,1]$ with $|x_1-x_2|<1/K$.}
\end{equation*}
For $i=1,2,3,4$, define
\begin{equation*}
	f_i(x)\coloneqq f\big(x-\tfrac{i}{4K}\big)\quad\tn{for any $x\in[0,1]$.}
\end{equation*}

For each $i\in\{1,2,3,4\}$ and  any $x_1,x_2\in[0,1]$ with $|x_1-x_2|<1/K$, we have 
$x_1-\tfrac{i}{4K},x_2-\tfrac{i}{4K}\in [-1,1]$ and $\big| (x_1-\tfrac{i}{4K})- (x_2-\tfrac{i}{4K})\big|=|x_1-x_2|<1/K$, which implies
\[|f_i(x_1)-f_i(x_2)|=\big|f(x_1-\tfrac{i}{4K})-f(x_2-\tfrac{i}{4K})\big|<\widetilde{\varepsilon}/2.\]
That is, for $i=1,2,3,4$, we have
\begin{equation*}
	|f_i(x_1)-f_i(x_2)|<\widetilde{\varepsilon}/2\quad \tn{for any $x_1,x_2\in[0,1]$ with $|x_1-x_2|<1/K$,}
\end{equation*}
which means we can apply Theorem~\ref{thm:main:d=1:half} to $f_i\in C([0,1])$.
For each $i\in \{1,2,3,4\}$, by Theorem~\ref{thm:main:d=1:half}, there exists a function  $\phi_i$ generated by an EUAF network with width $2$ and depth $3$ such that 
\begin{equation*}
    \|\phi_i\|_{L^\infty([0,1])}\le \|f_i\|_{L^\infty([0,1])}+1\le \|f\|_{L^\infty([-1,1])}+1
\end{equation*}
and
\begin{equation*}
	\big|\phi_i(x)-f_i(x)\big|<\widetilde{\varepsilon}=\varepsilon/4\quad \tn{for any}\  x\in \bigcup_{k=0}^{K-1}\big[\tfrac{2k}{2K},\tfrac{2k+1}{2K}\big].
\end{equation*}

Define 
\begin{equation*}
	\psi(x)=\sigma\big(x+1-\sigma(x+1)\big)\quad\tn{for any $x\in\R$.}
\end{equation*}
See an illustration of $\psi$  on $[0,2K]$ for $K=5$ in Figure~\ref{fig:psi:0:2K}. 

\begin{figure}[htbp!]
	\centering
	\includegraphics[width=0.6398\textwidth]{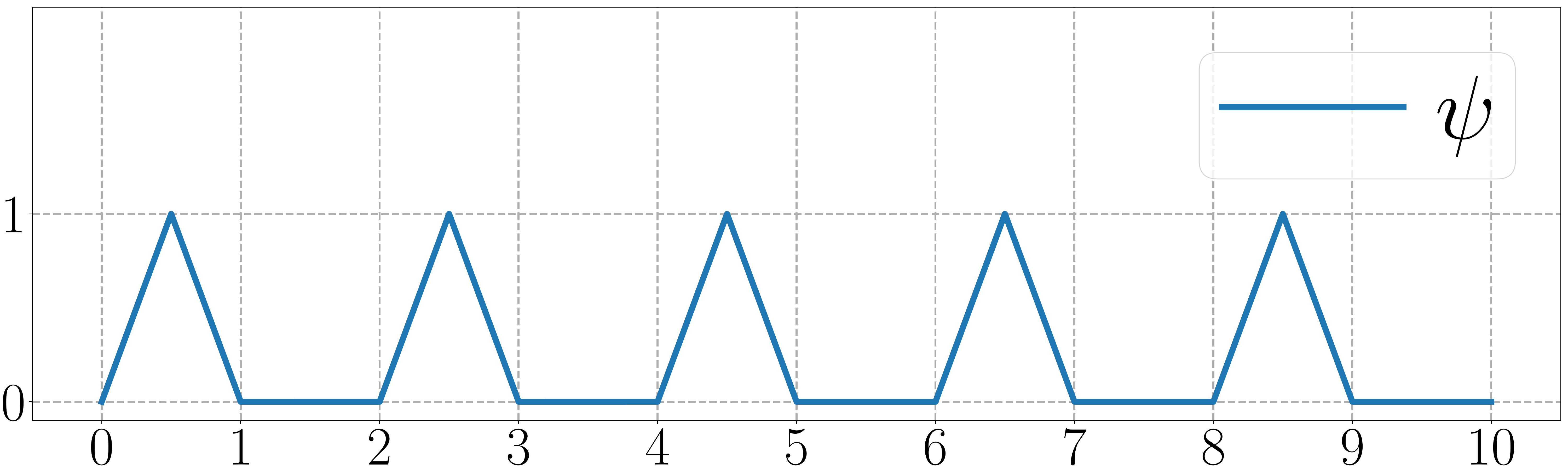}
	\caption{An illustration of $\psi$ on $[0,2K]$ for $K=5$.}
	\label{fig:psi:0:2K}
\end{figure}

Clearly, $0\le \psi(2Kx)\le 1$ for any $x\in [0,1]$, from which we deduce
\begin{equation*}
	\Big|\big(\phi_i(x)-f_i(x)\big)\psi(2Kx)\Big|\le\big|\phi_i(x)-f_i(x)\big|<\varepsilon/4\quad \tn{for any}\  x\in \bigcup_{k=0}^{K-1}\big[\tfrac{2k}{2K},\tfrac{2k+1}{2K}\big].
\end{equation*}
Observe that $\psi(y)=0$ for $y\in \bigcup_{k=0}^{K-1} [{2k+1},{2k+2}]$, which implies 
\begin{equation*}
	\psi(2Kx)=0\quad \tn{for any $x\in \bigcup_{k=0}^{K-1} [\tfrac{2k+1}{2K},\tfrac{2k+2}{2K}]\supseteq [0,1]\Big\backslash \bigcup_{k=0}^{K-1}\big[\tfrac{2k}{2K},\tfrac{2k+1}{2K}\big]$}.
\end{equation*}
It follows that
\begin{equation}\label{eq:phi:f:error:psi}
	\Big|\big(\phi_i(x)-f_i(x)\big)\psi(2Kx)\Big|<\varepsilon/4\quad \tn{for any $x\in[0,1]$ and $i=1,2,3,4$}.
\end{equation}

For each $i\in\{1,2,3,4\}$ and any $z\in[0,\tfrac{9}{10}]
\subseteq [0,1-\tfrac{1}{K}]
\subseteq [0,1-\tfrac{i}{4K}]$, we have 
\begin{equation*}
    y_i=z+\tfrac{i}{4K}\in [\tfrac{i}{4K},1]\subseteq [0,1].
\end{equation*}
Therefore, by bringing $x=y_i\in [0,1]$ into Equation~\eqref{eq:phi:f:error:psi}, we have
\begin{equation}\label{eq:error:phi:component}
	\begin{split}
		\varepsilon/4
		&> \Big|\big(\phi_i(y_i)-f_i(y_i)\big)\psi(2Ky_i)\Big|=\Big|\phi_i(y_i)\psi(2Ky_i)-f_i(y_i)\psi(2Ky_i)\Big|\\
		&=\Big|\phi_i(z+\tfrac{i}{4K})\psi\big(2K(z+\tfrac{i}{4K})\big)-f_i(z+\tfrac{i}{4K})\psi\big(2K(z+\tfrac{i}{4K})\big)\Big|\\
		&=\Big|\phi_i(z+\tfrac{i}{4K})\psi\big(2Kz+\tfrac{i}{2}\big)-f(z)\psi\big(2Kz+\tfrac{i}{2}\big)\Big|
	\end{split}
\end{equation}
for any $z\in [0,\tfrac{9}{10}]$,
where the last equality comes from the fact that $f_i(x)=f(x-\tfrac{i}{4K})$ for any $x\in[0,1]\supseteq [\tfrac{i}{4K},1]$. The desired $\phi$ is defined as 
\begin{equation*}
	\phi(x)\coloneqq\sum_{i=1}^4\phi_i(x+\tfrac{i}{4K})\psi\big(2Kx+\tfrac{i}{2}\big)\quad \tn{for any $x\in[0,\tfrac{9}{10}]$.}
\end{equation*}

It is easy to verify that
$\sum_{i=1}^4 \psi\big(x+\tfrac{i}{2}\big)=1$ for any $x\ge 0$ based on the definition of $\psi$. See Figure~\ref{fig:psi1234} for illustrations.
It follows that $\sum_{i=1}^4 \psi\big(2Kz+\tfrac{i}{2}\big)=1$ for any $z\in [0,\tfrac{9}{10}]$. 

\begin{figure}[htbp!]        
	\centering
	\begin{minipage}{0.985\textwidth}
		\centering
		\begin{subfigure}[b]{0.48\textwidth}
			\centering            
			\includegraphics[width=0.985\textwidth]{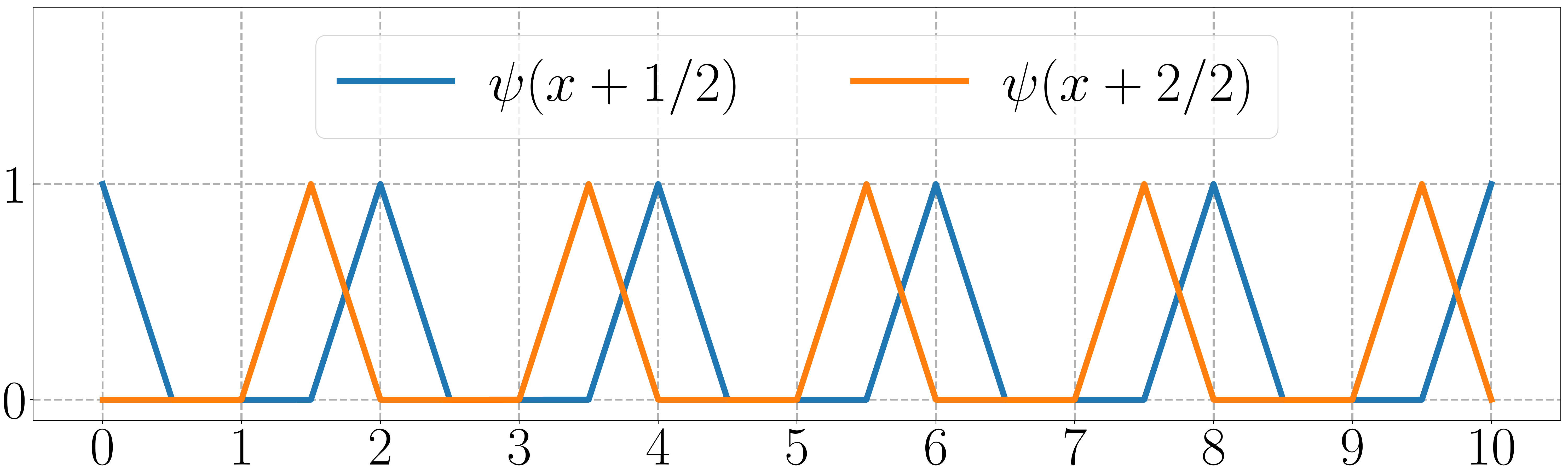}
		\end{subfigure}
		\begin{subfigure}[b]{0.48\textwidth}
			\centering            
			\includegraphics[width=0.985\textwidth]{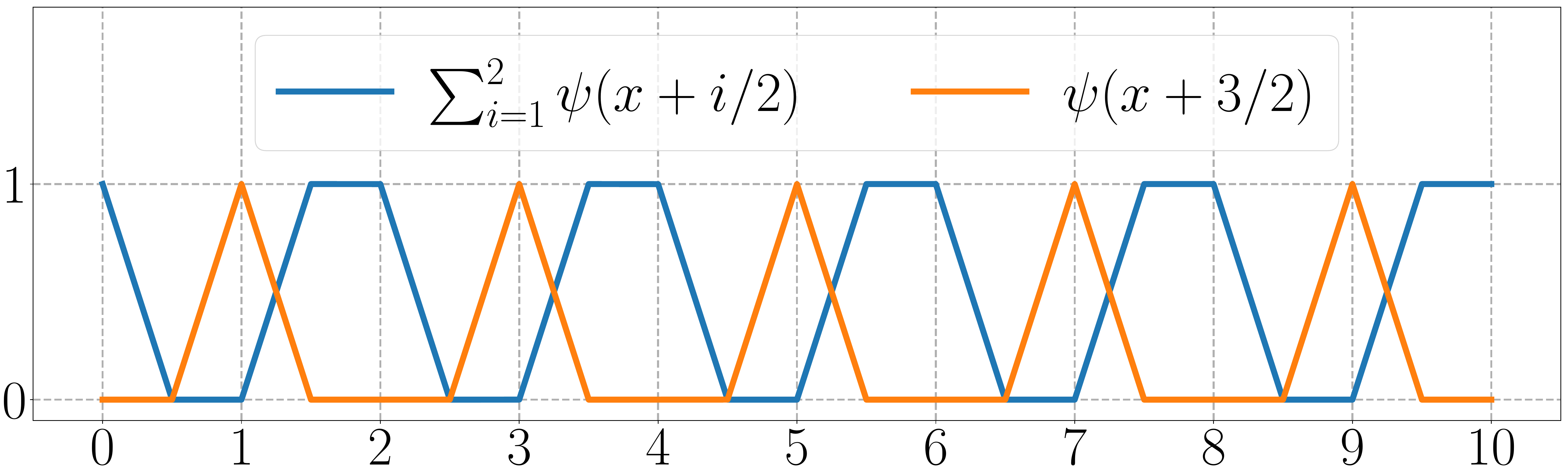}
		\end{subfigure}
		
		\vspace*{4pt}
			\begin{subfigure}[b]{0.48\textwidth}
		\centering            
		\includegraphics[width=0.985\textwidth]{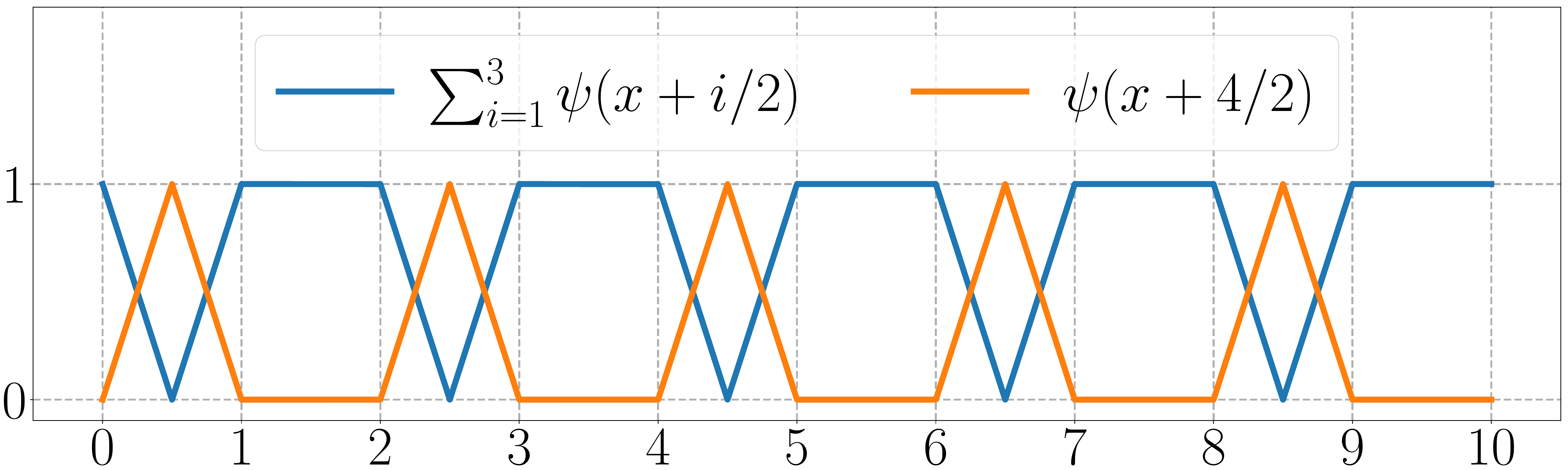}
	\end{subfigure}
			\begin{subfigure}[b]{0.485\textwidth}
	\centering            
	\includegraphics[width=0.985\textwidth]{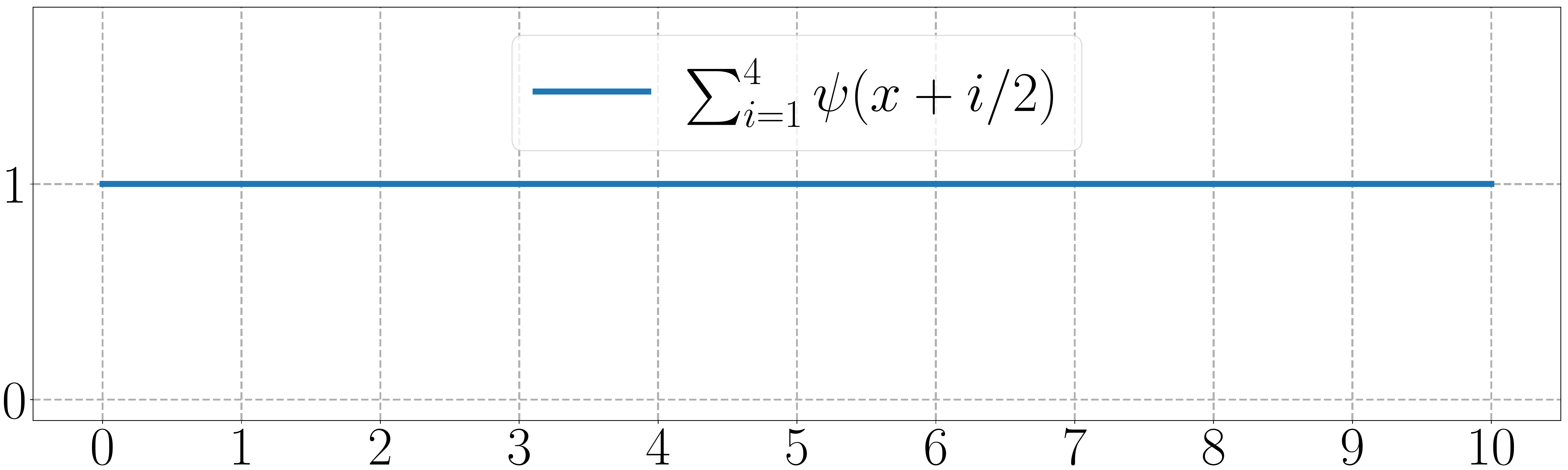}
\end{subfigure}
	\end{minipage}
	\caption{Illustrations of $\sum_{i=1}^{4}\psi(x+i/2)=1$ for any  $x\in [0,10]$.}
	\label{fig:psi1234}
\end{figure}

Hence, for any $z\in [0,\tfrac{9}{10}]$, by Equation~\eqref{eq:error:phi:component},  we have
\begin{equation*}
	\begin{split}
		\big|\phi(z)-f(z)\big|
		&=\Big|\sum_{i=1}^4\phi_i(z+\tfrac{i}{4K})\psi\big(2Kz+\tfrac{i}{2}\big)-f(z)\sum_{i=1}^4\psi\big(2Kz+\tfrac{i}{2}\big)\Big|\\
		&\le \sum_{i=1}^4\Big|\phi_i(z+\tfrac{i}{4K})\psi\big(2Kz+\tfrac{i}{2}\big)-f(z)\psi\big(2Kz+\tfrac{i}{2}\big)\Big|< 4\cdot\frac{\varepsilon}{4}=\varepsilon.
	\end{split}
\end{equation*}
That is, $|\phi(x)-f(x)|<\varepsilon$ for any $x\in[0,\tfrac{9}{10}]$ as desired. 
It remains to show that $\phi$, limited on $[0,\tfrac{9}{10}]$, can be generated by an EUAF network with the desired size.

Note that $x+1=(2K+1)\sigma(\tfrac{x+1}{2K+1})$ for any $x\in [0,2K]$, which implies
\begin{equation*}
	\psi(x)=\sigma\big(x+1-\sigma(x+1)\big)=\sigma\Big((2K+1)\sigma(\tfrac{x+1}{2K+1})-\sigma(x+1)\Big).
\end{equation*}
This means  
$\psi$, limited on $[0,2K]$, can be generated by an EUAF network with width $2$ and depth $2$. Since $0\le 2Kx+\tfrac{i}{2}\le 2K\tfrac{9}{10}+2=2K(\tfrac{9}{10}+\tfrac{1}{K})\le 2K$ for any $x\in [0,\tfrac{9}{10}]$, $\psi(2K\cdot+\tfrac{i}{2})$, limited on $[0,\tfrac{9}{10}]$, can also be generated by an EUAF network with width $2$ and depth $2$. 

Note that $\phi_i$, limited on $[0,1]$,
can also be generated by an EUAF network with width $2$ and depth $3$. Clearly, $x+\tfrac{i}{4K}\in [0,1]$ for any $x\in [0,\tfrac{9}{10}]$, and, therefore, $\phi_i(\cdot +\tfrac{i}{4K})$, limited on $[0,\tfrac{9}{10}]$, can also be generated by an EUAF network with width $2$ and depth $3$.

Recall that
$\|\phi_i\|_{L^\infty([0,1])}\le  \|f\|_{L^\infty([-1,1])}+1\eqqcolon M$. Thus, $|\phi_i(x +\tfrac{i}{4K})|\le M$ and $|\psi(2Kx+\tfrac{i}{2})|\le1\le  M$  for any $x\in [0,\tfrac{9}{10}]$ and $i=1,2,3,4$.
By Lemma~\ref{lem:xy}, there exists a function $\Gamma$ generated by an EUAF network with width $9$ and depth $2$ such that
\begin{equation*}
\Gamma(x,y)=xy\quad \tn{for any $x,y\in [-M,M]$.}
\end{equation*}
It follows that
\begin{equation*}
	\Gamma\Big(\phi_i(x+\tfrac{i}{4K}),\psi\big(2Kx+\tfrac{i}{2}\big)\Big)=\phi_i(x+\tfrac{i}{4K})\psi\big(2Kx+\tfrac{i}{2}\big)\quad\tn{for $i=1,2,3,4$.}
\end{equation*}

\begin{figure}[htbp!]        
	\centering
	\includegraphics[width=0.7245\textwidth]{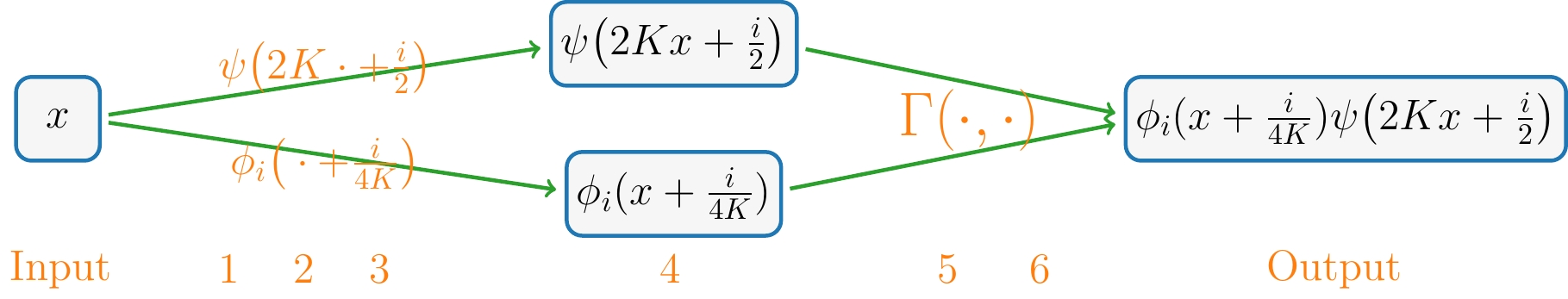}
	\caption{An illustration of the target EUAF network realizing each component of $\phi(x)$, $\phi_i(x+\tfrac{i}{4K})\psi\big(2Kx+\tfrac{i}{2}\big)$, for any $x\in [0,\tfrac{9}{10}]$ and each $i\in\{1,2,3,4\}$. The networks realizing $\phi_i(\cdot+\tfrac{i}{4K})$ and $\psi\big(2K\cdot+\tfrac{i}{2}\big)$ can be placed in parallel since we can manually add a hidden layers to $\psi$ since $\sigma\circ \psi\big(2Kx+\tfrac{i}{2}\big)=\psi\big(2Kx+\tfrac{i}{2}\big)$ for any $x\in[0,\tfrac{9}{10}]$. }
	\label{fig:phi:i:component}
\end{figure}

Therefore, each component of $\phi(x)$, $\phi_i(x+\tfrac{i}{4K})\psi\big(2Kx+\tfrac{i}{2}\big)$ for each $i\in\{1,2,3,4\}$, can be generated by the network in Figure~\ref{fig:phi:i:component} for any $x\in [0,\tfrac{9}{10}]$. Clearly, such a network has width $9$ and depth $6$. Since the $4$-th hidden layer of the network  in Figure~\ref{fig:phi:i:component} uses the identity map as an activation function for each neuron in this hidden layer, we can reduce the depth by $1$ via composing two adjacent affine linear maps to generate a new one. Thus, the network in Figure~\ref{fig:phi:i:component} can be interpreted as an EUAF network with width $9$ and depth $5$.

Note that $\phi$ is the sum of its four components, namely, 
\[\phi(x)=\sum_{i=1}^4\phi_i(x+\tfrac{i}{4K})\psi\big(2Kx+\tfrac{i}{2}\big)\quad \tn{ for any $x\in [0,\tfrac{9}{10}]$.}\] 
Therefore, $\phi$, limited on  $[0,\tfrac{9}{10}]$, 
can be
generated  by an EUAF network with width $9\times 4=36$ and depth $5$ as desired.
It is easy to verify that the designed network architecture is independent of the target function $f$ and the desired error $\varepsilon$. That is, we can fix the architecture and only adjust parameters to achieve an arbitrarily small  approximation error. So we finish the proof.
\end{proof}

\subsection{Proof of Lemma~\ref{lem:xy}}\label{sec:proof:lemma:xy}
The key idea of proving Lemma~\ref{lem:xy} is  the polarization identity  $2xy= (x+y)^2-x^2-y^2$. Thus, we need to reproduce $x^2$ locally by an EUAF network as shown in the following lemma.
\begin{lemma}\label{lem:xx}
	There exists a function $\phi$ generated by an EUAF network with width $3$ and depth $2$ such that
	\begin{equation*}
		\phi(x)=x^2\quad \tn{for any $x\in[-1,1]$.}
	\end{equation*}
\end{lemma}
\begin{proof}
	Observe that
	\begin{equation*}
		\sigma(y)+1=\frac{y}{|y|+1}+1=\frac{y}{-y+1}+1=\frac{1}{-y+1}\quad 
	\tn{for any $y\le 0$.}
	\end{equation*}
	For any $x\in[-1,1]$, we have $-x-1\le 0$ and $-x-2\le 0$, which implies
	\begin{equation*}
		\begin{split}
			\sigma(-x-1)-\sigma(-x-2)&=\Big(\sigma(-x-1)+1\Big)-\Big(\sigma(-x-2)+1\Big)\\
			&=\frac{1}{-(-x-1)+1}-\frac{1}{-(-x-2)+1}\\
			&=\frac{1}{x+2}-\frac{1}{x+3}=\frac{1}{(x+2)(x+3)}.
		\end{split}
	\end{equation*}
	It follows from  $1-\tfrac{12}{(x+2)(x+3)}\le 0$ for any $x\in[-1,1]$ that 
	\begin{equation*}
		\begin{split}
			\sigma\Big(1-\frac{12}{(x+2)(x+3)}\Big)+1
			=\frac{1}{-\big(1-\tfrac{12}{(x+2)(x+3)}\big)+1}
			=\frac{x^2+5x+6}{12},
		\end{split}
	\end{equation*}
	implying
	\begin{equation*}
		\begin{split}
			x^2&=12 \sigma\Big(1-\frac{12}{(x+2)(x+3)}\Big)+12-(5x+6)\\
			&=12 \sigma\Big(1- 12 \big(\sigma(-x-1)- \sigma(-x-2)\big)\Big)+11\frac{6-5x}{11}\\
			& =12 \sigma\Big(1-12\sigma(-x-1)+12\sigma(-x-2)\Big)+11\sigma\Big(\frac{6-5x}{11}\Big)\eqqcolon  \phi(x),
		\end{split}
	\end{equation*}
where the equality $\tfrac{6-5x}{11}=\sigma\big(\tfrac{6-5x}{11}\big)$ comes from two facts:  $\tfrac{6-5x}{11}\in [0,1]$ since $x\in[-1,1]$ and $\sigma(z)=z$ for any $z\in [0,1]$.

\begin{figure}[htbp!]        
	\centering
	\includegraphics[width=0.8945\textwidth]{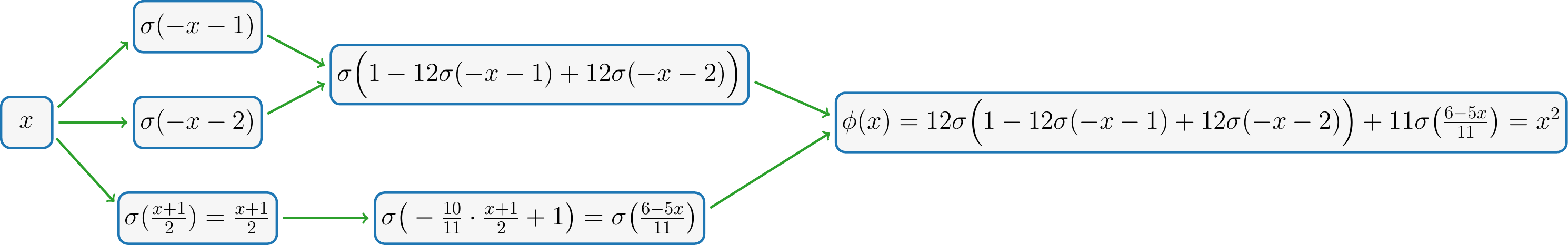}
	\caption{An illustration of the target EUAF network  realizing $\phi(x)=x^2$ for $x\in [-1,1]$. }
	\label{fig:xx}
\end{figure}

	Then,  $x^2$ can be generated by the network shown in Figure~\ref{fig:xx} for any $x\in [-1,1]$. The target network has width $3$ and depth $2$. So we finish the proof.
\end{proof}

With Lemma~\ref{lem:xx} at hand, we are ready to prove Lemma~\ref{lem:xy}.
\begin{proof}[Proof of Lemma~\ref{lem:xy}]
	By Lemma~\ref{lem:xx}, there exists a function $\tildephi$ generated by an EUAF network such that $\tildephi(t)=t^2$ for any $t\in [-1,1]$.
	Then, for any $x,y\in [-M,M]$, we have 
	\begin{equation*}
		\begin{split}
			xy&=2M^2\Big( \big(\tfrac{x+y}{2M}\big)^2- \big(\tfrac{x}{2M}\big)^2 -\big(\tfrac{y}{2M}\big)^2\Big)\\
			&=2M^2\Big( \tildephi\big(\tfrac{x+y}{2M}\big)- \tildephi\big(\tfrac{x}{2M}\big) -\tildephi\big(\tfrac{y}{2M}\big)\Big)\eqqcolon \phi(x,y).
		\end{split}
	\end{equation*}

\begin{figure}[htbp!]        
	\centering
	\includegraphics[width=0.7145\textwidth]{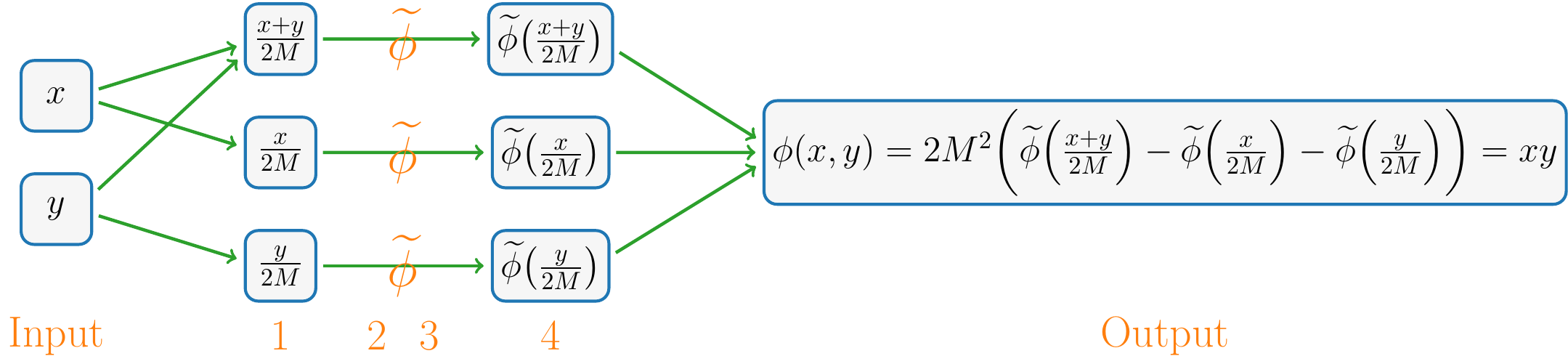}
	\caption{An illustration of the target network  realizing $\phi(x)=xy$ for $x,y\in [-M,M]$. 
		}
	\label{fig:xy}
	\end{figure}

	The target network realizing $\phi$ with width $9$ and depth $4$ is shown in Figure~\ref{fig:xy}. Note that we can reduce the depth by one if the activation function of each neuron in a hidden layer is the identity map. In fact, we can eliminate this hidden layer by composing two adjacent affine linear maps to generate a new one.
	The $1$-st and $4$-th hidden layers of the network in Figure~\ref{fig:xy} use the identity map as an activation function for each neuron. Thus, the network in Figure~\ref{fig:xy} can be interpreted as an EUAF network with width $9$ and depth $2$. So we finish the proof.	
\end{proof}

\section{Proof of Proposition~\ref{prop:dense}}
\label{sec:proof:prop:dense}

We will prove Proposition~\ref{prop:dense} in this section. 
The proof includes two main steps. First, we show how to simply generate a set of rationally independent numbers in Lemma~\ref{lem:r:i:numbers} below. Next, we prove that the target point set via a winding of the generated rationally independent numbers is dense in a hypercube. Such a proof relies on the fact that
an irrational winding on the torus is dense (e.g., see Lemma~$2$ of \cite{yarotsky:2021:02}) as shown in Lemma~\ref{lem:dense:periodic} below.
\begin{lemma}\label{lem:r:i:numbers}
	Given any $K\in\N^+$, any transcendental number $\alpha\in\R\backslash \A$, and any pairwise distinct rational numbers $r_1,r_2,\cdots,r_K\in\Q$,   the set of numbers
	\begin{equation*}
		\big\{\tfrac{1}{  \alpha+r_k}:k=1,2,\cdots,K\big\}
	\end{equation*}
	are rationally independent.
\end{lemma}

\begin{lemma}\label{lem:dense:periodic}
	Given any rationally independent numbers $a_1,a_2,\cdots,a_K$ for any $K\in\N^+$ and an arbitrary periodic function $g:\R\to\R$ with period $T$, i.e., $g(x+T)=g(x)$ for any $x\in\R$, assume there exist $x_1,x_2\in \R$ with $0<x_2-x_1<T$ such that $g$ is continuous on $[x_1,x_2]$. Then the following set 
	\begin{equation*}
		\Big\{\big[g(wa_1),\, g(wa_2),\, \cdots,\,g(wa_K)\big]^T: w\in\R
		\Big\}
	\end{equation*}
is dense in $[M_1,M_2]^K$, where $M_1=\min\limits_{x\in [x_1,x_2]} g(x)$ and $M_2=\max\limits_{x\in [x_1,x_2]} g(x)$.
\end{lemma}

The proofs of these two lemmas can be found in Sections~\ref{sec:proof:lem:r:i:numbers} and \ref{sec:proof:lem:dense:periodic}, respectively.
With these two lemmas at hand, the proof of Proposition~\ref{prop:dense} is straightforward. In fact, we can prove 
a more general result in Proposition~\ref{prop:dense:general} below, which implies Proposition~\ref{prop:dense} immediately.
\begin{proposition}\label{prop:dense:general}
	Given  an arbitrary periodic function $g:\R\to\R$ with period $T$, i.e., $g(x+T)=g(x)$ for any $x\in\R$, assume there exist $x_1,x_2\in \R$ with $0<x_2-x_1<T$ such that $g$ is continuous on $[x_1,x_2]$. Then, for any $K\in\N^+$, any transcendental number $\alpha\in \R\backslash\A$, and any pairwise distinct rational numbers $r_1,r_2,\cdots,r_K\in\Q$,  the following set 
	\begin{equation*}
		\Big\{\big[g(\tfrac{w}{\alpha+r_1}),\   g(\tfrac{w}{\alpha+r_2}),\  \cdots,\    g(\tfrac{w}{\alpha+r_K})\big]^T: w\in\R
		\Big\}
	\end{equation*}
	is dense in $[M_1,M_2]^K$, where $M_1=\min\limits_{x\in [x_1,x_2]} g(x)$ and $M_2=\max\limits_{x\in [x_1,x_2]} g(x)$. In the case of $M_1<M_2$, the following set 
		\begin{equation*}
		\bigg\{\Big[u\cdot   g(\tfrac{w}{\alpha+r_1})+v,\   u\cdot  g(\tfrac{w}{\alpha+r_2})+v,\   \cdots,\  u\cdot   g(\tfrac{w}{\alpha+r_K})+v\Big]^T: u,v,w\in\R
		\bigg\}
	\end{equation*}
is dense in $\R^K$.
\end{proposition}
Clearly, Proposition~\ref{prop:dense} is a special case of Proposition~\ref{prop:dense:general} with $g=\sigma_1$, $\alpha=\pi$, $r_k=k$ for $k=1,2,\cdots,K$. The transcendence of $\pi$ is well known (e.g., see the Lindemann-Weierstrass Theorem).  By setting $x_1=0$ and $x_2=1$, we have $[M_1,M_2]=[0,1]$ and $\sigma_1$ is continuous on $[0,1]$, which means that the following set
	\begin{equation*}
	\Big\{\big[\sigma_1(\tfrac{w}{\pi+1}),\, \sigma_1(\tfrac{w}{\pi+2}),\, \cdots,\,  \sigma_1(\tfrac{w}{\alpha+K})\big]^T: w\in\R
	\Big\}
\end{equation*}
is dense in $[0,1]^K$ as desired. 

Finally, let us prove Proposition~\ref{prop:dense:general} by assuming Lemmas~\ref{lem:r:i:numbers} and \ref{lem:dense:periodic} are true.
\begin{proof}[Proof of Proposition~\ref{prop:dense:general}]
	By Lemma~\ref{lem:r:i:numbers}, the set of numbers
	\begin{equation*}
		\Big\{\tfrac{1}{  \alpha+r_k}:k=1,2,\cdots,K\Big\}
	\end{equation*}
	are rationally independent. Denote $a_k=\tfrac{1}{   \alpha+r_k}$ for $k=1,2,\cdots,K$. Then, by Lemma~\ref{lem:dense:periodic}, 
	\begin{equation*}
		\begin{split}
				&\  \quad\Big\{\big[g(wa_1),\, g(wa_2),\, \cdots,\,g(wa_K)\big]^T: w\in\R
			\Big\}\\
			&=\Big\{\big[g(\tfrac{w}{\alpha+r_1}),\, g(\tfrac{w}{\alpha+r_2}),\, \cdots,\,g(\tfrac{w}{\alpha+r_K})\big]^T: w\in\R
			\Big\}
		\end{split}
	\end{equation*}
	is dense in $[M_1,M_2]^K$. 
	
	Next, let us consider the case $M_1<M_2$ for the latter result. For any $\varepsilon>0$ and any $\bmx\in\R^K$, by setting $J=\|\bmx\|_{\infty}+1>0$, we have $\tfrac{\bmx+J}{2J}\in[0,1]^K$, and hence
	\begin{equation*}
		\bmy\coloneqq\tfrac{\bmx+J}{2J}(M_2-M_1)+M_1\in [M_1,M_2]^K.
	\end{equation*}
	By the former result, there exists $w_0\in\R$ such that
	\begin{equation*}
		\Big\| \bmy- \big[g(\tfrac{w_0}{\alpha+r_1}),\, g(\tfrac{w_0}{\alpha+r_2}),\, \cdots,\,g(\tfrac{w_0}{\alpha+r_K})\big]^T\Big\|_{\infty}<\tfrac{M_2-M_1}{2J}\varepsilon
	\end{equation*}
It follows from $\bmy=\tfrac{\bmx+J}{2J}(M_2-M_1)+M_1$ that 
\begin{equation*}
    \bmx=\tfrac{2J}{M_2-M_1}\bmy+\tfrac{J(M_1+M_2)}{M_1-M_2}\eqqcolon u_0\bmy+v_0,
\end{equation*}
where $u_0=\tfrac{2J}{M_2-M_1}$ and $v_0=\tfrac{J(M_1+M_2)}{M_1-M_2}$. Therefore, 
\begin{equation*}
	\begin{split}
		&\quad\  \bigg\| \bmx- \Big[u_0g(\tfrac{w_0}{\alpha+r_1})+v_0,\, u_0g(\tfrac{w_0}{\alpha+r_2})+v_0,\, \cdots,\, u_0g(\tfrac{w_0}{\alpha+r_K})+v_0\Big]^T\bigg\|_{\infty}\\
		&=\bigg\| u_0\bmy+v_0- \Big[u_0g(\tfrac{w_0}{\alpha+r_1})+v_0,\, u_0g(\tfrac{w_0}{\alpha+r_2})+v_0,\, \cdots,\, u_0g(\tfrac{w_0}{\alpha+r_K})+v_0\Big]^T\bigg\|_{\infty}\\
		&<u_0  \tfrac{M_2-M_1}{2J}\varepsilon=\tfrac{2J}{M_2-M_1}\tfrac{M_2-M_1}{2J}\varepsilon =\varepsilon.
	\end{split}
\end{equation*}
Since $\varepsilon>0$ and $\bmx\in\R^K$ are arbitrary,
the following set 
\begin{equation*}
	\bigg\{\Big[u\cdot   g(\tfrac{w}{\alpha+r_1})+v,\  u\cdot  g(\tfrac{w}{\alpha+r_2})+v,\  \cdots,\ u\cdot   g(\tfrac{w}{\alpha+r_K})+v\Big]^T: u,v,w\in\R
	\bigg\}
\end{equation*}
is dense in $\R^K$. So we finish the proof.
\end{proof}

\subsection{Proof of Lemma~\ref{lem:r:i:numbers}}\label{sec:proof:lem:r:i:numbers}

Before proving Lemma~\ref{lem:r:i:numbers}, let us first briefly discuss related concepts.
Recall that a complex number  $\alpha$ is an algebraic number if and only if there exist
$\lambda_0,\lambda_1,\cdots,\lambda_J\in \Q$ with  $\sum_{j=0}^J \lambda_j\alpha^j=0$.  The set of all algebraic numbers is denoted by $\A$. We say a complex number is \textbf{transcendental} if it is not in $\A$. Almost all complex numbers are transcendental since the set $\A$ is {countable}. 
The best known transcendental numbers are $\pi$ (the ratio of a circle's circumference to its diameter) and $e$ (the natural logarithmic base). 

In order to prove Lemma~\ref{lem:r:i:numbers}, we need an auxiliary lemma below, characterizing some properties of 
coefficients of Lagrange basis polynomials. Recall that, for any given pairwise distinct numbers $x_1,x_2,\cdots,x_K\in\R$, the Lagrange basis polynomials are 
\begin{equation}\label{eq:L:poly}
	p_k(x)\coloneqq \prod_{
		\substack{   j\in\{1,2,\cdots,K\}\\ j\neq k}   }\frac{x-x_j}{x_k-x_j}\ =\  
	\frac{x-x_1}{x_k-x_1} \ \cdots\
	\frac{x-x_{k-1}}{x_k-x_{k-1}}  
	\frac{x-x_{k+1}}{x_k-x_{k+1}}\ \cdots\  
	\frac{x-x_K}{x_k-x_K}
\end{equation}
for $k=1,2,\cdots,K$. They are polynomials of degree $\le K-1$, which means we can represent each $p_k$ by
\begin{equation*}
	p_k(x)=\sum_{j=1}^{K}a_{k,j}x^{j-1}=a_{k,1}+a_{k,2}x+\cdots+a_{k,K}x^{K-1}
\end{equation*}
for $k=1,2,\cdots,K$ and any $x\in \R$.
Thus, the coefficients of these $K$ Lagrange basis polynomials $p_1,p_2,\cdots,p_K$ form a matrix 
\begin{equation}\label{eq:matrix:L:c}
	\bmA=(a_{i,j})=\left[\begin{matrix}
		a_{1,1} & a_{1,2} & \cdots & a_{1,K}\\
		a_{2,1} & a_{2,2} & \cdots & a_{2,K}\\
		\vdots & \vdots & \ddots & \vdots\\
		a_{K,1} & a_{K,2} & \cdots & a_{K,K}
	\end{matrix}\right]\in\R^{K\times K}. 
\end{equation}

The lemma below essentially characterizes the linear independence of Lagrange basis polynomials.
\begin{lemma}\label{lem:matrix:inv}
	With the same setting just above, the matrix $\bmA$ given in Equation~\eqref{eq:matrix:L:c} is invertible.
\end{lemma}
\begin{proof}
	For any $\bmy=[y_1,y_2,\cdots,y_K]\in \R^K$, by the definition of Lagrange basis polynomials $p_k(x)$ for $k=1,2,\cdots,K$ in Equation~\eqref{eq:L:poly}, $p(x)=\sum_{k=1}^K y_k p_k(x)$ is the target interpolation polynomial for sample points $(x_1,y_1),(x_2,y_2),\cdots,(x_K,y_K)$. That is,  for any $\ell\in \{1,2,\cdots,K\}$, we have
	\begin{equation*}
		\begin{split}
			y_\ell &=p(x_\ell)=\sum_{k=1}^K y_k p_k(x_\ell)=\sum_{k=1}^K y_k \sum_{j=1}^K a_{k,j}x_\ell^{j-1}\\
			&=[y_1,y_2,\cdots,y_K]\cdot
			\begin{bmatrix}
				a_{1,1} & a_{1,2} & \cdots & a_{1,K}\\
				a_{2,1} & a_{2,2} & \cdots & a_{2,K}\\
				\vdots & \vdots & \ddots & \vdots\\
				a_{K,1} & a_{K,2} & \cdots & a_{K,K}
			\end{bmatrix}\cdot
			\begin{bmatrix}
				x_{\ell}^0 \\
				x_{\ell}^1 \\
				\vdots \\
				x_{\ell}^{K-1}
			\end{bmatrix}
			=\bmy^T \bmA \begin{bmatrix}
				x_{\ell}^0 \\
				x_{\ell}^1 \\
				\vdots \\
				x_{\ell}^{K-1}
			\end{bmatrix}.
		\end{split}
	\end{equation*}
	It follows that
	\begin{equation*}
		\begin{split}
			\bmy^T=[y_1,y_2,\cdots,y_K]= 
			\bmy^T \bmA 
			\begin{bmatrix}
				x_{1}^0 & x_{2}^0 & \cdots & x_{K}^0\\
				x_{1}^1 & x_{2}^1 & \cdots & x_{K}^1\\
				\vdots & \vdots & \ddots & \vdots\\
				x_{1}^{K-1} & x_{2}^{K-1} & \cdots & x_{K}^{K-1}
			\end{bmatrix}.
		\end{split}
	\end{equation*}
	Since $\bmy\in\R^K$ is arbitrary, we have 
	\begin{equation*}
		\begin{split}
			\bmA \begin{bmatrix}
				x_{1}^0 & x_{2}^0 & \cdots & x_{K}^0\\
				x_{1}^1 & x_{2}^1 & \cdots & x_{K}^1\\
				\vdots & \vdots & \ddots & \vdots\\
				x_{1}^{K-1} & x_{2}^{K-1} & \cdots & x_{K}^{K-1}
			\end{bmatrix}=\bm{I}_K,
		\end{split}
	\end{equation*}
	where $\bm{I}_K\in\R^{K\times K}$ is the identity matrix. Recall that $x_1,x_2,\cdots,x_K$ are pairwise distinct, which implies the Vandermonde matrix 
	\begin{equation*}
		\begin{split}
			\begin{bmatrix}
				x_{1}^0 & x_{2}^0 & \cdots & x_{K}^0\\
				x_{1}^1 & x_{2}^1 & \cdots & x_{K}^1\\
				\vdots & \vdots & \ddots & \vdots\\
				x_{1}^{K-1} & x_{2}^{K-1} & \cdots & x_{K}^{K-1}
			\end{bmatrix}
		\end{split}
	\end{equation*}
	is invertible. Thus, $\bmA$ is also invertible. So we complete the proof.
\end{proof}

With Lemma~\ref{lem:matrix:inv} at hand, we are ready to prove Lemma~\ref{lem:r:i:numbers}.
\begin{proof}[Proof of Lemma~\ref{lem:r:i:numbers}]
	Let $x_k=-r_k\in\Q$ for $k=1,2,\cdots,K$ and define the Lagrange basis polynomials as
	\begin{equation*}
		\begin{split}
			p_k(x)\coloneqq \prod_{
				\substack{   j\in\{1,2,\cdots,K\}\\ j\neq k}   }\frac{x-x_j}{x_k-x_j}\ =\  
			w_k \prod_{
				\substack{   j\in\{1,2,\cdots,K\}\\ j\neq k}   }({x-x_j}),
		\end{split}
	\end{equation*}
	where
	\begin{equation*}
	    w_k=\prod_{
				\substack{   j\in\{1,2,\cdots,K\}\\ j\neq k}   }\frac{1}{x_k-x_j}\neq 0\quad \tn{for $k=1,2,\cdots,K$.}
	\end{equation*}
	 It follows from $x_k\in \Q$ that $w_k$ is rational and nonzero, i.e., $w_k\in \Q/\{0\}$ for any $k$. Clearly, each $p_k$ is a polynomial of degree $\le K-1$. That means we can represent $p_k$ by
	 \begin{equation*}
		p_k(x)=\sum_{j=1}^{K}a_{k,j}x^{j-1}=a_{k,1}+a_{k,2}x+\cdots+a_{k,K}x^{K-1}
	\end{equation*}
	for $k=1,2,\cdots,K$ and any $x\in \R$, where each coefficient $a_{k,j}$ is rational.
    Therefore, the coefficients of $p_1,p_2,\cdots,p_K$ form a matrix 
	\begin{equation*}
		\bmA=(a_{i,j})=\left[\begin{matrix}
			a_{1,1} & a_{1,2} & \cdots & a_{1,K}\\
			a_{2,1} & a_{2,2} & \cdots & a_{2,K}\\
			\vdots & \vdots & \ddots & \vdots\\
			a_{K,1} & a_{K,2} & \cdots & a_{K,K}
		\end{matrix}\right]\in\Q^{K\times K}. 
	\end{equation*}

	Now assume there exist rational numbers $\lambda_1,\lambda_2,\cdots,\lambda_K\in\Q$ such that $\sum_{k=1}^K \lambda_k\cdot \tfrac{1}{\alpha+r_k}=0$. Our goal is to prove $\lambda_1=\lambda_2=\cdots=\lambda_K=0$. 
	Clearly, we have 
	\begin{equation*}
		\begin{split}
			0&= \sum_{k=1}^{K}\frac{\lambda_k}{\alpha+r_k}
			=
			\underbrace{
			\sum_{k=1}^{K}\frac{\lambda_k}{\alpha-x_k}
			}_{= 0}
			=\prod_{j=1}^K (\alpha-x_j) 
			\cdot
			\underbrace{\sum_{k=1}^K \frac{\lambda_k}{\alpha-x_k}
			}_{= 0}
			=
			\sum_{k=1}^K \frac{\lambda_k}{w_k}\cdot w_k\prod_{\substack{   j\in\{1,2,\cdots,K\}\\ j\neq k}}(\alpha-x_j)\\
			&=\sum_{k=1}^K \frac{\lambda_k}{w_k}\cdot p_k(\alpha)
			= \sum_{k=1}^K \frac{\lambda_k}{w_k} \sum_{j=1}^{K}a_{k,j}\alpha^{j-1}
			= \sum_{j=1}^{K}\Big( 
			\underbrace{\sum_{k=1}^K\frac{\lambda_k}{w_k}a_{k,j}}_{=0 \tn{ since }\alpha\in\R\backslash\A}
			\Big)   \cdot \alpha^{j-1}.
		\end{split}
	\end{equation*}
	 For any $k,j\in \{1,2,\cdots,K\}$, we have $\lambda_k,w_k,a_{k,j}\in\Q$, implying $\sum_{k=1}^K\frac{\lambda_k}{w_k}a_{k,j}\in \Q$.
	Since $\alpha\in \R\backslash \A$ is a transcendental number, the coefficients must be $0$, i.e., 
	\begin{equation*}
		\sum_{k=1}^K\frac{\lambda_k}{w_k}a_{k,j}=0\quad \tn{for $j=1,2,\cdots,K$.}
	\end{equation*}
	It follows that 
	\begin{equation*}		
		\bm{0}=\begin{bmatrix}
			\frac{\lambda_1}{w_1},\frac{\lambda_2}{w_2},\cdots,\frac{\lambda_K}{w_K}
		\end{bmatrix}
		\begin{bmatrix}
			a_{1,1} & a_{1,2} & \cdots & a_{1,K}\\
			a_{2,1} & a_{2,2} & \cdots & a_{2,K}\\
			\vdots & \vdots & \ddots & \vdots\\
			a_{K,1} & a_{K,2} & \cdots & a_{K,K}
		\end{bmatrix}
		=\begin{bmatrix}
			\frac{\lambda_1}{w_1},\frac{\lambda_2}{w_2},\cdots,\frac{\lambda_K}{w_K}
		\end{bmatrix}\bmA.
	\end{equation*}
	By Lemma~\ref{lem:matrix:inv}, $\bmA$ is invertible. Thus, $\big[\frac{\lambda_1}{w_1},\frac{\lambda_2}{w_2},\cdots,\frac{\lambda_K}{w_K}\big]=\bm{0}$, which implies $\lambda_1=\lambda_2=\cdots=\lambda_K=0$. Hence, the set of numbers $\big\{\tfrac{1}{\alpha+r_k}:k=1,2,\cdots,K\big\}$
	are rationally independent, which means we finish the proof.	
\end{proof}

\subsection{Proof of Lemma~\ref{lem:dense:periodic}}\label{sec:proof:lem:dense:periodic}

The proof of Lemma~\ref{lem:dense:periodic} is mainly based on the fact that
an irrational winding  is dense on the torus (e.g., see Lemma~2 of \cite{yarotsky:2021:02}). 
For completeness, we establish a lemma below and give its detailed proof.
\begin{lemma}\label{lem:dense:decimal}
	Given any $K\in \N^+$ and an arbitrary set of rationally independent numbers $\{a_k:k=1,2,\cdots,K\}\subseteq \R$, the following set 
	\begin{equation*}
		\Big\{\,   \big[\tau(wa_1),\   \tau(wa_2),\    \cdots,\  \tau(wa_K)\big]^T \,   : \,   w\in\R \,   \Big\}\subseteq [0,1)^K
	\end{equation*}
is dense in $[0,1]^K$, where $ \tau(x)\coloneqq x-\lfloor x\rfloor$  for any $x\in\R$.
\end{lemma}

The proof of  Lemma~\ref{lem:dense:decimal} can be found later in this section. Now let us first prove Lemma~\ref{lem:dense:periodic} by assuming Lemma~\ref{lem:dense:decimal} is true.
\begin{proof}[Proof of Lemma~\ref{lem:dense:periodic}]
	Define $\tildeg(x)\coloneqq g(Tx)$ for any $x\in \R$. Clearly, $\tildeg$ is  periodic with period $1$ since $g$ is  periodic with period $T$.
	The continuity of $g$ on $[x_1,x_2]$ implies $\tildeg$ is continuous on $[\tfrac{x_1}{T},\tfrac{x_2}{T}]$ and therefore uniformly continuous on $[\tfrac{x_1}{T},\tfrac{x_2}{T}]$. 
	For any $\varepsilon>0$, there exists $\delta\in (0,\tfrac{x_2-x_1}{T})$ such that
	\begin{equation}\label{eq:gtilde:uniformly:cont}
		|\tildeg(u)-\tildeg(v)|<\varepsilon\quad\tn{for any $u,v\in [\tfrac{x_1}{T},\tfrac{x_2}{T}]$ with $|u-v|<\delta$.}
	\end{equation}
	
Given any $\bmxi=[\xi_1,\xi_2,\cdots,\xi_K]\in [M_1,M_2]^K$, by the extreme value theorem and the intermediate value theorem, there exists $z_1,z_2,\cdots,z_K\in [x_1,x_2]$ such that 
\begin{equation}\label{eq:g:zk:xik}
    g(z_k)=\xi_k\quad \tn{ for any $k=1,2,\cdots,K$.}
\end{equation}

For $k=1,2,\cdots,K$, 
set $y_k=z_k/T \in [\tfrac{x_1}{T},\tfrac{x_2}{T}]$ and \begin{equation*}
	\tildey_k=y_k +\tfrac{\delta}{2}\cdot\one_{\{y_k\le \tfrac{x_1}{T}+\tfrac{\delta}{2}\}}-\tfrac{\delta}{2}\cdot\one_{\{y_k\ge \tfrac{x_2}{T}-\tfrac{\delta}{2}\}}.
\end{equation*}
Then, for $k=1,2,\cdots,K$,  we have 
\begin{equation*}
	\tildey_k=y_k +\tfrac{\delta}{2}\cdot\one_{\{y_k\le \tfrac{x_1}{T}+\tfrac{\delta}{2}\}}-\tfrac{\delta}{2}\cdot\one_{\{y_k\ge \tfrac{x_2}{T}-\tfrac{\delta}{2}\}}\in \big[\tfrac{x_1}{T}+\tfrac{\delta}{2},\,\tfrac{x_2}{T}-\tfrac{\delta}{2}\big]
\end{equation*}
and 
\begin{equation*}
	|\tildey_k-y_k|\le \Big|\tfrac{\delta}{2}\cdot\one_{\{y_k\le \tfrac{x_1}{T}+\tfrac{\delta}{2}\}}-\tfrac{\delta}{2}\cdot\one_{\{y_k\ge \tfrac{x_2}{T}-\tfrac{\delta}{2}\}}\Big|\le \delta/2.
\end{equation*}

Define $\tau(x)\coloneqq  x-\lfloor x\rfloor$ for any $x\in\R$. Clearly, $[\tau(\tildey_1),\tau(\tildey_2),\cdots,\tau(\tildey_K)]^T\in [0,1]^K$. Then, by Lemma~\ref{lem:dense:decimal}, there exists $w_0\in\R$ such that
\begin{equation*}
	|\tau(w_0a_k)-\tau(\tildey_k)|<\delta/2\quad \tn{for $k=1,2,\cdots,K$.}
\end{equation*}
It follows that
\begin{equation*}
	\Big|\tau(w_0a_k)+\lfloor \tildey_k\rfloor- \tildey_k\Big|=\Big|\tau(w_0a_k)- (\tildey_k-\lfloor \tildey_k\rfloor)\Big|=\big|\tau(w_0a_k)-\tau(\tildey_k)\big|<\delta/2
\end{equation*}
for $k=1,2,\cdots,K$. Since $\tildey_k\in [\tfrac{x_1}{T}+\tfrac{\delta}{2},\tfrac{x_2}{T}-\tfrac{\delta}{2}]$, we have $\tau(w_0a_k)+\lfloor \tildey_k\rfloor\in [\tfrac{x_1}{T},\tfrac{x_2}{T}]$. 
Besides,
\begin{equation*}
	\Big|\tau(w_0a_k)+\lfloor \tildey_k\rfloor-y_k\Big|\le \Big|\tau(w_0a_k)+\lfloor \tildey_k\rfloor-\tildey_k\Big|+ \big|\tildey_k-y_k\big|<\delta/2+\delta/2=\delta
\end{equation*}
for $k=1,2,\cdots,K$. Then, by Equation~\eqref{eq:gtilde:uniformly:cont}, we have
\begin{equation*}
	\Big| \tildeg\big(\tau(w_0a_k)+\lfloor \tildey_k\rfloor\big)-\tildeg(y_k) \Big|<\varepsilon\quad \tn{for $k=1,2,\cdots,K$.}
\end{equation*}

Recall that $\tildeg$ is  periodic with period $1$, from which we deduce
\begin{equation*}
	\tildeg\big(\tau(w_0a_k)+\lfloor \tildey_k\rfloor\big)=\tildeg\big(w_0a_k-\lfloor w_0 a_k\rfloor+\lfloor \tildey_k\rfloor\big)=\tildeg(w_0a_k)=g(T\cdot  w_0a_k)
\end{equation*}
for $k=1,2,\cdots,K$. Also,  we have
\begin{equation*}
    \tildeg(y_k)=g(Ty_k)=g(z_k)=\xi_k \quad \tn{ for $k=1,2,\cdots,K$, }
\end{equation*}
where the last equality comes from Equation~\eqref{eq:g:zk:xik}.
It follows that
\begin{equation*}
	\begin{split}
		\big|g(T\cdot  w_0a_k)- \xi_k\big|=\Big| \tildeg\big(\tau(w_0a_k)+\lfloor \tildey_k\rfloor\big)-\tildeg(y_k) \Big|<\varepsilon\quad 
		\tn{for $k=1,2,\cdots,K$.}
	\end{split}
\end{equation*}
That is 
\begin{equation*}
	\Big\| \big[g(w_1a_1),\, g(w_1a_2),\, \cdots,\, g(w_1a_K)\big]^T -\bmxi    \Big\|_{\infty}<\varepsilon,
\end{equation*}
where $w_1=T\cdot  w_0\in\R$. Since $\bmxi\in [M_1,M_2]^K$ and $\varepsilon>0$ are arbitrary, the following set
\begin{equation*}
	\Big\{\big[g(wa_1),\, g(wa_2),\, \cdots,\,g(wa_K)\big]^T: w\in\R
	\Big\}
\end{equation*}
is dense in $[M_1,M_2]^K$ as desired. So we finish the proof.
\end{proof}

Finally, let us present the detailed proof of  Lemma~\ref{lem:dense:decimal}.
\begin{proof}[Proof of Lemma~\ref{lem:dense:decimal}]
	We prove this lemma by mathematical induction. 
	First, we consider the case $K=1$.  Note that $a_1\neq 0$ since it is rationally independent.
	Thus, we have $\{\tau(wa_1):w\in\R\}=[0,1)$, which implies $\{\tau(wa_1):w\in\R\}$ is dense in $[0,1]$.
	
	Now assume this lemma holds for $K=J-1\in\N^+$. Our goal is to prove the case $K=J$. 
	Given any $\varepsilon\in (0,1/100)$ and an arbitrary $\bmxi=[\xi_1,\xi_2,\cdots,\xi_J]^T\in [0,1]^J$, our goal is to find a proper $w\in\R$ such that
	\begin{equation}\label{eq:goal:w}
		|\tau(wa_j)-\xi_j|<C\varepsilon\quad \tn{for $j=1,2,\cdots,J$,\quad where $C$ is an absolute constant.}
	\end{equation}
We remark that the constant $C$ in the above equation is actually equal to $11$ in our proof.
As we shall see later, we need an assumption that the given point is in $[6\varepsilon,1-6\varepsilon]^J$. Thus, we slightly modify $\bmxi$ by setting
\begin{equation*}
	\tildexi_j=\xi_j+6\varepsilon\cdot\one_{\{\xi_j\le 6\varepsilon\}}-6\varepsilon\cdot\one_{\{\xi_j\ge 1-6\varepsilon\}}
	\quad \tn{for $j=1,2,\cdots,J$.}
\end{equation*}
Then, we have 
\begin{equation}\label{eq:tildexi1}
	\tildexi_j\in [6\varepsilon,1-6\varepsilon]\quad \tn{for $j=1,2,\cdots,J$}
\end{equation}
and
\begin{equation}\label{eq:tildexi2}
	\big|\xi_j-\tildexi_j\big|
	=\big|6\varepsilon\cdot\one_{\{\xi_j\le 6\varepsilon\}}-6\varepsilon\cdot\one_{\{\xi_j\ge 1-6\varepsilon\}}\big|\le 6\varepsilon\quad \tn{for $j=1,2,\cdots,J$.}
\end{equation}

For any $n\in \N^+$, we define 
\begin{equation*}
	\hatxi_j\coloneqq \tau(\tildexi_j-\tfrac{\tildexi_J}{a_J}a_j)\quad \tn{for $j=1,2,\cdots,J$.}
\end{equation*}
Then $\hatxi_J=0$ and $\hatxi_j\in [0,1)$ for $j=1,2,\cdots,J-1$. To approximate $[\hatxi_1,\hatxi_2,\cdots,\hatxi_{J-1}]^T\in [0,1)^{J-1}$, we only need to consider $J-1$ indices, and, therefore, we can use the induction hypothesis to continue our proof.

Clearly, the rational independence of $a_1,a_2,\cdots,a_J$ implies none of them is equal to  zero. Define
\begin{equation*}
	\bmb_n\coloneqq \big[\tau(\tfrac{n}{a_J}a_1),\, \tau(\tfrac{n}{a_J}a_2),\, \cdots,\, \tau(\tfrac{n}{a_J}a_{J-1})\big]^T\in [0,1)^{J-1}.
\end{equation*}
Then, the bounded sequence $(\bmb_n)_{n=1}^\infty$ has a convergent subsequence by the Bolzano-Weierstrass Theorem. 
Thus, there exist $n_1,n_2\in \N^+$ with $n_1<n_2$ such that 
$\|\bmb_{n_2}-\bmb_{n_1}\|_{\infty}<\varepsilon$, i.e.,
\begin{equation*}
	\big| \tau(\tfrac{n_2}{a_J}a_j) - \tau(\tfrac{n_1}{a_J}a_j) \big|<\varepsilon\quad \tn{for $j=1,2,\cdots,J-1$.}
\end{equation*}
Set $\hatn=n_2-n_1\in\N^+$ and 
\begin{equation*}
    k_j=\big\lfloor \tfrac{n_1}{a_J}a_j\big\rfloor-\big\lfloor \tfrac{n_2}{a_J}a_j\big\rfloor\in \Z\quad \tn{  for $j=1,2,\cdots,J-1$.}
\end{equation*}
Then, by defining 
\begin{equation*}
	\hata_j\coloneqq \tfrac{\hatn}{a_J}a_j +k_j\quad \tn{ for $j=1,2,\cdots,J-1$,}
\end{equation*} 
we have 
\begin{equation}\label{eq:hata:bound}
\begin{split}
		|\hata_j|=\big| \tfrac{\hatn}{a_J}a_j +k_j \big|
	&=\Big| \tfrac{n_2}{a_J}a_j - \tfrac{n_1}{a_J}a_j +\big\lfloor \tfrac{n_1}{a_J}a_j\big\rfloor-\big\lfloor \tfrac{n_2}{a_J}a_j\big\rfloor \Big|\\
	&=\Big| \big(\tfrac{n_2}{a_J}a_j -\big\lfloor \tfrac{n_2}{a_J}a_j\big\rfloor\big) 
	- \big(\tfrac{n_1}{a_J}a_j -\big\lfloor \tfrac{n_1}{a_J}a_j\big\rfloor\big)\Big|\\
	&=\big| \tau(\tfrac{n_2}{a_J}a_j) - \tau(\tfrac{n_1}{a_J}a_j) \big|<\varepsilon.
\end{split}
\end{equation}

It is easy to verify that $\hata_1,\hata_2,\cdots,\hata_{J-1}$ are rationally independent. To see this, assume there exist $\lambda_1,\lambda_2,\cdots,\lambda_{J-1}\in \Q$ such that
\begin{equation*}
 0=\sum_{j=1}^{J-1}\lambda_j \hata_j=	\sum_{j=1}^{J-1}\lambda_j\big( \tfrac{\hatn}{a_J}a_j +k_j \big)=\sum_{j=1}^{J-1}\lambda_j\tfrac{\hatn}{a_J}a_j+ \sum_{j=1}^{J-1}\lambda_jk_j.
\end{equation*}
It follows that
\begin{equation*}
	0=	\sum_{j=1}^{J-1}\lambda_j\hatn a_j+ \Big(\sum_{j=1}^{J-1}\lambda_j k_j\Big)a_J.
\end{equation*}
Recall that $\hatn\in \N^+$, $k_j\in \Z$, and $\lambda_j\in \Q$ for any $j$. That means the coefficients $\lambda_j\hatn$ and $\sum_{j=1}^{J-1}\lambda_jk_j$ are rational for any $j$.
Since $a_1,a_2,\cdots,a_J$ are rationally independent, we have 
\begin{equation*}
    \lambda_j\hatn=0\quad\tn{and}\quad \sum_{j=1}^{J-1}\lambda_jk_j=0 \quad \tn { for $j=1,2,\cdots,J-1$.}
\end{equation*}
It follows from  $\hatn=n_2-n_1>0$ that $\lambda_1=\lambda_2=\cdots=\lambda_{J-1}=0$. Therefore, $\hata_1,\hata_2,\cdots,\hata_{J-1}$ are rationally independent as desired.

By the  induction hypothesis, the following set
\begin{equation*}
	\Big\{\big[\tau( s\cdot\hata_1),  \  
	\tau(s\cdot\hata_2),\ \cdots,\ 
	\tau(s\cdot\hata_{J-1})\big]^T:s\in\R \Big\}\subseteq [0,1)^{J-1}
\end{equation*}
is dense in $[0,1]^{J-1}$. Recall that 
$\hatxi_j=\tau(\tildexi_j-\tfrac{\tildexi_J}{a_J}a_j)\in [0,1]$ for $j=1,2,\cdots,J-1$, implying
\begin{equation*}
	\hatxi_j+3\varepsilon\cdot\one_{\{\hatxi_j\le 3\varepsilon\}}-3\varepsilon\cdot\one_{\{\hatxi_j\ge 1- 3\varepsilon\}} \in [3\varepsilon,1-3\varepsilon].
\end{equation*}
Hence, there exists $s_0\in\R$ such that
\begin{equation*}
	\Big| \tau(s_0\hata_j) -\Big(
	\hatxi_j+3\varepsilon\cdot\one_{\{\hatxi_j\le 3\varepsilon\}}-3\varepsilon\cdot\one_{\{\hatxi_j\ge 1- 3\varepsilon\}}
	\Big)\Big|<\varepsilon
\end{equation*}
for $j=1,2,\cdots,J-1$. It follows that 
\begin{equation*}
	\tau(s_0\hata_j)\in [2\varepsilon,1-2\varepsilon] \quad \tn{for $j=1,2,\cdots,J-1$}
\end{equation*}
and 
\begin{equation}\label{eq:s0:hatxi}
	\Big| \tau(s_0\hata_j) -
	\hatxi_j	\Big|
	<\varepsilon+\Big|3\varepsilon\cdot\one_{\{\hatxi_j\le 3\varepsilon\}}-3\varepsilon\cdot\one_{\{\hatxi_j\ge 1- 3\varepsilon\}}\Big|\le 4\varepsilon
\end{equation}
for $j=1,2,\cdots,J-1$.

To estimate $\tau(\lfloor s_0\rfloor\hata_j)-\hatxi_j$, we need to bound $\tau(s_0\hata_j)-\tau(\lfloor s_0\rfloor\hata_j)$. To this end, we need an 
observation for any $x,y\in\R$ as follows.
\begin{equation}\label{eq:observe:xy}
	|x-y|<\varepsilon \tn{\ \  and \   } \tau(x)\in [2\varepsilon,1-2\varepsilon]  
	 \quad   \Longrightarrow \quad    |\tau(x)-\tau(y)|<\varepsilon.
\end{equation}
In fact, $\tau(x)\in [2\varepsilon,1-2\varepsilon]$ implies $\varepsilon\le \tau(x)-\varepsilon\le \tau(x)+\varepsilon\le 1-\varepsilon$, from which we deduce
\begin{equation*}
    \begin{split}
        	y\in [x-\varepsilon,x+\varepsilon]
        	&=
	\Big[\lfloor x\rfloor + \underbrace{\tau(x)-\varepsilon}_{\ge \varepsilon}, 
	\lfloor x\rfloor +\underbrace{\tau(x)+\varepsilon}_{\le 1- \varepsilon}    \Big]\\
	&\subseteq 
	\Big[\lfloor x\rfloor +\varepsilon, \lfloor x\rfloor +1-\varepsilon\Big]\subseteq \Big[\lfloor x\rfloor, \lfloor x\rfloor +1\Big).
    \end{split}
\end{equation*}
Then, we have $\lfloor y\rfloor=\lfloor x\rfloor$, which implies 
\begin{equation*}
    \begin{split}
        |\tau(x)-\tau(y)|&= \big|\tau(x)-\tau(y)+\lfloor x\rfloor-\lfloor y\rfloor\big|\\
        &= \Big|\big(\tau(x)+\lfloor x\rfloor\big)-\big(\tau(y)+\lfloor y\rfloor\big)\Big|=|x-y|<\varepsilon.
    \end{split}
\end{equation*}
Thus, Equation~\eqref{eq:observe:xy} is proved.

By Equation~\eqref{eq:hata:bound}, we have
\begin{equation*}
	\Big|s_0\hata_j-\lfloor s_0\rfloor \hata_j\Big|\le \Big|s_0-\lfloor s_0\rfloor \Big|\cdot|\hata_j|\le |\hata_j|<\varepsilon\quad \tn{for $j=1,2,\cdots,J-1$.}
\end{equation*}
Recall that 
\begin{equation*}
    \tau(s_0\hata_j)\in [2\varepsilon,1-2\varepsilon]\quad \tn{ for $j=1,\cdots,J-1$.}
\end{equation*}
Then, for each $j\in\{1,2,\cdots,J-1\}$, by the observation above in Equation~\eqref{eq:observe:xy} (set $x=s_0\hata_j$ and $y=\lfloor s_0\rfloor\hata_j$ therein), we have $\big|\tau(s_0\hata_j)-\tau(\lfloor s_0\rfloor\hata_j)\big|<\varepsilon$.

Recall that $\hatxi_j= \tau(\tildexi_j-\tfrac{\tildexi_J}{a_J}a_j)$ for $j=1,2,\cdots,J$.
Therefore, by Equation~\eqref{eq:s0:hatxi}, we have
\begin{equation*}
	\begin{split}
			\Big| \tau(\lfloor s_0\rfloor\hata_j) - \tau(\tildexi_j-\tfrac{\tildexi_J}{a_J}a_j)	\Big|
			&=\Big| \tau(\lfloor s_0\rfloor\hata_j) - \hatxi_j	\Big|\\ 
			&\le \Big| \tau(\lfloor s_0\rfloor\hata_j) - \tau(s_0\hata_j)\Big| + \Big| \tau(s_0\hata_j) -
			\hatxi_j	\Big|< \varepsilon+4\varepsilon=5\varepsilon,
	\end{split}
\end{equation*}
for $j=1,2,\cdots,J-1$. 

Observe that, for any $x,y\in \R$, there exist $z\in \Z$ such that
$\tau(x)-\tau(y)=x-y-z$. To see this, we set $z=\lfloor x\rfloor-\lfloor y\rfloor\in\Z$ and then  $\tau(x)-\tau(y)=x-\lfloor x\rfloor -\big(y-\lfloor y\rfloor \big)=x-y-z$.
Therefore, for $j=1,2,\cdots,J-1$, there exists $z_j\in \Z$ such that 
\begin{equation*}
	\tau(\lfloor s_0\rfloor\hata_j) - \tau(\tildexi_j-\tfrac{\tildexi_J}{a_J}a_j)
	=\lfloor s_0\rfloor\hata_j -\big(\tildexi_j-\tfrac{\tildexi_J}{a_J}a_j\big) -z_j
	=\lfloor s_0\rfloor\hata_j+\tfrac{\tildexi_J}{a_J}a_j -(z_j+\tildexi_j),
\end{equation*}
which implies
\begin{equation*}
	\Big|\lfloor s_0\rfloor\hata_j+\tfrac{\tildexi_J}{a_J}a_j -(z_j+\tildexi_j)\Big|
	=\Big|\tau(\lfloor s_0\rfloor\hata_j) - \tau(\tildexi_j-\tfrac{\tildexi_J}{a_J}a_j)\Big|<5\varepsilon.
\end{equation*}
It follows that, for $j=1,2,\cdots,J-1$, 
\begin{equation*}
	\lfloor s_0\rfloor\hata_j+\tfrac{\tildexi_J}{a_J}a_j\in 
	[z_j+   \underbrace{\tildexi_j-5\varepsilon}_{\ge \varepsilon }, 
	z_j+    \underbrace{\tildexi_j+5\varepsilon}_{\le 1- \varepsilon }  ]\subseteq [z_j+\varepsilon,z_j+1-\varepsilon],
\end{equation*}
where the fact  $\varepsilon\le \tildexi_j-5\varepsilon\le \tildexi_j+5\varepsilon\le 1-\varepsilon$ comes from Equation~\eqref{eq:tildexi1}. Therefore, 
we have 
\begin{equation*}
    \Big\lfloor \lfloor s_0\rfloor\hata_j+\tfrac{\tildexi_J}{a_J}a_j)\Big\rfloor=z_j\quad \tn{
for $j=1,2,\cdots,J-1$,}
\end{equation*}
 implying
\begin{equation*}
	\tau(\lfloor s_0\rfloor\hata_j+\tfrac{\tildexi_J}{a_J}a_j)=\big(\lfloor s_0\rfloor\hata_j+\tfrac{\tildexi_J}{a_J}a_j\big)-z_j\in 
	[\tildexi_j-5\varepsilon, 
	\tildexi_j+5\varepsilon].
\end{equation*}
 
 Clearly, we have 
\begin{equation*}
	\lfloor s_0\rfloor\hata_j+\tfrac{\tildexi_J}{a_J}a_j=\lfloor s_0\rfloor\Big(\tfrac{\hatn}{a_J}a_j+k_j\Big)+\tfrac{\tildexi_J}{a_J}a_j=\tfrac{\lfloor s_0\rfloor \hatn+\tildexi_J}{a_J}a_j+\underbrace{k_j\lfloor s_0\rfloor}_{\in\Z}
\end{equation*}
for $j=1,2,\cdots,J-1$, which implies 
\begin{equation*}
	\tau(\tfrac{\lfloor s_0\rfloor \hatn+\tildexi_J}{a_J}a_j)=\tau(\lfloor s_0\rfloor\hata_j+\tfrac{\tildexi_J}{a_J}a_j)\in 
	[\tildexi_j-5\varepsilon, 
	\tildexi_j+5\varepsilon].
\end{equation*}
We also need to consider the case $j=J$. By Equation~\eqref{eq:tildexi1}, we have $\tildexi_J\in [6\varepsilon,1-6\varepsilon]$, from which we deduce
\begin{equation*}
    \tau(\tfrac{\lfloor s_0\rfloor \hatn+\tildexi_J}{a_J}a_J)
    =\tau(\underbrace{\lfloor s_0\rfloor \hatn}_{\in\Z}+\tildexi_J)=\tildexi_J.
\end{equation*}
Thus, for $j=1,2,\cdots,J$, we have
\begin{equation*}
	\Big|\tau(\tfrac{\lfloor s_0\rfloor \hatn+\tildexi_J}{a_J}a_j)-\tildexi_j\Big|\le 5\varepsilon.
\end{equation*}
By Equation~\eqref{eq:tildexi2},  we have $|\tildexi_j-\xi_j|<6\varepsilon$ for $j=1,2,\cdots,J$, which implies
\begin{equation*}
	\Big|\tau(\tfrac{\lfloor s_0\rfloor \hatn+\tildexi_J}{a_J}a_j)-\xi_j\Big|\le 
	\Big|\tau(\tfrac{\lfloor s_0\rfloor \hatn+\tildexi_J}{a_J}a_j)-\tildexi_j\Big|+\big|\tildexi_j-\xi_j\big|\le 5\varepsilon+6\varepsilon=11\varepsilon.
\end{equation*}
That means $w_0=\tfrac{\lfloor s_0\rfloor \hatn+\tildexi_J}{a_J}$ is the desired $w$ in Equation~\eqref{eq:goal:w} and the constant $C>0$ therein is $11$. Therefore,
\begin{equation*}
	\big|\tau(w_0a_j)-\xi_j\big|\le 11\varepsilon\quad\tn{for $j=1,2,\cdots,J$.}
\end{equation*}
Since $\bmxi=[\xi_1,\xi_2,\cdots,\xi_J]^T\in [0,1]^J$ and $\varepsilon>0$ are arbitrary, the following set
\begin{equation*}
	\Big\{\big[\tau(w a_1),  \  
	\tau(wa_2),\ \cdots,\ 
	\tau(wa_J)\big]^T:w\in\R \Big\}\subseteq [0,1)^{J}
\end{equation*}
is dense in $[0,1]^{J}$ as desired. We finish the process of  mathematical induction and therefore finish the proof by the principle of mathematical induction.
\end{proof}
We remark that the target parameter $w_0=\tfrac{\lfloor s_0\rfloor \hatn+\tildexi_J}{a_J}$ designed in the above proof may not be bounded uniformly for any approximation error $\varepsilon$ since $\hatn$ can be arbitrarily large as $\varepsilon$ goes to $0$. Therefore, the network in Theorem~\ref{thm:main} may require sufficiently large parameters to achieve an arbitrarily small error $\varepsilon$.

\section{Conclusion}\label{sec:conclusion}
This paper studies the super approximation power of deep feed-forward neural networks activated by EUAF with a fixed size.
It is proved by construction that there exists an EUAF network architecture with $d$ input neurons, 
a maximum width $36d(2d+1)$, $11$ hidden layers, and at most $5437(d+1)(2d+1)$ nonzero parameters, 
achieving the universal approximation property by only adjusting its finitely many  parameters. That is, without changing the network size, our EUAF network can approximate any continuous function $f:[a,b]^d\to \R$ within an arbitrarily small error $\varepsilon>0$ with appropriate parameters depending on $f$, $\varepsilon$, $d$, $a$, and $b$. 
Moreover, augmenting this EUAF network using one more layer with $2$ neurons can exactly realize a classification function $\sum_{j=1}^J r_j\cdot \one_{E_j}$ in $\bigcup_{j=1}^J E_j$ for any $J\in\N^+$, where $r_1,r_2,\cdots,r_J$ are distinct rational numbers and $E_1,E_2,\cdots,E_J$ are arbitrary pairwise disjoint bounded closed subsets of $\mathbb{R}^d$.

While we are interested in the analysis of the approximation error here,
it would be very interesting to investigate the generalization and optimization errors of EUAF networks.
Acting as a proof of concept, our experimentation shows the numerical advantages of EUAF compared to ReLU.
We believe our EUAF activation function could be further developed and applied to real-world applications.

\acks{Z.~Shen was supported by 
Distinguished Professorship of National University of Singapore.
H.~Yang was partially supported by the US National Science Foundation under award DMS-2244988, DMS-2206333, and the Office of Naval Research Young Investigator Award.
}
\vskip 0.2in
\bibliography{references}
\end{document}